\documentclass[11pt]{article}

\usepackage[x11names]{xcolor}

\usepackage[utf8]{inputenc}
\usepackage[T1]{fontenc}
\usepackage[sc]{mathpazo}
\usepackage{fullpage}
\usepackage{authblk}

\usepackage[colorlinks=true,citecolor=Blue4,linkcolor=Red3]{hyperref}
\usepackage[round,sort]{natbib}

\usepackage{amsthm,amsmath,amssymb}
\usepackage{thm-restate}

\newtheorem{theorem}{Theorem}[section]
\newtheorem*{theorem*}{Theorem}
\newtheorem{lemma}[theorem]{Lemma}
\newtheorem{corollary}[theorem]{Corollary}
\newtheorem{remark}[theorem]{Remark}
\newtheorem{claim}[theorem]{Claim}
\newtheorem*{claim*}{Claim}
\newtheorem{example}[theorem]{Example}

\newtheorem{observation}[theorem]{Observation}
\theoremstyle{definition}
\newtheorem{definition}[theorem]{Definition}

\usepackage[nameinlink]{cleveref}
\usepackage{mathtools}

\usepackage{fancybox}
\usepackage{ascmac}

\usepackage{microtype}
\usepackage{graphicx}
\usepackage{subfigure}
\usepackage{multirow,multicol,dcolumn,booktabs,colortbl,hhline}
\usepackage{pdflscape}
\usepackage{enumitem}

\usepackage{comment}
\usepackage[colorinlistoftodos,prependcaption]{todonotes}
\usepackage{xspace}
\usepackage{bm}

\usepackage{hf-tikz}
\usepackage{framed}
\usepackage{pifont}

\crefname{theorem}{Theorem}{Theorems}
\crefname{lemma}{Lemma}{Lemmas}
\crefname{definition}{Definition}{Definitions}
\crefname{claim}{Claim}{Claims}
\crefname{remark}{Remark}{Remarks}
\crefname{observation}{Observation}{Observations}
\crefname{corollary}{Corollary}{Corollaries}
\crefname{appendix}{Appendix}{Appendices}
\crefname{section}{Section}{Sections}
\crefname{equation}{Eq.}{Eqs.}
\crefname{figure}{Figure}{Figures}
\crefname{table}{Table}{Tables}

\newcommand{\rme}{\mathrm{e}}

\newcommand{\bbN}{\mathbb{N}}
\newcommand{\bbQ}{\mathbb{Q}}
\newcommand{\bbR}{\mathbb{R}}

\newcommand{\calC}{\mathcal{C}}
\newcommand{\calI}{\mathcal{I}}
\newcommand{\calJ}{\mathcal{J}}
\newcommand{\calP}{\mathcal{P}}

\renewcommand{\vec}[1]{\mathbf{\bm{#1}}}
\newcommand{\mat}[1]{\mathbf{\bm{#1}}}

\newcommand{\bigO}{\mathord{\mathcal{O}}}
\newcommand{\alg}{\mathord{\textsc{alg}}}
\DeclareMathOperator{\rank}{rank}
\DeclareMathOperator{\sgn}{sgn}
\DeclareMathOperator{\per}{per}
\DeclareMathOperator{\tw}{tw}
\DeclareMathOperator{\nnz}{nz}
\DeclareMathOperator{\isize}{size}
\DeclareMathOperator{\diag}{diag}

\DeclareMathOperator*{\argmax}{argmax}
\DeclareMathOperator{\inv}{inv}
\DeclareMathOperator{\ZZ}{\mathcal{Z}}

\DeclareMathOperator{\intro}{\mathfrak{I}}
\DeclareMathOperator{\forget}{\mathfrak{F}}
\DeclareMathOperator{\join}{\mathfrak{J}}

\newcommand{\Sym}{\mathord{\mathfrak{S}}}
\newcommand{\iu}{\mathord{\mathrm{i\mkern1mu}}}
\newcommand{\state}{\mathord{\mathcal{C}}}
\newcommand{\OPT}{\mathord{\mathrm{OPT}}}
\newcommand*{\bij}{\mathrel{\rightarrowtail\!\!\!\!\!\rightarrow}}
\newcommand{\isep}{\mathrel{..}\nobreak}

\newcommand{\cP}{\textup{\textsf{P}}\xspace}
\newcommand{\FP}{\textup{\textsf{FP}}\xspace}
\newcommand{\RP}{\textup{\textsf{RP}}\xspace}
\newcommand{\UP}{\textup{\textsf{UP}}\xspace}
\newcommand{\ModkP}[1]{\textup{\textsf{Mod}$_{#1}$\textsf{P}}\xspace}
\newcommand{\NP}{\textup{\textsf{NP}}\xspace}
\newcommand{\shP}{\textup{\#\textsf{P}}\xspace}
\newcommand{\FPT}{\textup{\textsf{FPT}}\xspace}

\newcommand{\shW}[1]{\textup{\#\textsf{W}[#1]}\xspace}
\newcommand{\XP}{\textup{\textsf{XP}}\xspace}

\DeclarePairedDelimiter\ceil{\lceil}{\rceil}
\DeclarePairedDelimiter\floor{\lfloor}{\rfloor}

\newcommand{\nicered}{Red2}
\newcommand{\niceblue}{Blue4}
\newcommand{\nicegreen}{Cyan4}
\newcommand{\cmark}{\textcolor{\niceblue}{\ding{51}}\xspace}%
\newcommand{\xmark}{\textcolor{\nicered}{\ding{55}}\xspace}%
\newcommand{\Holder}{H\"{o}lder}

\usepackage{pdflscape}
\usepackage{afterpage}

\title{Computational Complexity of Normalizing Constants for the Product of Determinantal Point Processes}
\author[1]{Naoto Ohsaka\thanks{ \href{mailto:ohsaka@nec.com}{\texttt{ohsaka@nec.com}}}}
\author[1]{Tatsuya Matsuoka\thanks{\href{mailto:ta.matsuoka@nec.com}{\texttt{ta.matsuoka@nec.com}}}}
\affil[1]{NEC Corporation}
\date{\today}

\begin{document}

\maketitle

\begin{abstract}%
We consider the \emph{product of determinantal point processes} (DPPs),
a point process whose probability mass is
proportional to the product of principal minors of \emph{multiple} matrices,
as a natural, promising generalization of DPPs.
We study the computational complexity
of computing its \emph{normalizing constant},
which is among the most essential probabilistic inference tasks.
Our complexity-theoretic results (almost) rule out the existence of efficient algorithms for this task
unless the input matrices are forced to have favorable structures.
In particular, we prove the following:
\begin{itemize}
    \item Computing $ \sum_{S} \det(\mat{A}_{S,S})^p $ exactly for every (fixed) \emph{positive even integer} $p$ is \UP-hard and \ModkP{3}-hard,
    which gives a negative answer to an open question posed by \citet{kulesza2012determinantal}.
    
    \item $ \sum_{S} \det(\mat{A}_{S,S}) \det(\mat{B}_{S,S}) \det(\mat{C}_{S,S}) $
    is \NP-hard to approximate within a factor of $ 2^{\bigO(|\calI|^{1-\epsilon})} $ or $2^{\bigO(n^{1/\epsilon})}$ for any $\epsilon > 0$,
    where $|\calI|$ is the input size and $n$ is the order of the input matrix.
    This result is stronger than
    the \shP-hardness for the case of \emph{two} matrices derived by \citet{gillenwater2014approximate}.
    
    \item There exists a $ k^{\bigO(k)} n^{\bigO(1)} $-time algorithm for
    computing $\sum_{S} \det(\mat{A}_{S,S}) \det(\mat{B}_{S,S})$, where
    $k$ is the maximum rank of $\mat{A}$ and $\mat{B}$ or
    the treewidth of the graph formed by nonzero entries of $\mat{A}$ and $\mat{B}$.
    Such parameterized algorithms are said to be \emph{fixed-parameter tractable}.
\end{itemize}
These results can be extended to the \emph{fixed-size} case.
Further, we present two applications of fixed-parameter tractable algorithms given a matrix $\mat{A}$ of treewidth $w$:
\begin{itemize}
    \item We can compute a $2^{\frac{n}{2p-1}}$-approximation to $\sum_S \det(\mat{A}_{S,S})^p$ for any \emph{fractional} number $p>1$ in $w^{\bigO(wp)} n^{\bigO(1)}$ time.
    
    \item We can find a $2^{\sqrt{n}}$-approximation to \emph{unconstrained maximum a posteriori inference} in $w^{\bigO(w \sqrt{n})}n^{\bigO(1)}$ time.
\end{itemize}
\end{abstract}

\clearpage

\tableofcontents

\section{Introduction}
\label{sec:intro}

\emph{Determinantal point processes} (DPPs) offer an appealing probabilistic model to compactly express negative correlation among combinatorial objects
\citep{macchi1975coincidence,borodin2005eynard}.
Given an $n \times n$ matrix $\mat{A}$,
a DPP defines a probability distribution on $2^{[n]}$ such that
the probability of drawing a particular subset $S \subseteq [n]$ is proportional to
the principal minor $\det(\mat{A}_{S,S})$.
Consider the following subset selection task:
given $n$ items (e.g., images \citealp{kulesza2011kdpps}) associated with
quality scores $q_i$ and feature vectors $ \bm{\phi}_i $ for each $i \in [n]$,
we are asked to choose a small group of high-quality, diverse items.
One can construct $\mat{A}$ as $A_{i,j} = q_i q_j \bm{\phi}_i^\top \bm{\phi}_j $,
resulting in that
$ \det(\mat{A}_{S,S}) $
is the squared volume of the parallelepiped spanned by  $\{ q_i \bm{\phi}_i \}_{i \in S}$,
which balances item quality and set diversity.
With the development of efficient algorithms for many inference tasks,
such as normalization, sampling, and marginalization,
DPPs have come to be applied to numerous machine learning tasks, e.g.,
image search \citep{kulesza2011kdpps},
video summarization \citep{gong2014diverse},
object retrieval \citep{affandi2014learning},
sensor placement \citep{krause2008near}, and
the Nystr\"{o}m method \citep{li2016fasta}.

One of the recent research trends is to extend or generalize DPPs
to express more complicated distributions.
Computing the \emph{normalizing constant} (a.k.a.~the \emph{partition function})
for such new models is at the heart of efficient probabilistic inference.
For example, we can efficiently sample a subset from
partition DPPs \citep{celis2017complexity},
which are restricted to including a fixed number of elements from each prespecified group,
by quickly evaluating their normalizing constant.
Such tractability is, of course, not necessarily the case for every generalization.

In this paper, we consider a natural, (seemingly) promising generalization of DPPs involving \emph{multiple} matrices.
The \emph{product DPP} ($\Pi$-DPP) of $m$ matrices $ \mat{A}^1, \ldots, \mat{A}^m $ of size $n \times n$
defines the probability mass for each subset $S$ as proportional to
$ \det(\mat{A}^1_{S,S}) \cdots \det(\mat{A}^m_{S,S}) $,\footnote{
Here, $\mat{A}^1, \ldots, \mat{A}^m$ do not denote a power of matrix $\mat{A}$ but simply denotes (possibly) distinct matrices.
}
which can be significantly expressive:
it enables us to \emph{embed} some constraints in DPPs, e.g.,
those that are defined by partitions \citep{celis2017complexity} and bipartite matching, and
it contains \emph{exponentiated DPPs} \citep{mariet2018exponentiated} of an integer exponent as a special case.
The computational complexity of its normalizing constant, i.e.,
\begin{align*}
\ZZ_m(\mat{A}^1, \ldots, \mat{A}^m) \triangleq \sum_{S \subseteq [n]} \det(\mat{A}^1_{S,S}) \cdots \det(\mat{A}^m_{S,S}),
\end{align*}
is almost nebulous,
except for $m \leq 2$
\citep[see \cref{subsec:intro:related-work}]{gurvits2005complexity,gurvits2009complexity,gillenwater2014approximate,anari2017generalization}.
Our research question is thus the following:

\begin{framed}
\centering
\emph{How hard (or easy) is it to compute normalizing constants for $\Pi$-DPPs?}
\end{framed}

\subsection{Our Contributions}
We present an intensive study on
the computational complexity of computing the normalizing constant for $\Pi$-DPPs.
Our quest can be partitioned into five investigations: intractability, inapproximability, fixed-parameter tractability, extensions to the fixed-size version, and applications.
Our complexity-theoretic results in \cref{sec:intract,sec:inapprox,sec:fpt} summarized in \cref{tab:summary}
(almost) rule out the existence of efficient algorithms for this computation problem
unless the input matrices are forced to have favorable structures.
We also demonstrate two fundamental properties of $\Pi$-DPPs (\cref{sec:property}).
We refer the reader to \cref{subsec:pre:complexity} for a brief introduction to complexity classes.
The paragraph headings begin
with a \xmark mark for negative (i.e., hardness) results and
with a \cmark mark for positive (i.e., algorithmic) results.

\subsection*{\underline{Contribution 1: Intractability (\cref{sec:intract})}}
We analyze the hardness of computing normalizing constants \emph{exactly}.
Computing $\ZZ_2(\mat{A}, \mat{B}) = \sum_S \det(\mat{A}_{S,S}) \det(\mat{B}_{S,S})$ for two positive semi-definite matrices $\mat{A}$ and $\mat{B}$ is known to be \shP-hard \citep{gillenwater2014approximate}.\footnote{
\shP is the class of function problems of
counting the number of accepting paths of a nondeterministic polynomial-time Turing machine (\NP machine), and hence it holds that \NP $\subseteq$ \shP.}

\paragraph{\xmark \textcolor{\nicered}{Exponentiated DPPs.}}
Our first target is a special case where $ \mat{A}^i = \mat{A} $ for all $i$;
i.e., the probability mass for subset $S \subseteq [n]$ is proportional to $\det(\mat{A}_{S,S})^p$ for some integer $p$,
which includes \emph{exponentiated DPPs} (E-DPPs) \citep{mariet2018exponentiated} of an integer exponent.
The diversity preference can be controlled via exponent parameter $p$:
increasing $p$ prefers more diverse subsets than DPPs, while
setting $p=0$ results in a uniform distribution.
The original motivation for computing of the normalizing constant
$ \sum_{S} \det(\mat{A}_{S,S})^p $
is the Hellinger distance between two DPPs \citep{kulesza2012determinantal}.
We prove that for every (fixed) positive even integer $p = 2,4,6,\ldots$,
it is \UP-hard%
\footnote{
\UP is the class of decision problems solvable by an \NP machine with at most one accepting path;
it holds that \cP $\subseteq$ \UP $\subseteq$ \NP.
}
and \ModkP{3}-hard%
\footnote{
\ModkP{3} is the class of decision problems solvable by an \NP machine,
where the number of accepting paths is not divisible by $3$.
}
to compute this normalizing constant,
even when $\mat{A}$ is a $(-1,0,1)$-matrix or a P-matrix (\cref{cor:pow-hard}).
Hence, no polynomial-time algorithm exists for it unless
both \textsc{IntegerFactorization} $\in$ \UP and
\textsc{GraphIsomorphism} $\in$ \ModkP{3} are polynomial-time solvable.
In particular, \UP-hardness excludes the existence of any polynomial-time algorithm unless \RP~$=$~\NP \citep{valiant1986np}.
Our result gives a negative answer to an open question posed by \citet[Section 7.2]{kulesza2012determinantal}.
We must emphasize that
\citet{gurvits2005complexity,gurvits2009complexity} already proved
the \shP-hardness of computing $\sum_S \det(\mat{A}_{S,S})^2$
(see \cref{subsec:intro:related-work}).

\subsection*{\underline{Contribution 2: Inapproximability (\cref{sec:inapprox})}}
After gaining an understanding of the hardness of exact computation,
we examine the possibility of \emph{approximation}.
Our hope is to guess an accurate estimate; e.g.,
an $\rme^n$-approximation is possible for the case of two matrices \citep{anari2017generalization}.

\paragraph{\xmark \textcolor{\nicered}{(Sub)exponential Inapproximability for the Case of Three Matrices.}}
Unfortunately, our hopes are dashed:
we prove that
it is \NP-hard to approximate the normalizing constant for $\Pi$-DPPs of three matrices,
i.e., $\ZZ_3(\mat{A}, \mat{B}, \mat{C}) = \sum_S \det(\mat{A}_{S,S})\det(\mat{B}_{S,S})\det(\mat{C}_{S,S})$,
within a factor of $2^{\bigO(|\calI|^{1-\epsilon})}$ or $2^{\bigO(n^{1/\epsilon})}$ for any $\epsilon > 0 $ even when $\mat{A}$,~$\mat{B}$, and~$\mat{C}$ are positive semi-definite,
where $|\calI|$ is the number of bits required to represent the three matrices
(\cref{thm:inapprox-3}).
For instance, even a $2^{n^{100}}$-approximation cannot be expected in polynomial time.
Moreover, unless \RP~$=$~\NP, \emph{approximate sampling} is impossible; i.e.,
we cannot generate a sample (in polynomial time) from a distribution whose total variation distance from 
the $\Pi$-DPP defined by $\mat{A}$,~$\mat{B}$, and~$\mat{C}$ is at most $\frac{1}{3}$.
The same hardness results hold for the case of four or more matrices (i.e., $m \geq 4$).
On the other hand, a simple guess of the number $1$ is proven to be
a $ 2^{\bigO(|\calI|^2)} $-approximation (\cref{obs:approx-3}).

\paragraph{\xmark \textcolor{\nicered}{Approximation-Preserving Reduction from Mixed Discriminant to the Case of Two Matrices.}}
We devise an approximation-preserving reduction from the \emph{mixed discriminant} to the normalizing constant $\ZZ_2(\mat{A}, \mat{B})$ for the $\Pi$-DPP of two matrices
(\cref{thm:AP-D-ZZ2}).
Not only is the mixed discriminant \shP-hard to compute,
but no fully polynomial-time randomized approximation scheme (FPRAS)\footnote{
An FPRAS is a randomized algorithm 
that outputs an $\rme^{\epsilon}$-approximation with probability at least $\frac{3}{4}$ and runs in polynomial time in the input size and $\epsilon^{-1}$ (see \cref{def:fpras}).
}
is also currently known, and
its existence is rather doubtful \citep{gurvits2005complexity};
hence, the approximation-preserving reduction tells us that $\ZZ_2$ is unlikely to admit an FPRAS.

\subsection*{\underline{Contribution 3: Fixed-Parameter Tractability (\cref{sec:fpt})}}
We now resort to \emph{parameterization},
which has recently succeeded in overcoming the difficulty of machine learning problems \citep{ganian2018parameterized,eiben2019parameterized}.
Parameterized complexity \citep{downey2012parameterized} is a research field aiming to
classify (typically, \NP-hard) problems based on their computational complexity
with respect to some parameters.
Given a \emph{parameter} $k$ that may be independent of the input size $|\calI|$,
we say that a problem is \emph{fixed-parameter tractable} (FPT)
if it is solvable in $\textcolor{\nicegreen}{f(k)} |\calI|^{\bigO(1)}$ time
for some computable function $\textcolor{\nicegreen}{f}$.
On the other hand,
a problem solvable in $|\calI|^{\textcolor{\nicegreen}{f(k)}}$ time is \emph{slice-wise polynomial} (XP).
While both have polynomial runtimes for every fixed $k$,
the polynomial part is dramatically different between them ($|\calI|^{\bigO(1)}$ versus $|\calI|^{\textcolor{\nicegreen}{f(k)}}$).
Selecting appropriate parameters is vital to devising fixed-parameter tractability.
We introduce three parameters; the first two turn out to be FPT, and
the third is unlikely to be even XP.

\paragraph{\cmark \textcolor{\niceblue}{(1) Maximum Rank $\rightarrow$ FPT.}}
\emph{Rank} is a natural parameter for matrices.
We can assume bounded-rank matrices for DPPs if
the feature vectors $\bm{\phi}_i$ are low-dimensional \citep{celis2018fair}, or
the largest possible subset is far smaller than the ground set size $n$; e.g.,
\citet*{gartrell2017low}
learned a matrix factorization of rank $15$ for $n \approx 2{,}000$ by using real-world data.
We prove that there exists an $\textcolor{\nicegreen}{r^{\bigO(r)}} n^{\bigO(1)}$-time FPT algorithm for
computing the normalizing constant for two $n \times n$ positive semi-definite matrices $\mat{A}$ and $\mat{B}$,
where $r$ is the \emph{maximum rank} of $\mat{A}$ and $\mat{B}$ (\cref{thm:fpt-rank-2}).
The central idea is to decompose $\mat{A}$ and $\mat{B}$ into 
$n \times r$ rectangular matrices
and then apply the Cauchy--Binet formula.
Our FPT algorithm can be generalized to
the case of $m$ matrices of rank at most $r$,
increasing the runtime to $\textcolor{\nicegreen}{r^{\bigO(mr)}} n^{\bigO(1)}$ (\cref{thm:fpt-rank-m}).

\paragraph{\cmark \textcolor{\niceblue}{(2) Treewidth of Union $\rightarrow$ FPT.}}
\emph{Treewidth} \citep{halin1976s,robertson1986graph,arnborg1989linear,bertele1972nonserial} is one of the most important graph-theoretic parameters;
it measures the ``tree-likeness'' of a graph
(see \cref{def:treedecomp}).
Many \NP-hard problems on graphs have been shown to be FPT when parameterized by the treewidth \citep{cygan2015parameterized,fomin2010exact}.
Informally, the treewidth of a matrix is
that of the graph formed by the \emph{nonzero entries} in the matrix;
e.g., matrices of \emph{bandwidth} $b$ have treewidth $\bigO(b)$.
If feature vectors $\bm{\phi}_i$ exhibit clustering properties \citep{vandermaaten2008visualizing},
the similarity score $\bm{\phi}_i^\top \bm{\phi}_j$ between items from different clusters would be negligibly small,
and such entries can be discarded to obtain a small-bandwidth matrix.
In the context of change-point detection applications,
\citet*{zhang2016block} observe a small-bandwidth matrix to efficiently solve maximum a posteriori inference on DPPs.
We prove that there exists a $ \textcolor{\nicegreen}{w^{\bigO(w)}} n^{\bigO(1)} $-time FPT algorithm
for computing the normalizing constant $\ZZ_2(\mat{A}, \mat{B})$ for two matrices $\mat{A}$ and $\mat{B}$,
where $w$ is the treewidth of the \emph{union} of nonzero entries in $\mat{A}$ and $\mat{B}$ (\cref{thm:fpt-treewidth-2}).
The proof is based on dynamic programming,
which is a typical approach but requires complicated procedures.
Our FPT algorithm can be generalized
to the case of $m$ matrices, increasing the runtime
to $ \textcolor{\nicegreen}{w^{\bigO(mw)}} n^{\bigO(1)} $ (\cref{thm:fpt-treewidth-m}).

\paragraph{\xmark \textcolor{\nicered}{(3) Maximum Treewidth $\rightarrow$ Unlikely to be XP.}}
Our FPT algorithm in \cref{thm:fpt-treewidth-2} implicitly
benefits from the fact that
$\mat{A}$ and $\mat{B}$ have nonzero entries in similar places.
So, what happens if $\mat{A}$ and $\mat{B}$ are structurally different?
Can we still get FPT algorithms when the parameterization is by the \emph{maximum} treewidth of $\mat{A}$ and $\mat{B}$?
The answer is no:
computing the normalizing constant $\ZZ_2(\mat{A}, \mat{B})$ is \shP-hard even if
\emph{both} $\mat{A}$ and $\mat{B}$ have treewidth at most $3$
(\cref{thm:treewidth-hard}), implying that
even XP algorithms do not exist unless \FP~$=$~\shP
(which is a requirement that is at least as strong as \cP~$=$~\NP).

\afterpage{
\clearpage 
\thispagestyle{empty} 
\begin{landscape}

\begin{table*}[tbp]
    \centering
    \caption{
    Summary of complexity-theoretic results presented in \cref{sec:intract,sec:inapprox,sec:fpt} and in previous work.
    Our positive and negative results are marked with \cmark and \xmark, respectively.
    $ \ZZ_m(\mat{A}^1, \ldots, \mat{A}^m) \triangleq \sum_{S \subseteq [n]} \det(\mat{A}^1_{S,S}) \cdots \det(\mat{A}^m_{S,S}) $
    denotes the normalizing constant for $\Pi$-DPPs,
    $n$ denotes the order of input matrices,
    $|\calI|$ is the number of bits required to represent 
    $\mat{A}^1, \ldots, \mat{A}^m$,
    $\nnz$ is the set of nonzero entries in a matrix, and
    $\tw$ is the treewidth (see Section~\ref{sec:pre}).
    }
    \label{tab:summary}
    \setlength{\tabcolsep}{2pt}
    \small
    \begin{tabular}{|c|c|c|c|c|}
        \hline
        \rowcolor{gray!20}
        \textbf{exact/approx./parameters} & $\ZZ_p(\mat{A}, \ldots, \mat{A})$ \scriptsize{exponentiated DPP}
        & $\ZZ_2(\mat{A}, \mat{B})$
        & $\ZZ_3(\mat{A}, \mat{B}, \mat{C})$
        & $\ZZ_m(\mat{A}^1, \ldots, \mat{A}^m)$ \\
        \hline

        & \cellcolor{\nicered!15} \xmark \UP-hard \& \ModkP{3}-hard
        & \multicolumn{3}{c|}{\shP-hard}
        \\
        \multirow{-2}{*}{\textbf{exact}}
        & \cellcolor{\nicered!15} (\textbf{\cref{cor:pow-hard}}, \scriptsize{$p=2,4,6,\ldots$})
        & \multicolumn{3}{c|}{\citep{gurvits2005complexity,gurvits2009complexity,gillenwater2014approximate}
        }
        \\
        \hline

        & \multicolumn{2}{c|}{$ \rme^n $-approx.~in polynomial time}
        & \multicolumn{2}{c|}{\cellcolor{\nicered!15}
        \xmark $2^{\bigO(|\calI|^{1-\epsilon})}$-approx.~is \NP-hard (\textbf{\cref{thm:inapprox-3}})} \\
        
        \hhline{|~|~~|--}
        \multirow{-2}{*}{\textbf{approximation}}
        & \multicolumn{2}{c|}{\citep{anari2017generalization}}
        & \multicolumn{2}{c|}{\cellcolor{\niceblue!15} \cmark $1$ is $2^{\bigO(|\calI|^2)}$-approx.~(\textbf{\cref{obs:approx-3}})} \\
        \hline

        $m =$ number of matrices &
        \cellcolor{\niceblue!15}
        & \cellcolor{\niceblue!15} \cmark \FPT; $ r^{\bigO(r)}n^{\bigO(1)} $ time
        & \cellcolor{\niceblue!15} & \cellcolor{\niceblue!15} \cmark \FPT; $ r^{\bigO(mr)} n^{\bigO(1)} $ time \\
        $ r = \max_{i \in [m]} \rank(\mat{A}^i) $
        & \cellcolor{\niceblue!15} \multirow{-2}{*}{\cmark (special case of $\rightarrow$)}
        & \cellcolor{\niceblue!15} (\textbf{\cref{thm:fpt-rank-2}}) &
        \cellcolor{\niceblue!15} \multirow{-2}{*}{\cmark (special case of $\rightarrow$)}
        & \cellcolor{\niceblue!15} (\textbf{\cref{thm:fpt-rank-m}}) \\
        \hline

        $m=$ number of matrices
        & \cellcolor{\niceblue!15}
        & \cellcolor{\niceblue!15} \cmark \FPT; $w^{\bigO(w)} n^{\bigO(1)}$ time
        & \cellcolor{\niceblue!15}
        & \cellcolor{\niceblue!15} \cmark \FPT; $w^{\bigO(mw)}n^{\bigO(1)}$ time \\
        $ w = \tw \bigl(\bigcup_{i \in [m]} \nnz(\mat{A}^i) \bigr) $
        & \cellcolor{\niceblue!15} \multirow{-2}{*}{\cmark (special case of $\rightarrow$)}
        & \cellcolor{\niceblue!15} (\textbf{\cref{thm:fpt-treewidth-2}}) &
        \cellcolor{\niceblue!15} \multirow{-2}{*}{\cmark (special case of $\rightarrow$)}
        & \cellcolor{\niceblue!15} (\textbf{\cref{thm:fpt-treewidth-m}}) \\
        \hline
        
        $m=$ number of matrices
        & \cellcolor{\niceblue!15}
        & \multicolumn{3}{c|}{\cellcolor{Red3!20} \xmark \shP-hard ($w \leq 3, m=2$)} \\
        $ w = \max_{i \in [m]} \tw(\nnz(\mat{A}^i)) $
        & \cellcolor{\niceblue!15} \multirow{-2}{*}{\cmark (same as $\uparrow$)}
        & \multicolumn{3}{c|}{\cellcolor{Red3!20} (\textbf{\cref{thm:treewidth-hard}})}
        \\
        \hline
    \end{tabular}
\end{table*}

\end{landscape}
\clearpage
}

\subsection*{\underline{Contribution 4: Extensions to Fixed-Size $\Pi$-DPPs (\cref{sec:kdpp})}}
We extend the complexity-theoretic results devised so far to
the case of \emph{fixed-size} $\Pi$-DPPs.
Given $m$ matrices $\mat{A}^1, \ldots, \mat{A}^m$ of size $n \times n$ and a size parameter $k$, the \emph{$k\Pi$-DPP}
specifies the probability mass for each subset $S \subseteq [n]$ to be proportional to $\det(\mat{A}_{S,S})$ \emph{only if} $|S| = k$.
Thus, the normalizing constant for $k\Pi$-DPP is given by
\begin{align*}
    \sum_{S \subseteq [n]: |S|=k} \det(\mat{A}^1_{S,S}) \cdots \det(\mat{A}^m_{S,S}).
\end{align*}
The special case of $m=1$ coincides with \emph{$k$-DPPs} \citep{kulesza2011kdpps}, which are among the most important extensions of DPPs
because of their applications to image search \citep{kulesza2011kdpps} and the Nystr\"{o}m method \citep{li2016fasta}.
We derive complexity-theoretic results for $k\Pi$-DPPs 
corresponding to those on $\Pi$-DPPs presented in \cref{sec:intract,sec:inapprox,sec:fpt}, including
intractability of E-DPPs,
indistinguishability of $\ZZ_3$,
an approximation-preserving reduction from the mixed discriminant to $\ZZ_2$
(\cref{thm:kdpp-intract-inapprox}), and
FPT algorithms parameterized by maximum rank and treewidth
(\cref{cor:kdpp-fpt}).
Further, we examine the fixed-parameter tractability of
computing the normalizing constant for $k\Pi$-DPPs parameterized by $k$.
It is easy to check that
a brute-force algorithm runs in $n^{\bigO(k)}$ time, which is XP.
One might further expect there to be an FPT algorithm; however,
we show that computing the normalizing constant $\ZZ_2$ parameterized by $k$ is \shW{1}-hard
(\cref{thm:kdpp-size-hard}).
Since it is a plausible assumption that \FPT~$\neq$~\shW{1} in parameterized complexity \citep{flum2004parameterized},
the problem of interest is unlikely to be FPT.

\subsection*{\underline{Contribution 5: Applications of Parameterized Algorithms (\cref{sec:app})}}
Finally, we apply the FPT algorithm to two related problems,
which bypasses the complexity-theoretic barrier in the general case.
Hereafter, we will denote by $w$ the treewidth of an $n \times n$ matrix $\mat{A}$.

\paragraph{\cmark \textcolor{\niceblue}{Approximation Algorithm for E-DPPs of Fractional Exponents.}}
The first application is approximating the normalizing constant for E-DPPs of \emph{fractional} exponent $p > 1$, i.e., $\sum_{S \subseteq [n]} \det(\mat{A}_{S,S})^p$.
Our FPT algorithm (\cref{thm:fpt-treewidth-m}) does not directly apply to this case because
an E-DPP is a $\Pi$-DPP only if $p$ is \emph{integer}.
Generally,
it is possible to compute a $2^{n(p-1)}$-approximation to this quantity efficiently (see \cref{rmk:edpp-frac}),
which is tight up to constant in the exponent \citep{ohsaka2021unconstrained}.
On the other hand, it is known that there is an FPRAS if $p < 1$ (see \cref{subsec:intro:related-work}; \citealp{anari2019log,robinson2019flexible}).
Intriguingly, we can compute a $2^{\frac{n}{2p-1}}$-approximation for a P$_0$-matrix $\mat{A}$ in $w^{\bigO(wp)} n^{\bigO(1)}$ time (\cref{thm:edpp-frac}).
This results is apparently strange since the accuracy of estimation improves with the value of $p$.
The idea behind the proof is to compute the normalizing constant for E-DPPs of exponent $\floor{p}$ or $\ceil{p}$ by exploiting \cref{thm:fpt-treewidth-m},
either of which ensures the desired approximation.

\paragraph{\cmark \textcolor{\niceblue}{Subexponential Algorithm for Unconstrained MAP Inference.}}
The second application is \emph{unconstrained maximum a posteriori (MAP) inference} on DPPs,
which is equivalent to finding a principal submatrix having the maximum determinant,
i.e., to compute $\argmax_{S \subseteq [n]} \det(\mat{A}_{S,S})$.
The current best approximation factor for MAP inference is $2^{\bigO(n)}$ \citep{nikolov2015randomized},
which is optimal up to constant in the exponent \citep[see \cref{subsec:intro:related-work}]{ohsaka2021unconstrained}.
Here, we present a $w^{\bigO(w\sqrt{n})}n^{\bigO(1)}$-time randomized algorithm that approximates unconstrained MAP inference within a factor of $2^{\sqrt{n}}$ (\cref{thm:map-approx}).
In particular, if $w$ is a constant independent of $n$,
we can obtain a \emph{subexponential-time} and \emph{subexponential-approximation} algorithm, which is
a ``sweet spot'' between 
a $2^{n}n^{\bigO(1)}$-time exact (brute-force) algorithm and
a polynomial-time $2^{\bigO(n)}$-approximation algorithm \citep{nikolov2015randomized}.
The proof uses \cref{thm:fpt-treewidth-m} to generate a sample from an E-DPP of a sufficiently large exponent $p = \bigO(\sqrt{n})$,
for the desired guarantee of approximation with high probability.

\subsection{Related Work}
\label{subsec:intro:related-work}

\emph{Exponentiated DPPs} (E-DPPs) of exponent parameter $p > 0$ define the probability mass for each subset $S$
as proportional to $ \det(\mat{A}_{S,S})^p $.
A Markov chain Monte Carlo algorithm on E-DPPs for $p < 1$ is proven to mix rapidly as E-DPPs are strongly log-concave as shown by
\citet*{anari2019log,robinson2019flexible}, implying
an FPRAS for the normalizing constant.
\citet*{mariet2018exponentiated} investigate the case when
a DPP defined by the $p$-th power of $\mat{A}$ is close to an E-DPP of exponent $p$ for $\mat{A}$.
Quite surprisingly, \citet{gurvits2005complexity,gurvits2009complexity} proves
the \shP-hardness of exactly computing $\sum_S \det(\mat{A}_{S,S})^2$ for a P-matrix $\mat{A}$
before the more recent study by \citet{kulesza2012determinantal,gillenwater2014approximate};
this result seems to be not well known in the machine learning community.
\citet{gillenwater2014approximate} proves that computing the normalizing constant for $\Pi$-DPPs defined by \emph{two positive semi-definite} matrices is \shP-hard, while
\citet{anari2017generalization} prove that it is approximable within a factor of $\rme^n$ in polynomial time,
which is an affirmative answer to an open question posed by \citet{kulesza2012determinantal}.
\citet{ohsaka2021unconstrained,ohsaka2021some} proves that
it is \NP-hard to approximate the normalizing constant for E-DPPs
within a factor of $2^{\beta n p}$ when $p \geq \beta^{-1}$,
where $\beta = 10^{-10^{13}}$.
Our study strengthens these results by showing
the hardness for an E-DPP of exponent $p=2,4,6,\ldots$, and
the impossibility of an exponential approximation and approximate sampling for \emph{three} matrices.

$\Pi$-DPPs can be thought of as \emph{log-submodular point processes} \citep{djolonga2014map,gotovos2015sampling},
whose probability mass for a subset $S$ is proportional to $\exp(f(S)) $, where $f$ is a submodular set function.\footnote{
We say that a set function $f: 2^{[n]} \to \bbR$ is \emph{submodular} if
$ f(S) + f(T) \geq f(S \cup T) + f(S \cap T) $
for all $ S, T \subseteq [n] $.
}
Setting $ f(S) \triangleq \log \det(\mat{A}^1_{S,S}) + \cdots + \log \det(\mat{A}^m_{S,S}) = \log (\det(\mat{A}^1_{S,S}) \cdots \det(\mat{A}^m_{S,S})) $ coincides with $\Pi$-DPPs.
\citet*{gotovos2015sampling} devised a bound on the mixing time of a Gibbs sampler,
though this is not very helpful in our case because
$f$ can take $ \log(0) = -\infty $ as a value.

\emph{Constrained DPPs} output
a subset $S$ with probability proportional to $\det(\mat{A}_{S,S})$ if
$S$ satisfies specific constraints, e.g.,
size constraints \citep{kulesza2011kdpps},
partition constraints \citep{celis2018fair},
budget constraints \citep{celis2017complexity}, and
spanning-tree constraints \citep{matsuoka2021spanning}.
Note that $\Pi$-DPPs can express partition-matroid constraints.

\emph{Maximum a posteriori (MAP) inference} on DPPs finds applications wherein
we seek the most diverse subset.
Unconstrained MAP inference is especially preferable if we do not (or cannot) specify in advance the desired size of the output, as in, e.g.,
tweet timeline generation \citep{yao2016tweet},
object detection \citep{lee2016individualness}, and
change-point detection \citep{zhang2016block}.
On the negative side,
it is \NP-hard to approximate size-constrained MAP inference within a factor of $2^{ck}$ for the output size $k$ and some number $c > 0$ \citep{koutis2006parameterized,civril2013exponential,summa2014largest}.
\citet*{kulesza2012determinantal} prove an inapproximability factor of ($\frac{9}{8}-\epsilon$) for unconstrained MAP inference for any $\epsilon > 0$,
which has since been improved to $2^{\beta n}$ for $\beta = 10^{-10^{13}}$ \citep{ohsaka2021unconstrained}.
On the algorithmic side,
\citet*{civril2009selecting} prove that
the standard greedy algorithm for size-constrained MAP inference achieves an approximation factor of $k!^2 = 2^{\bigO(k \log k)}$, where $k$ denotes the output size.
\citet*{nikolov2015randomized} gives
an $\rme^{k}$-approximation algorithm for size-constrained MAP inference; this is the current best approximation factor.
Note that invoking \citeauthor{nikolov2015randomized}'s algorithm
for all $k \in [n]$ immediately yields
an $\rme^{n}$-approximation for unconstrained MAP inference.
The greedy algorithm is widely used in the machine learning community
because it efficiently extracts reasonably diverse subsets in practice \citep{yao2016tweet,zhang2016block}.
Other than the greedy algorithm,
\citet*{gillenwater2012near} propose a gradient-based algorithm;
\citet*{zhang2016block} develop a dynamic-programming algorithm designed for small-bandwidth matrix, which has no provable approximation guarantee.

This article is an extended version of our conference paper presented at the 37th International Conference on Machine Learning \citep{ohsaka2020intractability}.
It includes the following new results:
\begin{itemize}
\item \cref{sec:property} proves two fundamental properties of $\Pi$-DPPs:
(1) we can generate a sample from $\Pi$-DPPs in polynomial time if we are given access to an oracle for the normalizing constant (\cref{thm:sampling});
(2) $\ZZ_m$ either admits an FPRAS or cannot be approximated within a factor of $2^{n^\delta}$ for any $\delta \in (0,1)$ (\cref{thm:all-or-nothing}).

\item \cref{subsec:inapprox:ap-red} presents an approximation-preserving reduction
from the mixed discriminant to the normalizing constant for $\Pi$-DPPs of two matrices, rather than the polynomial-time Turing reduction presented in \citet*{ohsaka2020intractability}.

\item \cref{sec:kdpp} extends the complexity-theoretic results in \cref{sec:intract,sec:inapprox,sec:fpt} to fixed-size $\Pi$-DPPs.

\item \cref{sec:app} introduces two applications of the FPT algorithms:
(1) a $w^{\bigO(wp)}n^{\bigO(1)}$-time algorithm that approximates the normalizing constant for E-DPPs of any fractional exponent $p > 1$ within a factor of $2^{\frac{n}{2p-1}}$ (\cref{thm:edpp-frac});
(2) a $w^{\bigO(w\sqrt{n})}n^{\bigO(1)}$-time (randomized) algorithm that approximates unconstrained MAP inference within a factor of $2^{\sqrt{n}}$ (\cref{thm:map-approx}),
where $n$ is the order of an input matrix and
$w$ is the treewidth of the matrix.

\end{itemize}

\section{Preliminaries}
\label{sec:pre}

\subsection{Notations}
For two integers $m,n \in \bbN$ such that $m \leq n$,
let $ [n] \triangleq \{1, 2, \ldots, n\} $ and
$[m\isep n] \triangleq \{m,m+1,\ldots, n-1,n\}$.
The imaginary unit is denoted $ \iu = \sqrt{-1}$.
For a finite set $S$ and an integer $k \in [0\isep |S|]$,
we write $ {S \choose k} $ for the family of all size-$k$ subsets of $S$.
For a statement $P$, $[\![P]\!]$ is $1$ if $P$ is true, and $0$ otherwise.
The symbol $\uplus$ is used to emphasize that the union is taken over two \emph{disjoint} sets.
The symmetric group on $[n]$,
consisting of all permutations over $[n]$, is denoted $\Sym_n$.
We use $\sigma: S \bij T$ for two same-sized sets $S$ and $T$ to mean a bijection from $S$ to $T$, and
$\sigma|_{X}$ for a set $X$ to denote the restriction of $\sigma$ to $X \cap S$.
We also define
$\sigma(X) \triangleq \{\sigma(x) \mid x \in X\}$ for a set $X \subseteq S$ and
$\sigma^{-1}(Y) \triangleq \{y \mid \sigma(y) \in Y\}$ for a set $Y \subseteq T$.
For a bijection $\sigma: S \bij T$ and an ordering $\prec$ on $S \cup T$,
the \emph{inversion number} and \emph{sign} of $\sigma$ regarding $\prec$ are defined as
\begin{align*}
\inv_{\prec}(\sigma) & \triangleq |\{ (i,j) \mid i \prec j, \sigma(i) \succ \sigma(j) \}|, \\
    \sgn_{\prec}(\sigma) & \triangleq (-1)^{\inv_{\prec}(\sigma)}.
\end{align*}
In particular, if $\sigma$ is a permutation in $\Sym_n$,
there exists $\inv(\sigma) \in \bbN$ and $\sgn(\sigma) \in \{+1,-1\}$ such that
$\inv_{\prec}(\sigma) = \inv(\sigma)$ and $\sgn_{\prec}(\sigma) = \sgn(\sigma)$ for any ordering $\prec$.
We use $S \prec T$ for two sets $S$ and $T$ to mean that $i \prec j$ for all $i \in S$ and $j \in T$.
For two orderings $\prec_1$ on a set $S_1$ and $\prec_2$ on a set $S_2$,
we say that $\prec_1$ and $\prec_2$ \emph{agree on} a set $T \subseteq S_1 \cap S_2$
whenever $i \prec_1 j$ if and only if $i \prec_2 j$ for all $i,j \in T$.
Unless otherwise specified, the base of the logarithm is 2.
For a set $S$ and element $e$, we shall write $S+e$ and $S-e$ as shorthand for $S\cup\{e\}$ and $S\setminus\{e\}$, respectively.

We denote the $n \times n$ identity matrix by $\mat{I}_n$ and the $n \times n$ all-ones matrix by $\mat{J}_n$.
For an $ m \times n $ matrix $\mat{A}$
and two subsets $S \subseteq [m]$ and $T \subseteq [n]$ of indices,
we write
$\mat{A}_{S}$ for the $ |S| \times n $ submatrix
whose rows are the rows of $\mat{A}$ indexed by $S$, and
$\mat{A}_{S,T}$ for the $ |S| \times |T| $ submatrix
whose rows are the rows of $\mat{A}$ indexed by $S$ and
columns are the columns of $\mat{A}$ indexed by $T$.
The \emph{determinant} and \emph{permanent} of a matrix $\mat{A} \in \bbR^{n \times n}$ are defined as
\begin{align*}
    \det(\mat{A}) & \triangleq
    \sum_{\sigma \in \Sym_n} \sgn(\sigma) \prod_{i \in [n]} A_{i,\sigma(i)}, \\
    \per(\mat{A}) & \triangleq
    \sum_{\sigma \in \Sym_n} \prod_{i \in [n]} A_{i,\sigma(i)}.
\end{align*}
In particular,
$ \det(\mat{A}_{S,S}) $ for any $S \subseteq [n]$
is called a \emph{principal minor}.
We define $ \det(\mat{A}_{\emptyset, \emptyset}) \triangleq 1 $.
A symmetric matrix $\mat{A} \in \bbR^{n \times n}$ is called \emph{positive semi-definite} if 
$ \vec{x}^\top \mat{A} \vec{x} \geq 0 $ for all $ \vec{x} \in \bbR^n $.
A matrix $\mat{A} \in \bbR^{n \times n}$ is called a \emph{P-matrix} (resp.~\emph{P$_0$-matrix}) if all of its principal minors are positive (resp.~nonnegative).
A positive semi-definite matrix is a P$_0$-matrix, but not vice versa.
A real-valued matrix $\mat{A}$ is a P-matrix whenever
it has positive diagonal entries and
is row diagonally dominant (i.e., $ |A_{i,i}| > \sum_{j \neq i} |A_{i,j}| $ for all $i$).
For a bijection $\sigma$ from $S \subseteq [n]$ to $T \subseteq [n]$,
we define $\mat{A}(\sigma) \triangleq \prod_{i \in S} A_{i,\sigma(i)}$.

\subsection{Product of Determinantal Point Processes}
Given a matrix $\mat{A} \in \bbR^{n \times n}$,
a \emph{determinantal point process} (DPP) \citep{macchi1975coincidence,borodin2005eynard} is defined as
a probability measure on the power set $2^{[n]}$ whose probability mass for $S \subseteq [n]$ is proportional to $ \det(\mat{A}_{S,S}) $.\footnote{We adopt the L-ensemble form introduced by \citet{borodin2005eynard}.}
Generally speaking, a P$_0$-matrix is acceptable to define a proper probability distribution, while positive semi-definite matrices are commonly used \citep{gartrell2019learning}.
The normalizing constant for a DPP has the following simple closed form \citep{kulesza2012determinantal}:
\begin{align*}
\sum_{S \subseteq [n]} \det(\mat{A}_{S,S}) = \det(\mat{A} + \mat{I}_n).
\end{align*}
Hence, the probability mass for a set $S \subseteq [n]$ is
equal to $ \det(\mat{A}_{S,S}) / \det(\mat{A} + \mat{I}_n) $.
This equality holds for
any (not necessarily symmetric) real-valued matrix $\mat{A}$.

This paper studies a point process whose probability mass is determined from
the \emph{product} of principal minors for multiple matrices.
Given $m$ matrices
$ \mat{A}^1, \ldots, \mat{A}^m \in \bbR^{n \times n} $,
\emph{the product DPP} ($\Pi$-DPP) defines the probability mass for each subset $S \subseteq [n]$ as being proportional to
$\det(\mat{A}^1_{S,S}) \cdots \det(\mat{A}^m_{S,S})$.
We use $\ZZ_m(\mat{A}^1, \ldots, \mat{A}^m)$ to denote its \emph{normalizing constant};
namely,
\begin{align*}
    \ZZ_m(\mat{A}^1, \ldots, \mat{A}^m) \triangleq
    \sum_{S \subseteq [n]} \prod_{i \in [m]} \det(\mat{A}^i_{S,S}).
\end{align*}
In particular, we have that $\ZZ_1(\mat{A}) = \det(\mat{A} + \mat{I}_n) $.
Since $\prod_{i \in [m]} \det(\mat{A}^i_{S,S})$ is easy to compute,
evaluating $\ZZ_m$ is crucial for estimating the probability mass.
Our objective in this paper is to elucidate the computational complexity of estimating $\ZZ_m$.
We shall raise two examples of $\Pi$-DPPs.

\begin{example}[Embedding partition and matching constraints]
Given a partition $\calP$ of $[n]$,
we can build $\mat{A}$ such that
$ \det(\mat{A}_{S,S}) = [\![S$ contains at most one element from each group of $\calP]\!] $
by defining $A_{i,j} = [\![i,j \text{ belong to the same group}]\!]$.
Given a bipartite graph whose edge set is $[n]$,
we can build $\mat{A}$ and $\mat{B}$ such that
$ \det(\mat{A}_{S,S}) \det(\mat{B}_{S,S}) =
[\![S$ has no common vertices$]\!] $ \citep{gillenwater2014approximate};
such an $S$ is called a matching.
\end{example}

\begin{example}[Exponentiated DPPs]
Setting $ \mat{A}^i = \mat{A} $ for all $i \in [m]$,
the $\Pi$-DPP becomes an exponentiated DPP of exponent $p = m \geq 1$,
which sharpens the diversity nature of DPPs
\citep{mariet2018exponentiated}.
\end{example}

\subsection{Graph-Theoretic Concepts}
Here, we briefly introduce the notions and definitions from graph theory
that will play a crucial role in \cref{sec:fpt}.
Let $ G = (V,E)$ be a graph,
where $V$ is the set of vertices, and $E$ is the set of edges.
We use $(u,v)$ to denote an (undirected or directed) edge connecting $u$ and $v$.
Moreover, we define the treewidth of a graph and matrix.
\emph{Treewidth} \citep{halin1976s,robertson1986graph,arnborg1989linear,bertele1972nonserial}
is one of the most important notions in graph theory, which captures the ``tree-likeness'' of a graph.

\begin{definition}[\protect{\citealp[tree decomposition]{robertson1986graph}}]
\label{def:treedecomp}
A \emph{tree decomposition} of an undirected graph $G=(V,E)$ is a pair
$(T, \{X_t\}_{t \in T})$, where
$T$ is a tree of which vertex $t \in T$, referred to as a \emph{node},\footnote{
We will refer to the vertices of $T$ as \emph{nodes} to distinguish them from the vertices of $G$.}
is associated with a vertex set $X_t \subseteq V$,
referred to as a \emph{bag}, such that the following conditions are satisfied:
\begin{itemize}
\item $ \bigcup_{t \in T} X_t = V $;
\item for every edge $(u,v) \in E$,
there exists a node $t \in T$ such that $u, v \in X_t$;
\item for every vertex $v \in V$,
the set $T_u = \{t \mid v \in X_t\}$ induces a connected subtree of $T$.
\end{itemize}
The \emph{width} of a tree decomposition $(T, \{X_t\}_{t \in T})$
is defined as $ \max_{t \in T} |X_t|-1 $.
The \emph{treewidth} of a graph $G$,
denoted $\tw(G)$,
is the minimum possible width among all tree decompositions of $G$.
\end{definition}
For example, a tree has treewidth $1$, an $n$-vertex planar graph has treewidth $\bigO(\sqrt{n})$, and an $n$-clique has treewidth $n-1$.

For an $n \times n$ square matrix $\mat{A}$,
we define $ \nnz(\mat{A}) \triangleq \{ (i, j) \mid A_{i,j} \neq 0, i \neq j \} $.
The treewidth of $\mat{A}$,
denoted $\tw(\nnz(\mat{A}))$ or $ \tw(\mat{A}) $,
is defined as the treewidth of the graph $ ([n], \nnz(\mat{A})) $ formed by the nonzero entries of $\mat{A}$.
See \cref{fig:mat-A,fig:G,fig:td} for an example of a tree decomposition of a matrix.
For example,
$\tw(\mat{I}_n) = 1$,
$\tw(\mat{J}_n) = n-1$, and
a matrix of bandwidth\footnote{
The bandwidth of a matrix $\mat{A}$ is defined as
the smallest integer $b$ such that $A_{i,j} = 0$ whenever $|i-j| > b$.}
$b$ has treewidth $\bigO(b)$.
One important property of tree decompositions is that any bag $X_t$ is a \emph{separator}:
for three nodes $t,t',t''$ of $T$ such that
$t$ is on the (unique) path from $t'$ to $t''$,
$X_t$ separates $X_{t'} \setminus X_t$ and $X_{t''} \setminus X_t$; i.e.,
the submatrices $ \mat{A}_{X_{t'} \setminus X_t, X_{t''} \setminus X_t} $ and $\mat{A}_{X_{t''} \setminus X_t, X_{t'} \setminus X_t}$ must be zero matrices.
It is known that the permanent of bounded-treewidth matrices is polynomial-time computable \citep{courcelle2001fixed}.
Though it is \NP-complete to determine whether an input graph $G$
has treewidth at most $w$,
there exist numerous FPT and approximation algorithms, e.g.,
a $w^{\bigO(w^3)}n$-time exact algorithm \citep{bodlaender1996linear}
and
a $5$-approximation algorithm having faster runtime $2^{\bigO(w)}n$ \citep{bodlaender2016approximation}.

\begin{remark}
Our FPT algorithms parameterized by rank (\cref{subsec:fpt:rank})
and by treewidth
(\cref{subsec:fpt:treewidth})
are not comparable in the sense that
the identity matrix $\mat{I}_n$ has rank $n$ and treewidth $1$ while
the all-ones matrix $\mat{J}_n$ has rank $1$ and treewidth $n-1$.
\end{remark}

\subsection{Computational Models}
We will introduce the notion of input size and computational model carefully
since we use several \emph{reductions} that transform
an input for one problem to an input for another problem.

The \emph{size} of an input $\calI$, denoted $|\calI|$, is
defined as the number of bits required to represent $\calI$.
In particular, we assume that all numbers appearing in this paper are \emph{rational}.\footnote{
We can run the algorithms in \cref{sec:fpt} for real-valued matrices if arithmetic operations on real numbers are allowed.}
The size of
a rational number $ x = p/q \in \bbQ $ (where $p$ and $q$ are relatively prime integers) and
a rational matrix $ \mat{A} \in \bbQ^{m \times n} $ is defined as follows \citep{schrijver1999theory}:
\begin{align*}
\mathrm{size}(x) & \triangleq 1 + \lceil \log(|p|+1) \rceil + \lceil \log(|q|+1) \rceil, \\
\mathrm{size}(\mat{A}) & \triangleq mn + \sum_{i \in [m], j \in [n]} \mathrm{size}(A_{i,j}).
\end{align*}
The size of a graph is defined as the size of its incidence matrix.

Selection of computational models is crucial for determining the runtime of algorithms;
e.g., while multiplying two $n$-bit integers can be done
in $\bigO(n \log n \ 8^{\log^* n})$ time on \emph{Turing machines} \citep{harvey2016even},
we do not need this level of precision.
Thus, for ease of analysis,
we adopt the \emph{unit-cost random-access machine} model of computation,
which can perform basic arithmetic operations
(e.g., add, subtract, multiply, and divide) in constant time.
In other words, we will measure the runtime in terms of the number of operations.
However, abusing unrealistically powerful models
leads to an unreasonable conclusion \citep[Example 16.1]{arora2009computational}:
``iterating $n$ times the operation $ x \leftarrow x^2 $,
we can compute $ 2^{2^n} $, a $2^n$-bit integer,
in $\bigO(n)$ time.''
To avoid such pitfalls,
we will ensure that numbers produced during
the execution of algorithms intermediately are of size $|\calI|^{\bigO(1)}$.

\subsection{Brief Introduction to Complexity Classes}
\label{subsec:pre:complexity}
Here, we briefly introduce the complexity classes appearing throughout this paper.

\paragraph{Decision Problems.}

\begin{itemize}
    \item \cP: The class of decision problems solvable by
    a deterministic polynomial-time Turing machine.
    Examples include \textsc{Primes} (Q.~\emph{Is an input integer prime?}) \citep{agrawal2004primes}.
    
    \item \NP: The class of decision problems solvable by
    a nondeterministic polynomial-time Turing machine (\NP machine).
    Examples include
    \textsc{SAT} (Q.~\emph{Is there a truth assignment satisfying an input Boolean formula?}).
    It is widely believed that \cP~$\neq$~\NP, see, e.g., \citet{arora2009computational}.
    
    \item \RP:
    The class of decision problems for which a probabilistic polynomial-time Turing machine
    exists such that
    (1) if the answer is ``yes,'' then it returns ``yes'' with probability at least $\frac{1}{2}$, and
    (2) if the answer is ``no,'' then it always returns ``no.''
    Note that \cP~$\subseteq$~\RP~$\subseteq$~\NP, and
    it is suspected that \RP~$\neq$~\NP.

    \item \UP: The class of decision problems solvable by an \NP machine
    with \emph{at most one} accepting path.
    Note that \cP~$\subseteq$~\UP~$\subseteq$~\NP, but
    it is unknown if the inclusion is strict.
    Examples include
    \textsc{UnambiguousSAT}
    (Q.~\emph{Is there a truth assignment satisfying
    an input Boolean formula that is restricted to have at most one satisfying assignment?}) and
    \textsc{IntegerFactorization}
    (Q.~\emph{Is there a factor $d \in [m]$ of an integer $n$ given $n$ and $m$?}),
    for which no polynomial-time algorithms are known.
    The Valiant--Vazirani theorem states that if \textsc{UnambiguousSAT} is solvable in  polynomial-time, then \RP~$=$~\NP \citep{valiant1986np}.
\end{itemize}

\paragraph{Counting-Related Problems.}
\begin{itemize}
    \item \FP: The class of 
    functions computable by a deterministic polynomial-time Turing machine,
    which is a function-problem analogue of \cP.
    Examples include \textsc{Determinant}
    (Q.~\emph{Compute the determinant of an input square matrix}),
    which is efficiently-computable via Gaussian elimination \citep{edmonds1967systems,schrijver1999theory}.
    
    \item \shP: The class of function problems of counting the number of accepting paths of an \NP machine.
    Examples of \shP-complete problems include
    \textsc{Permanent} (Q.~\emph{Compute the permanent of an input matrix}) and
    \textsc{$\#$SAT} (Q.~\emph{Compute the number of truth assignments satisfying an input Boolean formula}).
    Note that \cP~$\neq$~\NP implies \FP~$\neq$~\shP \citep{arora2009computational}.
    
    \item \ModkP{k}: The class of decision problems solvable by an \NP machine,
    where the number of accepting paths \emph{is not divisible by $k$}.
    Examples include
    \textsc{GraphIsomorphism}
    (Q.~\emph{Is there an edge-preserving bijection between the vertex sets of two input graphs?}),
    which is in \ModkP{k} for all $k$ \citep{arvind2006graph}.
\end{itemize}

\paragraph{Parameterized Problems.}
\begin{itemize}
    \item \FPT (fixed-parameter tractable):
    The class of problems with \emph{parameter} $k$ solvable in
    $f(k) |\calI|^{\bigO(1)}$ time for some computable function $f$,
    where $|\calI|$ is the input size.
    Examples include
    \textsc{$k$-VertexCover}
    (Q.~\emph{Is there a $k$-vertex set including at least one endpoint of every edge of an $n$-vertex graph?}),
    for which an $ \bigO(1.2738^k + kn) $-time algorithm is known \citep{chen2010improved}.

    \item \XP (slice-wise polynomial):
    The class of problems with \emph{parameter} $k$ solvable in $ |\calI|^{f(k)} $ time
    for some computable function $f$; hence it holds that \FPT~$\subseteq$~\XP.
    Examples include
    \textsc{$k$-Clique} (Q.~\emph{Is there a size-$k$ complete subgraph in an $n$-vertex graph?}),
    for which a brute-force search algorithm runs in $n^{\bigO(k)}$ time.
    It is suspected that \FPT~$\neq$~\XP in parameterized complexity \citep{downey2012parameterized}.
    
    \item \shW{1}:
    The class of function problems parameterized reducible to \#$k$-\textsc{Clique}
    (Q.~\emph{Compute the number of $k$-cliques in an $n$-vertex graph}).
    Note that \FPT~$\subseteq$~\shW{1}~$\subseteq$~\XP.
    It is a plausible assumption in parameterized complexity \citep{flum2004parameterized} that
    \FPT~$\neq$~\shW{1}; i.e.,
    \#$k$-\textsc{Clique} does not admit an FPT algorithm parameterized by $k$.
\end{itemize}

\subsection{Approximation Algorithms}
Here, we introduce some concepts related to approximation algorithms.
We say that an estimate $\widehat{\ZZ}$ is
a \emph{$\rho$-approximation} to some value $\ZZ$ for $\rho \geq 1$ if
\begin{align*}
\ZZ \leq \widehat{\ZZ} \leq \rho \cdot \ZZ.
\end{align*}
The \emph{approximation factor} $\rho$ can be a function of the input size, e.g., $\rho(n) = 2^{n}$;
an (asymptotically) smaller $\rho$ is a better approximation factor.
For a function $f: \Sigma^* \to \bbR $ and an approximation factor $\rho$,
\emph{a $\rho$-approximation algorithm} is a polynomial-time algorithm that returns
a $\rho$-approximation to $f(\calI)$ for every input $\calI \in \Sigma^*$.

We define a fully polynomial-time randomized approximation scheme (FPRAS).
The existence of an FPRAS for a particular problem means that the problem can be efficiently approximated to an arbitrary precision.
\begin{definition}
\label{def:fpras}
For a function $f: \Sigma^* \to \bbR$,
a \emph{fully polynomial-time randomized approximation scheme} (FPRAS)
is a randomized algorithm $\alg$ that takes
an input $\calI \in \Sigma^*$ of $f$ and an error tolerance $\epsilon \in (0,1)$ and
satisfies the following conditions:
\begin{itemize}
    \item for every input $\calI \in \Sigma^*$ and $\epsilon \in (0,1)$,
    it holds that
    \begin{align}
    \label{eq:fpras}
        \Pr_{\alg}\Bigl[ \rme^{-\epsilon} \cdot f(\calI) \leq \alg(\calI) \leq \rme^{\epsilon} \cdot f(\calI) \Bigr] \geq \frac{3}{4},
    \end{align}
    where $\alg(\calI)$ denotes $\alg$'s output on $\calI$;\footnote{
    The constant $\frac{3}{4}$ in \cref{eq:fpras} can be replaced by
    any number in $(\frac{1}{2},1)$ \citep{jerrum1986random}.
    }
    \item the running time of $\alg$ is bounded by
    a polynomial in $|\calI|$ and $\epsilon^{-1}$,
    where $|\calI|$ denotes the size of input $\calI$.
\end{itemize}
\end{definition}

\section{Fundamental Properties of $\Pi$-DPPs}
\label{sec:property}

In this section,
we establish two fundamental properties of $\Pi$-DPPs, i.e.,
(1) exact sampling given oracle access (\cref{thm:sampling}), and
(2) an all-or-nothing nature (\cref{thm:all-or-nothing}).
These properties are common to counting problems that have \emph{self-reducibility} \citep{jerrum2003counting}.

\subsection{Exact Sampling Given Exact Oracle}
\label{subsec:property:sampling}
We will show that if we are given access to an \emph{oracle} that
can (magically) return the value of $\ZZ_m$ in a single step, we can generate a sample from a $\Pi$-DPP defined by any $m$ matrices in polynomial time.

\begin{theorem}
\label{thm:sampling}
Suppose we are given access to an oracle that returns $\ZZ_m$.
Let $\mat{A}^1, \ldots, \mat{A}^m$ be $m$ P$_0$-matrices in $\bbQ^{n \times n}$.
Then, there exists a polynomial-time algorithm that
generates a sample from the $\Pi$-DPP defined by $\mat{A}^1, \ldots, \mat{A}^m$ by calling the oracle for $L \triangleq \bigO(n^2)$ tuples of $m$ matrices.
Furthermore,
if $\{(\mat{A}^{1,\ell}, \ldots, \mat{A}^{m,\ell})\}_{\ell \in [L]}$ denotes the set consisting of the tuples of $m$ matrices
for which the oracle is called (i.e., we call the oracle to evaluate $\ZZ_m(\mat{A}^{1,\ell}, \ldots, \mat{A}^{m,\ell})$ for all $\ell \in [L]$),
then it holds that
$ \rank(\mat{A}^{i,\ell}) \leq \rank(\mat{A}^i) $ and
$ \nnz(\mat{A}^{i,\ell}) \subseteq \nnz(\mat{A}^i)$
for all $i \in [m]$ and $\ell \in [L]$.
\end{theorem}

The proof involves a general sampling procedure using the conditional probability, e.g., \citet{celis2017complexity,jerrum2003counting}.
For two disjoint subsets $Y$ and $N$ of $[n]$ and an element $e$ of $[n]$ not in $Y \uplus N$,
let us consider the following conditional probability:
\begin{align}
\label{eq:condition}
    \Pr_{S \sim \bm{\mu}}\Bigl[e \in S \mid Y \subseteq S, N \cap S = \emptyset\Bigr] =
    \frac{\displaystyle
        \sum_{\substack{S \subseteq [n]: \\ Y+e \subseteq S, N \cap S = \emptyset}} \prod_{i \in [m]} \det(\mat{A}^i_{S,S})
    }{\displaystyle
        \sum_{\substack{S \subseteq [n]: \\ Y \subseteq S, N \cap S = \emptyset}} \prod_{i \in [m]} \det(\mat{A}^i_{S,S})
    },
\end{align}
where $\bm{\mu}$ denotes the $\Pi$-DPP defined by $\mat{A}^1, \ldots, \mat{A}^m$.
\cref{eq:condition} represents the probability that we draw a sample $S$ including $e$ from $\bm{\mu}$ \emph{conditioned on} that $S$ contains $Y$ but does not include any element of $N$.
We first prove that 
the conditional probability can be computed in polynomial time given access to an oracle for $\ZZ_m$.

\begin{lemma}
\label{lem:condition}
Let $Y$ and $N$ be disjoint subsets of $[n]$ and
$e$ be an element of $[n]$ not in $Y \uplus N$.
Given access to an oracle that returns $\ZZ_m$,
we can compute the conditional probability in \cref{eq:condition} in polynomial time by calling the oracle $L \triangleq 4n+2$ times.
Furthermore,
if $\{(\mat{A}^{1,\ell}, \ldots, \mat{A}^{m,\ell})\}_{\ell \in [L]}$ denotes the set consisting of the tuples of $m$ matrices
for which the oracle is called,
then it holds that
$ \rank(\mat{A}^{i,\ell}) \leq \rank(\mat{A}^i) $ and
$ \nnz(\mat{A}^{i,\ell}) \subseteq \nnz(\mat{A}^i)$
for all $i \in [m]$ and $\ell \in [L]$.
\end{lemma}
\begin{proof}
Fix two disjoint subsets $Y$ and $N$ of $[n]$.
It is sufficient to show how to compute the denominator of \cref{eq:condition} in polynomial time
by calling the oracle $2n+1$ times.
Introduce a positive number $x \in \bbQ$ and
define a matrix $\mat{X} \in \bbQ^{n \times n}$ depending on the value of $x$ as follows:
\begin{align*}
    X_{i,j} \triangleq \begin{cases}
    1 & \text{if } i,j \not\in Y \uplus N, \\
    x & \text{if } i \in Y \text{ and } j \not\in Y \uplus N, \\
    x & \text{if } j \in Y \text{ and } i \not\in Y \uplus N, \\
    x^2 & \text{if } i,j \in Y, \\
    0 & \text{otherwise.}
    \end{cases}
\end{align*}
An example of constructing $\mat{X}$ in the case of
$n=6$, $Y=\{3,4\}$, and $N=\{5,6\}$ is shown below.
\begin{align*}
\mat{X} = \begin{bmatrix}
        1 & 1 & x & x & 0 & 0 \\
        1 & 1 & x & x & 0 & 0 \\
        x & x & x^2 & x^2 & 0 & 0 \\
        x & x & x^2 & x^2 & 0 & 0 \\
        0 & 0 & 0 & 0 & 0 & 0 \\
        0 & 0 & 0 & 0 & 0 & 0
    \end{bmatrix}.
\end{align*}
Consider the matrix $ \mat{A}^1 \circ \mat{X} $, where
the symbol $\circ$ denotes the Hadamard product; namely,
$ (\mat{A}^1 \circ \mat{X})_{i,j} = A^1_{i,j} \cdot X_{i,j} $ for each $i,j \in [n]$.
It is easy to see that for each set $S \subseteq [n]$,
\begin{align*}
    \det((\mat{A}^1 \circ \mat{X})_{S,S}) =
    \begin{cases}
    0 & \text{if } S \cap N \neq \emptyset, \\
    x^{2|S \cap Y|} \cdot \det(\mat{A}^1_{S,S}) & \text{otherwise.}
    \end{cases}
\end{align*}
Further, we define a univariate polynomial $\ZZ$ in $x$ as
\begin{align*}
    \ZZ(x) \triangleq \ZZ_m(\mat{A}^1 \circ \mat{X}, \mat{A}^2, \ldots, \mat{A}^m).
\end{align*}
Observe that the degree of $\ZZ$ is at most $2n$.
Expanding $\ZZ(x)$ yields
\begin{align*}
    \ZZ(x)
    & = \sum_{S \subseteq [n]} \det((\mat{A}^1 \circ \mat{X})_{S,S}) \prod_{2 \leq i \leq m} \det(\mat{A}^i_{S,S}) \notag \\
    & = \sum_{S \subseteq [n] \setminus N} x^{2|S \cap Y|} \prod_{i \in [m]} \det(\mat{A}^i_{S,S}) \notag \\
    & = \sum_{X \subseteq Y} \sum_{S \subseteq [n] \setminus N, S \cap Y = X}
    x^{2|X|} \prod_{i \in [m]} \det(\mat{A}^i_{S,S}) \notag \\
    & = \sum_{X \subseteq Y} x^{2|X|} \sum_{S \subseteq [n] \setminus N, S \cap Y = X}
    \prod_{i \in [m]} \det(\mat{A}^i_{S,S}).
\end{align*}
Therefore, the coefficient of $x^{2n}$ in $\ZZ(x)$ is exactly equal to the desired value, i.e.,
\begin{align*}
    \sum_{S \subseteq [n] \setminus N, Y \subseteq S} \prod_{i \in [m]} \det(\mat{A}^i_{S,S}).
\end{align*}
Given $\ZZ(1), \ZZ(2), \ldots, \ZZ(2n+1)$, each of which is obtained by calling the oracle for $\ZZ_m$,
we can exactly recover all the coefficients in $\ZZ(x)$ by Lagrange interpolation as desired.
Observe that
$\rank(\mat{A}^1 \circ \mat{X}) \leq \rank(\mat{A}^1)$ and $\nnz(\mat{A}^1 \circ \mat{X}) \subseteq \nnz(\mat{A}^1)$, which completes the proof.
\end{proof}

\begin{proof}[Proof of \cref{thm:sampling}]
Our sampling algorithm is essentially equivalent to that given by \citet{celis2017complexity}.
Starting with an empty set $Y \triangleq \emptyset$,
it sequentially determines whether to include each element $e \in [n]$ in $Y$ or not, by computing the conditional probability in \cref{eq:condition} with \cref{lem:condition}.
A precise description is presented as follows:

\begin{oframed}
\begin{center}
    \textbf{Sampling algorithm for $\Pi$-DPPs given access to oracle for $\ZZ_m$.}
\end{center}
\begin{itemize}
    \item \textbf{Step 1.}~initialize $Y \triangleq \emptyset$ and $N \triangleq \emptyset$.
    \item \textbf{Step 2.}~for each element $e \in [n]$:
    \begin{itemize}
        \item \textbf{Step 2-1.}~compute the conditional probability $\displaystyle p_e \triangleq \Pr_{S \sim \bm{\mu}}[e \in S \mid Y \subseteq S, N \cap S = \emptyset]$ in \cref{eq:condition},
        where $\bm{\mu}$ denotes the $\Pi$-DPP defined by $\mat{A}^1, \ldots, \mat{A}^m$,
        by calling an oracle for $\ZZ_m$ according to \cref{lem:condition}.
        \item \textbf{Step 2-2.}~add $e$ to $Y$ with probability $p_e$;
        otherwise, add $e$ to $N$.
    \end{itemize}
    \item \textbf{Step 3.}~output $S \triangleq Y$ as a sample.
\end{itemize}
\end{oframed}
The above algorithm correctly produces a sample from $\bm{\mu}$.
The number of oracle calls is bounded by
$n (4n+2) = \bigO(n^2)$ due to \cref{lem:condition}.
The structural arguments on the matrices for which the oracle is called are obvious from \cref{lem:condition}.
\end{proof}

\subsection{All-or-Nothing Nature}
\label{subsec:property:self-reduce}
Here, we point out that
the computation of $\ZZ_m$ (for fixed $m$) either admits
an FPRAS or cannot be approximated within any subexponential factor.

\begin{theorem}
\label{thm:all-or-nothing}
For a fixed positive integer $m$, either of the following two statements holds:
\begin{itemize}
    \item there exists an FPRAS for $\ZZ_m$, or
    \item there does not exist a $2^{n^\delta}$-approximation randomized algorithm for $\ZZ_m$ for any $\delta \in (0,1)$, where $n$ is the order of the input matrices.
\end{itemize}
\end{theorem}
\begin{proof}
We show that if
there exists a $2^{n^\delta}$-approximation randomized algorithm for $\ZZ_m$ for some $\delta \in (0,1)$, then
there exists an FPRAS for $\ZZ_m$.
Let $\mat{A}^1, \ldots, \mat{A}^m$ be
$m$ positive semi-definite matrices in $\bbQ^{n \times n}$, and
let $\epsilon \in (0,1)$ be an error tolerance.
We also introduce a positive integer $t$, the value of which will be determined later.
For each $i \in [m]$,
we define $\mat{A}^{i,(t)}$ to be an $nt \times nt$ block diagonal matrix, each diagonal block of which is $\mat{A}^i$.
Note that $\mat{A}^{i,(t)}$ is positive semi-definite for all $i \in [m]$.
Then, by a simple calculation, we can expand $\ZZ_m(\mat{A}^{1,(t)}, \ldots, \mat{A}^{m,(t)})$ as follows:
\begin{align*}
    \ZZ_m(\mat{A}^{1,(t)}, \ldots, \mat{A}^{m,(t)}) & = \sum_{S \subseteq [nt]}
    \prod_{i \in [m]} \det(\mat{A}^{i,(t)}_{S,S}) \notag \\
    & =
    \sum_{S_1 \subseteq [n]} \cdots \sum_{S_t \subseteq [n]} 
    \prod_{i \in [m]} \det(\mat{A}^i_{S_1,S_1}) \cdots \det(\mat{A}^i_{S_t,S_t}) \notag \\
    & = \Biggl( \sum_{S \subseteq [n]} \prod_{i \in [m]} \det(\mat{A}^i_{S,S}) \Biggr)^t \notag \\
    & = \ZZ_m(\mat{A}^1, \ldots, \mat{A}^m)^t.
\end{align*}

Suppose now there exists a $2^{n^\delta}$-approximation randomized algorithm for $\ZZ_m$ for some fixed $\delta \in (0,1)$, where $n$ is the order of the input matrices.
Here, we can assume that the algorithm satisfies the approximation guarantee with probability at least $\frac{3}{4}$.
When invoking this approximation algorithm on $\mat{A}^{1,(t)}, \ldots, \mat{A}^{m,(t)}$,
we obtain an estimate $\widehat{\ZZ}$ to
$\ZZ_m(\mat{A}^{1,(t)}, \ldots, \mat{A}^{m,(t)})$
such that
\begin{align*}
    \ZZ_m(\mat{A}^1, \ldots, \mat{A}^m)^t \leq
    \widehat{\ZZ} \leq 2^{(nt)^\delta} \cdot \ZZ_m(\mat{A}^1, \ldots, \mat{A}^m)^t.
\end{align*}
Taking the $t$-th root of both sides yields
\begin{align}
\label{eq:t-root-approx}
    \ZZ_m(\mat{A}^1, \ldots, \mat{A}^m) \leq 
    \widehat{\ZZ}^{1/t} \leq 2^{n^\delta t^{\delta-1}} \cdot \ZZ_m(\mat{A}^1, \ldots, \mat{A}^m).
\end{align}
We now specify the value of $t$:
\begin{align*}
    t \triangleq \left\lceil \left(\frac{n^\delta}{\frac{\epsilon}{2} \cdot \log_2 \rme} \right)^{\frac{1}{1-\delta}} \right\rceil,
\end{align*}
the number of bits required to represent which is bounded by a polynomial in $\log n$ and $\log \epsilon^{-1}$ for fixed $\delta$.
The approximation factor $2^{n^\delta t^{\delta-1}}$ in \cref{eq:t-root-approx} can be bounded as follows:
\begin{align*}
    2^{n^\delta t^{\delta-1}} \leq 2^{n^\delta \left(\frac{n^\delta}{\frac{\epsilon}{2} \cdot \log_2 \rme} \right)^{\frac{\delta-1}{1-\delta}}} = 2^{\frac{\epsilon}{2} \cdot \log_2 \rme} = \rme^{\epsilon/2}.
\end{align*}

Our algorithm simply constructs $m$ positive semi-definite matrices
$\mat{A}^{1,(t)}, \ldots, \mat{A}^{m,(t)}$,
invokes a $2^{n^\delta}$-approximation algorithm on them to obtain $\widehat{\ZZ}$,
computes an $\rme^{\epsilon/2}$-approximation to $\widehat{\ZZ}^{1/t}$,%
\footnote{
We can use, for example, the Newton--Raphson method to approximate the $t$-th root efficiently.
}
denoted $\widetilde{\ZZ}$, and 
outputs it as an estimate.
This algorithm meets the specifications for FPRAS because
the size of the $m$ matrices $\mat{A}^{i,(t)}, \ldots, \mat{A}^{m,(t)}$ is bounded by a polynomial in $n$ and $\epsilon^{-1}$ for fixed $\delta$, and
the output $\widetilde{\ZZ}$ satisfies \cref{eq:fpras} with probability at least $\frac{3}{4}$,
which completes the proof.
\end{proof}

\section{Intractability of Exponentiated DPPs}
\label{sec:intract}
We present the intractability of computing the normalizing constant for exponentiated DPPs of every positive even exponent,
e.g.,
$\ZZ_2(\mat{A}, \mat{A})$,
$\ZZ_4(\mat{A}, \mat{A}, \mat{A}, \mat{A})$,
$\ZZ_6(\mat{A}, \mat{A}, \mat{A}, \mat{A}, \mat{A}, \mat{A})$ and so on.
For a positive number $p$ and a matrix $\mat{A} \in \bbR^{n \times n}$, we define
\begin{align*}
 \ZZ^p(\mat{A}) \triangleq \sum_{S \subseteq [n]} \det(\mat{A}_{S,S})^p.
\end{align*}
We will prove the following theorem.

\begin{theorem}
\label{thm:sos-hard}
Computing $\ZZ^2(\mat{A}) \bmod 3$ for a matrix $\mat{A} \in \bbQ^{n \times n}$ is
\UP-hard and \ModkP{3}-hard.
The same statement holds even when $\mat{A}$ is restricted to be
either a $(-1,0,1)$-matrix or a P-matrix.
\end{theorem}

As a corollary,
we can show the same hardness for every \emph{fixed} positive even integer $p$
(i.e., $p$ is not in the input),
thus giving a negative answer to an open question of \citet{kulesza2012determinantal}.

\begin{corollary}
\label{cor:pow-hard}
For every fixed positive even integer $p$,
computing $ \ZZ^p(\mat{A}) \bmod 3 $ for either a $(-1,0,1)$-matrix or a P-matrix $\mat{A}$ is \UP-hard and \ModkP{3}-hard.
\end{corollary}

\begin{proof}
Since $ 0^p \equiv 0^2, 1^p \equiv 1^2, 2^p \equiv 2^2 \mod 3 $ if $p$ is a positive even integer,
we have that $ \ZZ^p(\mat{A}) \equiv \ZZ^2(\mat{A}) \mod 3$.
\end{proof}

The proof of \cref{thm:sos-hard} relies on the celebrated results relating $\ZZ^2$ to the permanent
by \citet{kogan1996computing},
who presented an efficient algorithm for computing
$ \per(\mat{A}) \bmod 3 $ for a matrix $\mat{A}$ with $ \rank(\mat{A}\mat{A}^\top - \mat{I}_n) \leq 1 $.
In the remainder of this subsection,
\emph{arithmetic operations are performed over modulo $3$}, and
the symbol $\equiv$ means congruence modulo $3$.

\begin{lemma}[\protect{\citealp[Lemma 2.7]{kogan1996computing}}]
\label{lem:heart-per}
Let $\mat{X}$ be a matrix such that
$ \det(\mat{X} + \iu \mat{I}_n) \not \equiv 0 $.
Then, it holds that
\begin{align*}
\ZZ^2(\mat{X}) \equiv \det(\mat{X} + \iu \mat{I}_n)^2 \per((\mat{I}_n+\iu \mat{X})^{-1} + \mat{I}_n).
\end{align*}
\end{lemma}

\begin{proof}[Proof of \cref{thm:sos-hard}]
\label{thm:per-hard}
We reduce from the problem of computing the permanent of a $(0,1)$-matrix $\bmod 3$, which is \UP-hard and \ModkP{3}-hard \citep[Theorem 2]{valiant1979complexity},
to the problem of computing $\ZZ^2 \bmod 3$.
Let $\mat{A}$ be an $n \times n$ $(0,1)$-matrix.
By Proposition 2.2 due to \citet{kogan1996computing},
we compute a diagonal $(-1,1)$-matrix $\mat{D}$
in polynomial time such that
$ \mat{D}\mat{A} - \mat{I}_n $ is not singular and
$\per(\mat{A}) \equiv \det(\mat{D}) \per(\mat{D}\mat{A})$.
We then compute the matrix $\mat{X} \triangleq \iu^{-1} ((\mat{D}\mat{A}-\mat{I}_n)^{-1} - \mat{I}_n)$ by performing Gaussian elimination modulo $3$.
Since $ \det(\mat{X} + \iu \mat{I}_n) \not \equiv 0 $, we have by \cref{lem:heart-per} that
$ \per(\mat{A}) \equiv \ZZ^2(\mat{X}) \det(\mat{D}) \det(\mat{X} + \iu \mat{I}_n)^{-2}. $
We transform $\mat{X}$ into a new matrix $\mat{X}'$ according to the following two cases:
\begin{itemize}
    \item \textbf{Case (1)} $\det(\mat{D}) \det(\mat{X} + \iu \mat{I}_n)^{-2} \equiv 1$:
    let $\mat{X}' \triangleq \mat{X}$.
    \item \textbf{Case (2)} $\det(\mat{D}) \det(\mat{X} + \iu \mat{I}_n)^{-2} \equiv 2$:
    let $ \mat{X}' \triangleq \left[\begin{smallmatrix} \mat{X} & \vec{0} & \vec{0} \\
    \vec{0}^\top & \iu & \iu \\ \vec{0}^\top & \iu & \iu \end{smallmatrix}\right]$,
    where $\vec{0}$ is the $n \times 1$ zero matrix.
    We have that $ \ZZ^2(\mat{X}') \equiv 2 \ZZ^2(\mat{X}) $; note that
    $ \ZZ^2([\begin{smallmatrix} \iu & \iu \\ \iu & \iu \end{smallmatrix}]) = -1 $.
\end{itemize}
Consequently, we always have that 
$ \per(\mat{A}) \equiv \ZZ^2(\mat{X}') $.
Because the entries of $\mat{X}'$ are purely imaginary numbers by construction,
we can uniquely define a real-valued matrix $\mat{Y}$ such that $ \mat{X}' = \iu \mat{Y} $.
Consider the polynomial $\ZZ^2(x \mat{Y})$ for a variable $x$ as a polynomial, i.e.,
\begin{align*}
\ZZ^2(x\mat{Y}) = \sum_{S \subseteq [n]} x^{2|S|} \det(\mat{Y}_{S,S})^2 \equiv a_0 + a_1 x + a_2 x^2
\end{align*}
for some $a_0, a_1, a_2$.
Solving a system of linear equations
\begin{align*}
\begin{bmatrix}
1 & 0 & 0 \\
1 & 1 & 1 \\
1 & 2 & 4
\end{bmatrix}
\begin{bmatrix}
a_0 \\ a_1 \\ a_2
\end{bmatrix}
\equiv
\begin{bmatrix}
\ZZ^2(0\mat{Y}) \\
\ZZ^2(1\mat{Y}) \\
\ZZ^2(2\mat{Y}) 
\end{bmatrix}
\end{align*}
and noting that $ \ZZ^2(1\mat{Y}) \equiv \ZZ^2(2\mat{Y}) $,
we have that
$ a_0 \equiv 1, a_1 \equiv 0, a_2 \equiv \ZZ^2(\mat{Y}) - 1 $ and hence
$\ZZ^2(\iu \mat{Y}) \equiv 2-\ZZ^2(\mat{Y}).$
We can transform $\mat{Y}$ into a $(-1,0,1)$-matrix $\mat{Y}'$
having the same permanent so that
\begin{align*}
    Y'_{i,j} =
    \begin{cases}
    0 & \text{if } Y_{i,j} \equiv 0, \\
    +1 & \text{if } Y_{i,j} \equiv 1, \\
    -1 & \text{if } Y_{i,j} \equiv 2.
    \end{cases}
\end{align*}
Further, we can obtain another P-matrix $\mat{Y}''$ defined as
$ \mat{Y}'' \triangleq \mat{Y}' + 3n \mat{I}_n $.
Finally, we find that $\per(\mat{A}) \equiv 2 - \ZZ^2(\mat{Y}') \equiv 2 - \ZZ^2(\mat{Y}'')$.
Accordingly, deciding whether $\per(\mat{A}) \not \equiv 0 $ is reduced to
deciding whether $\ZZ^2(\mat{Y}') \not \equiv 2$
(and $\ZZ^2(\mat{Y}'') \not \equiv 2$),
in polynomial time, completing the proof.
\end{proof}

\section{Inapproximability for Three Matrices (and Two Matrices)}
\label{sec:inapprox}

Albeit the \shP-hardness of $\ZZ_m$ for all $m \geq 2$,
there is still room to consider the approximability of $\ZZ_m$;
e.g., \citet{anari2017generalization} have given
an $\rme^n$-approximation algorithm for $\ZZ_2$.
Unfortunately, we show below strong inapproximability for the case of $m \geq 3$.

\subsection{(Sub)exponential Inapproximability}
\label{subsec:inapprox:three-inapprox}
We will show
\emph{exponential inapproximability} for the case of three matrices.
For two probability measures $\bm{\mu}$ and $\bm{\eta}$ on $\Omega$,
the \emph{total variation distance} is defined as 
$ \frac{1}{2} \sum_{S \in \Omega} | \mu_S - \eta_S | $.
The proof is reminiscent of the one on the \NP-hardness of 
three-matroid intersection \citep{papadimitriou2013combinatorial}.

\begin{theorem}
\label{thm:inapprox-3}
For any fixed positive number $\epsilon > 0$,
it is \NP-hard to approximate $\ZZ_3(\mat{A}, \mat{B}, \mat{C})$ 
for three matrices $ \mat{A}, \mat{B}, \mat{C}$ in $\bbQ^{n \times n} $
within a factor of $2^{\bigO(|\calI|^{1-\epsilon})}$ or $2^{\bigO(n^{1/\epsilon})}$, where $|\calI|$ is the input size.
Moreover, unless \RP~$=$~\NP, no polynomial-time algorithm can generate
a random sample from a distribution whose total variation distance from
the $\Pi$-DPP defined by $\mat{A},\mat{B},\mat{C}$ is at most $\frac{1}{3}$.
The same statement holds if $\mat{A}, \mat{B}, \mat{C}$ are restricted to be positive semi-definite.
\end{theorem}

\begin{proof}
We will show a polynomial-time Turing (a.k.a.~Cook) reduction from
an \NP-complete \textsc{HamiltonianPath} problem \citep{garey1979computers}, which,
for a directed graph $G=(V, E)$ on
$n$ vertices and $m$ edges,
asks us to find
a directed simple path that visits every vertex of $V$ exactly once (called a \emph{Hamiltonian path}).
Such a graph $G$ having a Hamiltonian path is called \emph{Hamiltonian}.

We construct $m \times m$ three positive semi-definite matrices $\mat{A}, \mat{B}, \mat{C}$ indexed by edges in $E$
such that $ \ZZ_3(\mat{A},\mat{B},\mat{C}) $ is ``significantly'' large if $G$ is Hamiltonian.
Define $\mat{A}$ and $\mat{B}$ so that
$A_{i,j}$ is $1$ if edges $i,j$ share a common head and $0$ otherwise, and
$B_{i,j}$ is $1$ if edges $i,j$ share a common tail and $0$ otherwise.
Note that for any $S \subseteq E$, $ \det(\mat{A}_{S,S}) \det(\mat{B}_{S,S}) $
takes $1$ if $S$ consists of directed paths or cycles only, and 0 otherwise.
Next, define $\mat{C}$ so that
$ \det(\mat{C}_{S,S}) = \Pr_{T}[S \subseteq T] $ for all $S \subseteq E$,
where a random edge set $T$ is chosen from a uniform distribution over all spanning trees in (the undirected version of) $G$.
Such $\mat{C}$ can be found in polynomial time:
it in fact holds that
$\mat{C} = \mat{M} \mat{L}^\dagger \mat{M}^\top$
\citep{burton1993local},
where $\mat{M} \in \{-1,0,1\}^{m \times n} $
is the edge-vertex incidence matrix of $G$, and
$\mat{L}^\dagger \in \bbQ^{n \times n}$
is the Moore--Penrose inverse of the Laplacian of $G$,
which can be obtained as $(\mat{L} + \frac{1}{n} \mat{J}_n)^{-1} - \frac{1}{n}\mat{J}_n$ by
Gaussian elimination in polynomial time
\citep{edmonds1967systems,schrijver1999theory}.
Since $m \leq n^2$,
$\det(\mat{C}_{S,S})$ for $S \subseteq E$ is within the range between $2^{-n^2}$ and $1$
if a spanning tree exists that contains $S$ and 0 otherwise.
It turns out that
$\det(\mat{A}_{S,S}) \det(\mat{B}_{S,S}) \det(\mat{C}_{S,S})$ for $S \in {E \choose n-1}$
is positive \emph{if and only if} $S$ is a Hamiltonian path.

Redefine $\epsilon \leftarrow \lfloor 1/ \epsilon \rfloor^{-1}$,
which does not decrease the value of $\epsilon$, and
$\mat{A} \leftarrow \theta \mat{A} $, where $\theta \triangleq 2^{n^{4/\epsilon}} \in \bbN$.
Since each entry of $\mat{A}$ is an integer at most $\theta$ and
each entry of $\mat{B}$ is $1$,
we have that $ \isize(\mat{A}) = \bigO(m^2 \log(2^{n^{4/\epsilon}})) = \bigO(n^{(4/\epsilon) + 4}) $ and
$ \isize(\mat{B}) = \bigO(n^4) $.
Since $ \isize(\mat{X}^{-1}) = \bigO(\isize(\mat{X}) n^2) $ for any $n \times n$ nonsingular matrix $\mat{X}$ \citep{schrijver1999theory} and
$ \isize(\mat{L} + \frac{1}{n} \mat{J}) = \bigO(n^2 \log n) $,
we have that $ \isize(\mat{L}^\dagger) = \bigO(n^4 \log n) $, and
thus $ \isize(\mat{C}) = m^2 \bigO(n^4 \log n) = \bigO(n^8 \log n) $.
Consequently, the input size is bounded by
$ |\calI| = \bigO(n^{(4/\epsilon) + 4}) + \bigO(n^4) + \bigO(n^8 \log n) = \bigO(n^{(4/\epsilon)+4}) $,
a polynomial in $n$ (for fixed $\epsilon < 1$).

Now, we explain how to use $\ZZ_3$ to decide the Hamiltonicity of $G$.
The value of $ \det(\mat{A}_{S,S}) \det(\mat{B}_{S,S}) \det(\mat{C}_{S,S}) $ for edge set $S \subseteq E$ is
0 whenever ``$|S| \geq n$,'' or ``$|S| = n-1$ but $S$ is not a Hamiltonian path.''
Then, $\ZZ_3(\mat{A},\mat{B},\mat{C})$ can be decomposed into two sums
\begin{align*}
    \sum_{S: |S| < n-1} \det(\mat{A}_{S,S}) \det(\mat{B}_{S,S}) \det(\mat{C}_{S,S}) + \sum_{S: \text{ Hamiltonian}} \det(\mat{A}_{S,S}) \det(\mat{B}_{S,S}) \det(\mat{C}_{S,S}).
\end{align*}
There are two cases:
\begin{itemize}
    \item \textbf{Case (1)} if there exists (at least) one Hamiltonian path $S^*$ in $G$,
    then $\ZZ_3(\mat{A},\mat{B},\mat{C})$ is \emph{at least}
    $ \theta^{|S^*|}2^{-n^2} = 2^{n^{(4/\epsilon)+1}-n^{4/\epsilon}-n^2} $.
    \item \textbf{Case (2)} if no Hamiltonian path exists in $G$,
    then, $\ZZ_3(\mat{A},\mat{B},\mat{C})$ is \emph{at most}
    $ \sum_{S: |S| < n-1} \theta^{n-2} \leq 2^{n^2} 2^{n^{4/\epsilon}(n-2)} = 2^{n^{(4/\epsilon)+1}-2n^{4/\epsilon}+n^2} $.
\end{itemize}
Hence, there is an exponential gap $ 2^{n^{4/\epsilon}-2n^2} $
between the two cases.
Since $ |\calI|^{1-\epsilon} = \bigO(n^{(4/\epsilon)-4\epsilon}) $,
a $ 2^{\bigO(|\calI|^{1-\epsilon})} $- or $2^{\bigO(n^{1/\epsilon})}$-approximation to $\ZZ_3$ suffices
to distinguish the two cases (for sufficiently large $n$).

Now let us prove the second argument.
Assume that $G$ is Hamiltonian.
Observe that a random edge set drawn from
the $\Pi$-DPP defined by $ \mat{A},\mat{B},\mat{C} $ (denoted $\bm{\mu}$)
is Hamiltonian with probability at least
$ 1 - \frac{1}{1 + 2^{n^{4/\epsilon} - 2n^2}} $.
Hence, provided a polynomial-time algorithm to generate random edge sets
whose total variation distance from $\bm{\mu}$ is at most $\frac{1}{3}$,
we can use it to verify the Hamiltonicity of $G$
with probability at least
$ \frac{2}{3} - \frac{1}{1 + 2^{n^{4/\epsilon} - 2n^2}} > \frac{1}{2} $
(whenever $ n \geq 2 $),
implying that \textsc{HamiltonianPath} $ \in $ \RP;
hence, \RP~$=$~\NP. This completes the proof.
\end{proof}

\subsection{Exponential Approximability}
\label{subsec:inapprox:three-approx}
Whereas making a subexponential approximation for $\ZZ_3$ in terms of the input size $|\calI|$ is hard,
we show that there is a simple exponential approximation for $\ZZ_m$ for all $m$.

\begin{observation}
\label{obs:approx-3}
For $m$ P$_0$-matrices $\mat{A}^1, \ldots, \mat{A}^m$,
the number $1$ is a $2^{\bigO(|\calI|^2)}$-approximation to
$ \ZZ_m(\mat{A}^1, \ldots, \mat{A}^m) $,
where $|\calI|$ is the input size.
\end{observation}

\begin{proof}
Obviously, $\ZZ_m$ is bounded from below by 1, so we only have to show an upper bound.
Applying Hadamard's inequality, we find that
all principal minors are at most $M^n n^{n/2}$,
where $M$ is the maximum absolute entry in the $m$ matrices.
Hence, we have that
\begin{align*}
    \ZZ_m(\mat{A}^1, \ldots, \mat{A}^m)
    & = \sum_{S \subseteq [n]} \det(\mat{A}^1_{S,S}) \cdots \det(\mat{A}^m_{S,S}) \\
    & \leq 2^n (M^n n^{n/2})^m \\
    & = 2^{n + mn \log M + \frac{mn}{2} \log n} = 2^{\bigO(|\calI|^2)},
\end{align*}
where the last deformation comes from the fact that
$ |\calI| \geq \log M $ and $ |\calI| \geq mn^2 $.
Thus, $\ZZ_m(\mat{A}^1, \ldots, \mat{A}^m)$ takes a number between $1$ and $2^{\bigO(|\calI|^2)}$, which completes the proof.
\end{proof}

\subsection{Approximation-Preserving Reduction from Mixed Discriminant to Two Matrices}
\label{subsec:inapprox:ap-red}

Finally, we present a relation between $\ZZ_2$ and the mixed discriminant.
The \emph{mixed discriminant} of $m$ positive semi-definite matrices $\mat{K}^1, \ldots, \mat{K}^m \in \bbR^{m \times m}$
is defined as
\begin{align*}
    D(\mat{K}^1, \ldots, \mat{K}^m)
    \triangleq \frac{\partial^m}{\partial x_1 \cdots \partial x_m}
    \det(x_1 \mat{K}^1 + \cdots + x_m \mat{K}^m).
\end{align*}
Mixed discriminants are known to be a generalization of permanents \citep{barvinok2016combinatorics}:
for an $m \times m$ matrix $\mat{A}$,
we define $m \times m$ matrices $\mat{K}^1, \ldots, \mat{K}^m$ such that 
$\mat{K}^i = \diag(A_{i,1}, \ldots, A_{i,m})$ for all $i \in [m]$;
it holds that $D(\mat{K}^1, \ldots, \mat{K}^m) = \per(\mat{A})$.
Hence, computing the mixed discriminant is \shP-hard.
We demonstrate an approximation-preserving reduction from
the mixed discriminant $D$ to $\ZZ_2$,
which means that if $\ZZ_2$ admits an FPRAS, then so does the mixed discriminant.
Since the existence of an FPRAS for the mixed discriminant
is suspected to be false \citep{gurvits2005complexity},
our result implies that $\ZZ_2$ is unlikely to have an FPRAS.
We stress that \citet{gillenwater2014approximate} proves the \shP-hardness of computing $\ZZ_2$ by using a parsimonious reduction from the problem of counting all matchings in a bipartite graph, which admits an FPRAS \citep{jerrum1996markov}.

Let us begin with the definition of an approximation-preserving reduction.

\begin{definition}
\label{def:apreduce}
For two functions
$f: \Sigma^* \to \bbR$ and
$g: \Sigma^* \to \bbR$,
an \emph{approximation-preserving reduction} (AP-reduction)
from $f$ to $g$
is a randomized algorithm $\alg$ that takes
an input $\calI \in \Sigma^*$ of $f$ and an error tolerance $\epsilon \in (0,1)$ and
satisfies the following conditions:
\begin{itemize}
    \item every oracle call for $g$
    made by $\alg$ is of the form $(\calJ,\delta)$,
    where $\calJ \in \Sigma^*$ is an input of $g$ and
    $\delta \in (0,1)$ is an error tolerance such that $ \delta^{-1} $ is bounded by a polynomial in
    $|\calI|$ and $\epsilon^{-1}$;
    \item if the oracle meets the specification for an FPRAS for $g$, then $\alg$ meets the specifications for an FPRAS for $f$;
    \item the running time of $\alg$ is bounded by
    a polynomial in $|\calI|$ and $\epsilon^{-1}$.
\end{itemize}
We say that $f$ is \emph{AP-reducible} to $g$ if an AP-reduction from $f$ to $g$ exists.
\end{definition}
It is known \citep{dyer2004relative} that
assuming $f$ to be AP-reducible to $g$,
an FPRAS for $g$ implies an FPRAS for $f$; in other words,
if $f$ does not admit an FPRAS (under some plausible assumption), neither does $g$.
Our result is shown below.

\begin{theorem}
\label{thm:AP-D-ZZ2}
The mixed discriminant $D$ for $m$ positive semi-definite matrices in $\bbQ^{m \times m}$ is AP-reducible to $\ZZ_2$ for two positive semi-definite matrices in $\bbQ^{m^2 \times m^2}$.
Therefore, if there exists an FPRAS for $\ZZ_2$, then there exists an FPRAS for $D$.
\end{theorem}
\begin{proof}
We will construct an AP-reduction from the mixed discriminant $D$ to $\ZZ_2$
mimicking the reduction from $D$ to spanning-tree DPPs presented by the same set of authors as this article \citep{matsuoka2021spanning}.
Suppose we have an FPRAS for $\ZZ_2$.
Let $\mat{K}^1, \ldots, \mat{K}^m$ be $m$
positive semi-definite matrices in $\bbQ^{m \times m}$, and 
define $n = m^2$.
Let $\epsilon \in (0,1)$ be an error tolerance for $D$;
i.e., we are asked to estimate $D(\mat{K}^1, \ldots, \mat{K}^m)$ within a factor of $\rme^{\epsilon}$.
In accordance with \citet[Proof of Lemma 12]{celis2017complexity},
we first construct an $n \times n$ positive semi-definite matrix $\mat{A}$ and an equal-sized partition of $[n]$, 
denoted $P_1, P_2, \ldots, P_m$ with 
$|P_1| = |P_2| = \cdots = |P_m| = m$, in polynomial time such that
\begin{align}
\label{eq:ZZ-D}
    \sum_{S \in \calC} \det(\mat{A}_{S,S})
     = m! \; D(\mat{K}^1, \ldots, \mat{K}^m),
\end{align}
where we define $\calC \triangleq \{ S \in {[n] \choose m} \mid |S \cap P_i| = 1 \text{ for all } i \in [m] \}$.

Then, we construct an $n \times n$ positive semi-definite matrix
$\mat{B}$ as follows:
\begin{align*}
    B_{i,j} \triangleq
    \begin{cases}
    1 & \text{if } i \text{ and } j \text{ belong to the same group in the partition of } [n], \\
    0 & \text{otherwise.}
    \end{cases}
\end{align*}
We claim the following:
\begin{claim}
\label{clm:ZZ2-D}
For each subset $S \subseteq [n]$,
$\det(\mat{B}_{S,S})$ is $1$ if
no two elements in $S$ belong to the same group in the partition and $0$ otherwise.
In particular, we have that
\begin{itemize}
    \item for any $S \in {[n] \choose m}$,
$\det(\mat{B}_{S,S}) = 1$ if and only if $S \in \calC$;
    \item $\det(\mat{B}_{S,S})$ is $0$ whenever $|S| > m$.
\end{itemize}
\end{claim}
\begin{proof}[Proof of \cref{clm:ZZ2-D}]
Fix a subset $S \subseteq [n]$.
If $S$ contains two elements $i,j$ that belong to the same group in the partition, we have that
$\mat{B}_{S,\{i\}} = \mat{B}_{S,\{j\}}$;
hence $\det(\mat{B}_{S,S}) = 0$.
On the other hand,
if $S$ does not contain two such elements,
$\mat{B}_{S,S}$ is exactly an $S \times S$ identity matrix; hence, $\det(\mat{B}_{S,S}) = 1$.
\end{proof}

The following equality is a direct consequence of \cref{clm:ZZ2-D}.
\begin{align}
\label{eq:ZZ2-D}
    \sum_{S \in {[n] \choose m}} \det(\mat{A}_{S,S}) \det(\mat{B}_{S,S}) = \sum_{S \in \calC} \det(\mat{A}_{S,S}).
\end{align}
First, we verify whether a subset $S \in \calC$ exists such that $\det(\mat{A}_{S,S}) > 0$ because otherwise, we can safely declare that \cref{eq:ZZ2-D} is $0$; i.e.,
$D(\mat{K}^1, \ldots, \mat{K}^m)$ is $0$ as well.
Such a subset can be found (if it exists) by performing matroid intersection because 
$\calI_1 = \{S \subseteq [n] \mid \det(\mat{A}_{S,S}) > 0\}$  forms a linear matroid and
$\calI_2 = \{S \subseteq [n] \mid \exists T \in \calC, S \subseteq T \}$ forms a partition matroid.
Denote the subset found by $\tilde{S} \in \calI_1 \cap \calI_2$.
Next, we define a positive number $x \in \bbQ$ as
\begin{align*}
    x \triangleq \frac{\det(\mat{A} + \mat{I}_n)}{\det(\mat{A}_{\tilde{S},\tilde{S}})}
    \frac{2}{\epsilon}.
\end{align*}
Note that the size of $x$ is bounded by a polynomial in the size of $\mat{A}$ and $\epsilon^{-1}$.
It is easy to see that for each $S \subseteq [n]$,
\begin{align}
\label{eq:ZZ2-ap-x}
    \det((x\mat{B})_{S,S}) = x^{|S|}\det(\mat{B}_{S,S}).
\end{align}
Since \cref{clm:ZZ2-D} ensures that
$\det(\mat{A}_{S,S}) \det(\mat{B}_{S,S}) = 0$ whenever $|S| > m$,
we can bound $\ZZ_2(\mat{A}, x\mat{B})$ from above as follows:
\begin{align*}
    \ZZ_2(\mat{A}, x\mat{B})
    & = \sum_{S \subseteq [n]} \det(\mat{A}_{S,S}) \det((x\mat{B})_{S,S}) \\
    & = \sum_{S:|S| \leq m-1} x^{|S|} \det(\mat{A}_{S,S})\det(\mat{B}_{S,S}) + \sum_{S:|S|=m} x^{|S|} \det(\mat{A}_{S,S})\det(\mat{B}_{S,S}) \\
    & \leq \sum_{S:|S|\leq m-1}x^{m-1}\det(\mat{A}_{S,S}) + \sum_{S \in \calC} x^{m} \det(\mat{A}_{S,S}) \\
    & \leq x^{m} \left(\sum_{S \in \calC} \det(\mat{A}_{S,S})\right) \left( 1 + \frac{\sum\limits_{S:|S|\leq m-1}\det(\mat{A}_{S,S})}{\sum\limits_{S \in \calC} \det(\mat{A}_{S,S})}\frac{1}{x} \right) \\
    & \leq x^m \left(\sum_{S \in \calC} \det(\mat{A}_{S,S})\right) \left( 1 + \frac{\sum\limits_{S:|S|\leq m-1}\det(\mat{A}_{S,S})}{\sum\limits_{S \in \calC} \det(\mat{A}_{S,S})} \frac{\det(\mat{A}_{\tilde{S},\tilde{S}})}{\det(\mat{A}+\mat{I}_n)} \frac{\epsilon}{2} \right).
\end{align*}
Observing the fact that 
\begin{align*}
    \sum_{S: |S|\leq m-1} \det(\mat{A}_{S,S}) \leq \det(\mat{A}+\mat{I}_n) \text{ and }
    \det(\mat{A}_{\tilde{S},\tilde{S}}) \leq \sum_{S \in \calC} \det(\mat{A}_{S,S}),
\end{align*}
we further have that
\begin{align*}
    \ZZ_2(\mat{A}, x\mat{B})
    & \leq \Bigl(1+\frac{\epsilon}{2} \Bigr) x^m \sum_{S \in \calC} \det(\mat{A}_{S,S}) \\
    & \leq \rme^{\frac{\epsilon}{2}} x^m m!\; D(\mat{K}^1, \ldots, \mat{K}^m). \qquad\qquad\qquad(\text{by \cref{eq:ZZ-D}})
\end{align*}
Since $ \ZZ_2(\mat{A}, x\mat{B}) \geq x^m m! \; D(\mat{K}^1, \ldots, \mat{K}^m) $, we have that
\begin{align}
\label{eq:ZZ2-ap-bound}
    D(\mat{K}^1, \ldots, \mat{K}^m) \leq \frac{\ZZ_2(\mat{A},x \mat{B})}{x^m m!} \leq \rme^{\frac{\epsilon}{2}} \cdot D(\mat{K}^1, \ldots, \mat{K}^m).
\end{align}

We are now ready to describe the AP-reduction from the mixed discriminant $D$ to $\ZZ_2$.
\begin{oframed}
\begin{center}
\textbf{AP-reduction from $D$ to $\ZZ_2$.}
\end{center}
\begin{itemize}
    \item \textbf{Step 1.}~construct two matrices $\mat{A}, \mat{B} \in \bbQ^{n \times n}$ satisfying \cref{eq:ZZ-D,eq:ZZ2-D} by following the procedure described at the beginning of the proof.
    \item \textbf{Step 2.}~determine if there exists a subset $S \subseteq [n]$
    such that $S \in \calC$ and $\det(\mat{A}_{S,S}) > 0$ by matroid intersection in polynomial time \citep{edmonds1970submodular}.
    If no such subset has been found, declare that
    ``$D(\mat{K}^1, \ldots, \mat{K}^m) = 0$''; otherwise, denote the subset found by $\tilde{S}$.
    \item \textbf{Step 3.}~calculate the value of $x$ according to \cref{eq:ZZ2-ap-x},
    which can be done in polynomial time in the input size and $\epsilon^{-1}$ because the size of $\mat{A}$ is bounded by a polynomial in the size of $\mat{K}^1, \ldots, \mat{K}^m$ and the determinant can be computed in polynomial time by Gaussian elimination \citep{edmonds1967systems,schrijver1999theory}.
    \item \textbf{Step 4.}~call an oracle for $\ZZ_2$ on
    $\mat{A}$ and $x\mat{B}$ with error tolerance $\delta = \epsilon/2$ to obtain an $\rme^{\epsilon/2}$-approximation to $\ZZ_2(\mat{A}, x\mat{B})$, which will be denoted by $\widehat{\ZZ}$.
    \item \textbf{Step 5.}~output $\displaystyle \frac{\widehat{\ZZ}}{x^m m!}$ as an estimate for $D(\mat{K}^1, \ldots, \mat{K}^m)$.
\end{itemize}
\end{oframed}

By \cref{eq:ZZ2-ap-bound},
if the oracle meets the specifications for an FPRAS for $\ZZ_2$,
then the output $\widehat{\ZZ}$ of the AP-reduction described above satisfies that
\begin{align*}
    \rme^{-\epsilon} \cdot D(\mat{K}^1, \ldots, \mat{K}^m) \leq
    \frac{\widehat{\ZZ}}{x^{m} m!} \leq \rme^{\epsilon} \cdot D(\mat{K}^1, \ldots, \mat{K}^m).
\end{align*}
with probability at least $\frac{3}{4}$.
Therefore, the AP-reduction meets the specification for an FPRAS for $D$, which completes the proof.
\end{proof}

\section{Fixed-Parameter Tractability}
\label{sec:fpt}

Here, we investigate the fixed-parameter tractability of computing $\ZZ_m$.
Given a parameter $k$,
a problem is said to be
\emph{fixed-parameter tractable} (FPT) and \emph{slice-wise polynomial} (XP)
if it is solvable in $f(k)|\calI|^{\bigO(1)}$ and $|\calI|^{f(k)}$ time 
for some computable function $f$, respectively.
It should be noted that the value of $k$
may be independent of the input size $|\calI|$ and
may be given by some computable function $k = k(\calI)$ on input $\calI$ (e.g., the rank of an input matrix).
Our goal is either
(1) to develop an FPT algorithm for an appropriate parameter, or
(2) to disprove the existence of such algorithms
under plausible assumptions.

\subsection{Parameterization by Maximum Rank}
\label{subsec:fpt:rank}

First, let us consider the \emph{maximum rank} of matrices as a parameter.
The theorem below demonstrates that computing $ \ZZ_2(\mat{A},\mat{B}) $
for two positive semi-definite matrices $\mat{A}$ and $\mat{B}$
parameterized by the maximum rank is FPT.

\begin{theorem}
\label{thm:fpt-rank-2}
Let $\mat{A}, \mat{B}$ be two positive semi-definite matrices in $\bbQ^{n \times n}$ of rank at most $r$.
Then, there exists an $ r^{\bigO(r)} n^{\bigO(1)} $-time algorithm computing $\ZZ_2(\mat{A}, \mat{B})$ exactly.
\end{theorem}

Before proceeding to the proof, we introduce the following technical lemma.

\begin{lemma}
\label{lem:dp}
Let $ \mat{A}^1, \ldots, \mat{A}^m $ be $m$ matrices in $\bbQ^{n \times s}$, and
$ \sigma_1, \ldots, \sigma_m \in \Sym_s$
be $m$ permutations over $[s]$.
Then,
\begin{align*}
    \sum_{S \subseteq {[n] \choose s}}
    \mat{A}^1_S(\sigma_1) \cdots \mat{A}^m_S(\sigma_m) =
    \sum_{S \subseteq {[n] \choose s}}
    \prod_{i \in [s]} (A^1_S)_{i, \sigma_1(i)} \cdots (A^m_S)_{i, \sigma_m(i)}
\end{align*}
can be computed in $\bigO(msn^2)$ time.
\end{lemma}
\begin{proof}
The proof is based on dynamic programming.
First, we define a table $dp$ of size $ s \times n $,
whose entries for each pair of
$\ell \in [s]$ and $o \in [n]$ are defined as
\begin{align*}
    dp[\ell, o] \triangleq \sum_{\substack{S \subseteq {[n] \choose \ell} \\ \max_{o' \in S}o' = o}}
    \prod_{i \in [\ell]}
    (A^1_S)_{i, \sigma_1(i)} \cdots  (A^m_S)_{i, \sigma_m(i)}.
\end{align*}
The desired value is equal to $ \sum_{o \in [n]} dp[s, o] $.
Observe that for $\ell \in [2 \isep s]$ and $ o \in [n]$,
\begin{align*}
    dp[\ell, o] & = \sum_{1 \leq o' < o} dp[\ell-1, o'] A^1_{o, \sigma_1(\ell)} \cdots A^m_{o, \sigma_m(\ell)}, \\
    dp[1, o] & = A^1_{o, \sigma_1(1)} \cdots A^m_{o, \sigma_m(1)}.
\end{align*}
Note that the number of bits required to express each entry is bounded by
$ \log(2^n) (\isize(\mat{A}^1) + \cdots + \isize(\mat{A}^m)) $.
Since calculating $dp[\ell, o]$ given $dp[\ell-1, o']$ for all $o' \in [n]$ requires $\bigO(nm)$ arithmetic operations,
standard dynamic programming fills all entries of $dp$ within $ \bigO(msn^2) $ arithmetic operations.
\end{proof}

Next, we introduce the Cauchy--Binet formula.
\begin{lemma}[Cauchy--Binet formula]
Let $\mat{A}$ be an $s \times r$ matrix and $\mat{B}$ be an $r \times s$ matrix.
Then, the determinant of $\mat{A}\mat{B}$ is
\begin{align*}
\det(\mat{A}\mat{B}) = \sum_{C \in {[r] \choose s}}
    \det(\mat{A}_{[s], C}) \det(\mat{B}_{C, [s]}).
\end{align*}
\end{lemma}

\begin{proof}[Proof of \cref{thm:fpt-rank-2}]
We decompose $\mat{A}$ into two $n \times r$ rectangular matrices.
For this purpose, we first compute an LDL decomposition\footnote{We do not use the Cholesky decomposition because we must avoid the square root computation, which violates the assumption that every number appearing in this paper is rational.}
$\mat{A} = \mat{L}\mat{D}\mat{L}^\top $, where
$\mat{L} \in \bbQ^{n \times n} $ and
$\mat{D} \in \bbQ^{n \times n}$
is a diagonal matrix such that
$ D_{i,i} = 0 $ for all $i \in [r+1 \isep n]$ (since the rank is at most $r$).
This is always possible in polynomial time \citep{odonnell2011semidefinite}
because $\mat{A}$ is positive semi-definite.
We further decompose $\mat{D}$ into the product of
an $n \times r$ matrix $\mat{C}$ such that
$C_{i,i} = D_{i,i}$ for all $i \in [r]$ and all the other elements are $0$, and
an $r \times n$ matrix $\mat{I}$ such that
$I_{i,i} = 1$ for all $i \in [r]$ and all the other elements are $0$.
Setting $\mat{U} = \mat{L}\mat{C} \in \bbQ^{n \times r}$ and
$\mat{V} = \mat{L}\mat{I}^\top \in \bbQ^{n \times r}$,
we have that $ \mat{A} = \mat{U} \mat{V}^\top $.
Similarly, we decompose $\mat{B} = \mat{X}\mat{Y}^\top$,
where $ \mat{X}$ and $\mat{Y}$ are some $n \times r$ rectangular matrices in $\bbQ^{n \times r}$.

Because $ \det(\mat{A}_{S,S}) \det(\mat{B}_{S,S}) = 0$ for all $S \subseteq [n]$ of size greater than $r$,
we can expand $\ZZ_2(\mat{A},\mat{B})$ by using the Cauchy--Binet formula as follows.
\begin{align*}
\ZZ_2(\mat{A},\mat{B}) & = \sum_{0 \leq s \leq r} \sum_{S \in {[n] \choose s}} 
    \det(\mat{U}_{S} \mat{V}_{S}^\top) \det(\mat{X}_{S} \mat{Y}_{S}^\top) \\
    & = \sum_{\substack{
        0 \leq s \leq r \\
        S \in {[n] \choose s}
    }}
    \sum_{C_1 \in {[r] \choose s}}
    \det(\mat{U}_{S,C_1} \mat{V}_{S,C_1}^\top) \sum_{C_2 \in {[r] \choose s}} \det(\mat{X}_{S,C_2} \mat{Y}_{S,C_2}^\top).
\end{align*}

Noting that $|S|=|C_1|=|C_2|$,
we further expand $\ZZ_2(\mat{A},\mat{B})$ as 
\begin{align*}
& \sum_{\substack{0 \leq s \leq r \\ C_1, C_2 \in {[r] \choose s}}} \sum_{\sigma_1, \tau_1, \sigma_2, \tau_2 \in \Sym_s}
    \sgn(\sigma_1) \sgn(\tau_1) \sgn(\sigma_2) \sgn(\tau_2) \times \\
& \quad\quad\quad\quad \underbrace{\sum_{S \in {[n] \choose s}}
    \mat{U}_{S,C_1}(\sigma_1) \mat{V}_{S,C_1}(\tau_1) \mat{X}_{S,C_2}(\sigma_2) \mat{Y}_{S,C_2}(\tau_2)}_{\bigstar}.
\end{align*}
Since $\bigstar$ can be evaluated in $\bigO(sn^2)$ time by \cref{lem:dp},
we can take the sum of $\bigstar$ over all possible combinations of
$s, C_1, C_2, \sigma_1, \tau_1, \sigma_2, \tau_2$ in
$\sum_{0 \leq s \leq r} {r \choose s}^2 (s!)^4 \bigO(s n^2) = \bigO(r^{4r} r^2 n^2)$ time.
Consequently, the overall computation time is bounded by
$ r^{\bigO(r)} n^2 $.
\end{proof}

\cref{thm:fpt-rank-2} can be generalized to the case of
$m$ matrices $\mat{A}^1, \ldots, \mat{A}^m$; that is,
the computation of $ \ZZ_m $
parameterized by the maximum rank $ \max_{i \in [m]} \rank(\mat{A}^i) $ \emph{plus} the number of matrices $m$ is FPT.

\begin{theorem}
\label{thm:fpt-rank-m}
For a positive integer $m$,
let $\mat{A}^1, \ldots, \mat{A}^m$ be $m$ positive semi-definite matrices in $\bbQ^{n \times n}$ of rank at most $r$.
Then, there exists an $r^{\bigO(mr)} n^{\bigO(1)}$-time algorithm computing $\ZZ_m(\mat{A}^1, \ldots, \mat{A}^m)$ exactly.
\end{theorem}
\begin{proof}
Similar to the proof of \cref{thm:fpt-rank-2},
we first decompose each of $m$ matrices into the product of two $n \times r$ rectangular matrices, i.e.,
$ \mat{A}^i = \mat{X}^i (\mat{Y}^i)^\top $ for some $\mat{X}^i, \mat{Y}^i \in \bbQ^{n \times r}$ for all $ i \in [m]$,
by LDL decomposition.
We then expand $\ZZ_m(\mat{A}^1, \ldots, \mat{A}^m)$ as
\begin{align*}
& \sum_{S \subseteq [n]} \det(\mat{A}^1_{S,S}) \cdots \det(\mat{A}^m_{S,S}) \\
& = \sum_{0 \leq s \leq r} \sum_{S \in {[n] \choose s}}
\sum_{C_1 \in {[r] \choose s}} \det(\mat{X}^1_{S,C_1}) \det((\mat{Y}^1_{S,C_1})^\top)
\cdots
\sum_{C_m \in {[r] \choose s}} \det(\mat{X}^m_{S,C_m}) \det((\mat{Y}^m_{S,C_m})^\top)
\\
& = \sum_{0 \leq s \leq r}
\sum_{\substack{C_1 \in {[r] \choose s} \\ \sigma_1 , \tau_1 \in \Sym_s}} \cdots 
\sum_{\substack{C_m \in {[r] \choose s} \\ \sigma_m, \tau_m \in \Sym_s}}
\sgn(\sigma_1) \sgn(\tau_1) \cdots \sgn(\sigma_m) \sgn(\tau_m) \times \\
& \quad\quad\quad\quad \underbrace{
\sum_{S \in {[n] \choose s}} \prod_{i \in [s]}
(X^1_{S,C_1})_{i,\sigma_1(i)} (Y^1_{S,C_1})_{i,\tau_1(i)} \cdots (X^m_{S,C_m})_{i,\sigma_m(i)} (Y^m_{S,C_m})_{i,\tau_m(i)}}_{\clubsuit}.
\end{align*}
By applying \cref{lem:dp},
we can calculate $\clubsuit$ in $\bigO(msn^2)$ time; thus,
the entire time complexity is bounded by
\begin{align*}
\sum_{0 \leq s \leq r} {r \choose s}^m (s!)^{2m} \bigO(msn^2)
=
\bigO(r^{2mr} r^2 n^2) = r^{\bigO(mr)} n^2,
\end{align*}
which completes the proof.
\end{proof}

\subsection{Parameterization by Treewidth of Union}
\label{subsec:fpt:treewidth}

Now we consider the \emph{treewidth} of the graph
formed by the \emph{union} of nonzero entries as a parameter.
The following theorem demonstrates that
computing $\ZZ_2(\mat{A},\mat{B})$ parameterized by $\tw(\nnz(\mat{A}) \cup \nnz(\mat{B}))$ is FPT.

\begin{theorem}
\label{thm:fpt-treewidth-2}
Let $\mat{A}, \mat{B}$ be two matrices in $\bbQ^{n \times n}$.
Then, there exists a $w^{\bigO(w)} n^{\bigO(1)}$-time algorithm
that,
given a tree decomposition of
the graph $ ([n], \nnz(\mat{A}) \cup \nnz(\mat{B})) $
of width at most $w$,
computes $\ZZ_2(\mat{A}, \mat{B})$ exactly.
\end{theorem}

\begin{remark}
To construct ``reasonable'' tree decompositions, we can use existing algorithms,
e.g., a $ 2^{\bigO(w)}n $-time 5-approximation algorithm by \citet*{bodlaender2016approximation},
where $n$ is the number of vertices and $w$ is the treewidth.
Hence, we do not need to be given a tree decomposition to use the algorithm of \cref{thm:fpt-treewidth-2}.
\end{remark}

\subsubsection{Design of Dynamic Programming}
Our proof is based on dynamic programming upon a tree decomposition.
First, we define a \emph{nice tree decomposition} $(T, \{X_t\}_{t \in T})$ formally, which is a convenient form of tree decomposition.
Think of $T$ as a rooted tree by
referring to a particular node $r$ as the \emph{root} of $T$,
which naturally introduces the notions of parent, child, and leaf.
\begin{definition}[\protect{\citealp[nice tree decomposition]{kloks1994treewidth}}]
\label{def:nice-td}
A tree decomposition $(T, \{X_t\}_{t \in T})$ rooted at $r$ is said to be \emph{nice} if
\begin{itemize}
\item every leaf and the root have empty bags; i.e.,
$ X_r = \emptyset $ and $X_\ell = \emptyset$ for every leaf $\ell$ of $T$;
\item each non-leaf node is one of the following:
\begin{itemize}
\item \emph{Introduce node}:
a node $t$ with exactly one child $t'$ such that $X_t = X_{t'} + v$
for some $v \not \in X_{t'}$.
\item \emph{Forget node}:
a node $t$ with exactly one child $t'$ such that $X_t = X_{t'} - v$
for some $v \in X_{t'}$.
\item \emph{Join node}:
a node $t$ with exactly two children $t', t''$ such that
$X_t = X_{t'} = X_{t''}$.
\end{itemize}

\end{itemize}
\end{definition}
For a node $t$ of $T$, we define
\begin{align*}
V_t \triangleq \bigcup_{t' \text{ in subtree rooted at } t} X_{t'}.
\end{align*}
In particular, it holds that $ V_r = [n] $ for the root $r$, and $ V_\ell = X_\ell = \emptyset $ for every leaf $\ell$ of $T$.
\cref{fig:mat-A,fig:G,fig:td,fig:ntc} show an example of a (nice) tree decomposition.

\begin{figure}[tbp]
\begin{minipage}{0.3\hsize}
\centering
$\begin{bmatrix}
*&*&*&0&0&0 \\
*&*&*&*&*&0 \\
*&*&*&0&*&* \\
0&*&0&*&*&0 \\
0&*&*&*&*&* \\
0&0&*&0&*&* \\
\end{bmatrix}$
\caption{Matrix $\mat{A} \in \bbQ^{6 \times 6}$, where ``$*$'' denotes nonzero entries.}
\label{fig:mat-A}
\end{minipage}
\hfill
\begin{minipage}{0.3\hsize}
\centering
\scalebox{0.8}{
\begin{tikzpicture}
[circlenode/.style={draw, circle, minimum height=1.1cm, font=\Large, inner sep=0}]
\node[circlenode, thick](v1){$1$};
\node[circlenode, thick, above right=0.1cm and 1cm of v1](v2){$2$};
\node[circlenode, thick, below right=0.1cm and 1cm of v1](v3){$3$};
\node[circlenode, thick, above right=0.1cm and 1cm of v2](v4){$4$};
\node[circlenode, thick, below right=0.1cm and 1cm of v2](v5){$5$};
\node[circlenode, thick, below right=0.1cm and 1cm of v3](v6){$6$};
\foreach \u / \v in {v1/v2, v1/v3, v2/v3, v2/v4, v2/v5, v3/v5, v3/v6, v4/v5, v5/v6}
    \draw[thick] (\u)--(\v);
\end{tikzpicture}
}
\caption{Graph $G=(V, E)$ constructed from the nonzero entries of $\mat{A}$, where $V = [6]$ and $E=\nnz(\mat{A})$.}
\label{fig:G}
\end{minipage}
\hfill
\begin{minipage}{0.3\hsize}
\centering
\scalebox{0.8}{
\begin{tikzpicture}
[circlenode/.style={draw, circle, minimum height=1.1cm, font=\Large, inner sep=0}]
\node[circlenode, thick](v123){$123$};
\node[circlenode, thick, right=1cm of v123](v235){$235$};
\node[circlenode, thick, above right=0.5cm and 1cm of v235](v245){$245$};
\node[circlenode, thick, below right=0.5cm and 1cm of v235](v356){$356$};
\foreach \u / \v in {v123/v235, v235/v245, v235/v356}
    \draw[thick] (\u)--(\v);
\end{tikzpicture}
}
\caption{Tree decomposition $(T, \{X_t\}_{t \in T})$ of $G$. $T$ contains four nodes, and bags are of size $3$; i.e., its treewidth is $2$.}
\label{fig:td}
\end{minipage}
\vspace{2em}
\begin{minipage}{1\hsize}
\centering
\scalebox{0.6}{
\begin{tikzpicture}
[circlenode/.style={draw, circle, minimum height=1.1cm, font=\Large, inner sep=0}]
\node[circlenode, thick](v0a){$\emptyset$};
\node[above=0cm of v0a, font=\Large]{root};
\node[circlenode, thick, right=1cm of v0a](v1){$1$};
\node[circlenode, thick, right=1cm of v1](v12){$12$};
\node[circlenode, thick, right=1cm of v12](v123){$123$};
\node[circlenode, thick, right=1cm of v123](v23){$23$};
\node[circlenode, thick, right=1cm of v23](v235a){$235$};
\node[circlenode, thick, above right=1cm of v235a](v235b){$235$};
\node[circlenode, thick, right=1cm of v235b](v25){$25$};
\node[circlenode, thick, right=1cm of v25](v245){$245$};
\node[circlenode, thick, right=1cm of v245](v24){$24$};
\node[circlenode, thick, right=1cm of v24](v4){$4$};
\node[circlenode, thick, right=1cm of v4](v0b){$\emptyset$};
\node[above=0cm of v0b, font=\Large]{leaf};
\node[circlenode, thick, below right=1cm of v235a](v235c){$235$};
\node[circlenode, thick, right=1cm of v235c](v35){$35$};
\node[circlenode, thick, right=1cm of v35](v356){$356$};
\node[circlenode, thick, right=1cm of v356](v36){$36$};
\node[circlenode, thick, right=1cm of v36](v6){$6$};
\node[circlenode, thick, right=1cm of v6](v0c){$\emptyset$};
\node[above=0cm of v0c, font=\Large]{leaf};

\foreach \u / \v in {v0a/v1, v1/v12, v12/v123, v123/v23, v23/v235a, v235a/v235b, v235b/v25, v25/v245, v245/v24, v24/v4, v4/v0b, v235a/v235c, v235c/v35, v35/v356, v356/v36, v36/v6, v6/v0c}
    \draw[thick] (\u)--(\v);
\end{tikzpicture}
}
\caption{Nice tree decomposition of $G$. This decomposition is essentially identical to $(T, \{X_t\}_{t \in T})$, but this representation makes easier to develop and analyze dynamic programming algorithms.}
\label{fig:ntc}
\end{minipage}
\end{figure} 

Next we design dynamic programming tables.
Given a nice tree decomposition $(T, \{X_t\}_{t \in T}) $ of
the graph $ ([n], \nnz(\mat{A}) \cup \nnz(\mat{B})) $,
we assume to be given an ordering $\prec_t$ on $V_t$ for node $t$ of $T$, whose definition is deferred to \cref{subsubsec:fpt:treewidth:proof-2}.
We aim to compute the following quantity for each node $t$:
\begin{align}
\label{eq:dp-sum}
    \underset{\substack{
    S \subseteq V_t \setminus X_t \qquad\qquad\qquad\qquad\qquad\qquad \\
    O_{A1}, O_{A2} \subseteq X_t: |O_{A1}|=|O_{A2}|,
    \sigma_A: S \uplus O_{A1} \bij S \uplus O_{A2} \\
    O_{B1}, O_{B2} \subseteq X_t: |O_{B1}|=|O_{B2}|,
    \sigma_B: S \uplus O_{B1} \bij S \uplus O_{B2}
    }}
    {\qquad \sum \quad \sgn_{\prec_t}(\sigma_A) \sgn_{\prec_t}(\sigma_B)
    \mat{A}(\sigma_A) \mat{B}(\sigma_B)}.
\end{align}
Recall that $\mat{A}(\sigma) \triangleq \prod_i A_{i,\sigma(i)}$ for bijection $\sigma$.
In particular,
\cref{eq:dp-sum} is equal to $ \ZZ_2(\mat{A}, \mat{B}) $
at the root $r$ since $ X_r=\emptyset $ and $ V_r=[n] $.
We then discuss how to group exponentially many bijections into an FPT number of bins.
A \emph{configuration for node} $t$ is defined as a tuple $ \state = (O_1, O_2, F_1, F_2, \tau, \nu) $, where
\begin{itemize}
    \item $O_1, O_2 \subseteq X_t$ are subsets such that $|O_1|=|O_2|$;
    \item $F_1 \subseteq O_1, F_2 \subseteq O_2$ are subsets such that $ |F_1|=|F_2|$;
    \item $\tau: O_1 \setminus F_1 \bij O_2 \setminus F_2$ is a bijection;
    \item $\nu \in \{0,1\}$ is the parity of inversion number.
\end{itemize}
In the remainder of this subsection,
\emph{arithmetic operations on the parity of inversion number to be performed over modulo $2$}, and
the symbol $\equiv$ means congruence modulo $2$.
We say that a bijection $\sigma$ is \emph{consistent} with
$ S \subseteq V_t \setminus X_t $ and
$\state = (O_1, O_2, F_1, F_2, \tau, \nu)$ if
\begin{itemize}
    \item $\sigma$ is a bijection $S \uplus O_1 \bij S \uplus O_2$;
    \item $F_1 = \sigma^{-1}(S) \cap O_1$ and $F_2 = \sigma(S) \cap O_2$;
    \item $\tau = \sigma|_{O_1 \setminus F_1}$;
    \item $\nu \equiv \inv_{\prec_t}(\sigma)$.
\end{itemize}
We show that
for any bijection appearing in \cref{eq:dp-sum},
there exists a unique pair of a subset $S$ and a configuration $\state$ that is consistent with the bijection.

\begin{lemma}
\label{lem:bijection}
Let $ S \subseteq V_t \setminus X_t $ and
$ O_1, O_2 \subseteq X_t$ be two subsets such that $ |O_1| = |O_2| $.
For any bijection $ \sigma: S \uplus O_1 \bij S \uplus O_2 $,
there exists a unique configuration $\state$ for $t$ that
$\sigma$ is consistent with.
\end{lemma}
\begin{proof}
Since $\sigma$ is a bijection,
we can let $F_1 \triangleq \sigma^{-1}(S) \cap O_1$
$F_2 \triangleq \sigma(S) \cap O_2$,
$\tau \triangleq \sigma|_{O_1 \setminus F_1}$, and
$\nu \triangleq \inv_{\prec_t}(\sigma)$.
Then, $\sigma$ must be consistent with $S$ and $\state$.
Uniqueness is obvious from the definition of configuration and consistency.
\end{proof}
Hereafter, we will use $ \Sigma(S, \state) $
to denote the set of all bijections
consistent with a subset $S \subseteq V_t \setminus X_t$ and
a configuration $\state$ for a node $t$ of $T$.
By \cref{lem:bijection}, we have that
\begin{align}
\label{eq:bijection-partition}
\biguplus_{\substack{
S \subseteq V_t \setminus X_t \\ \state \text{ for } t
}}
\Sigma(S, \state) =
\Bigl\{ \sigma: S \uplus O_1 \bij S \uplus O_2 \mid 
S \subseteq V_t \setminus X_t, O_1 \subseteq X_t, O_2 \subseteq X_t, |O_1|=|O_2| \Bigr\}.
\end{align}
We can thus express \cref{eq:dp-sum} as follows:
\begin{align*}
    \text{value of \cref{eq:dp-sum}} & = 
    \sum_{\substack{
        S \in V_t \setminus X_t \\
        \state_A = (O_{A1}, O_{A2}, F_{A1}, F_{A2}, \tau_A, \nu_A) \text{ for } t \\
        \state_B = (O_{B1}, O_{B2}, F_{B1}, F_{B2}, \tau_B, \nu_B) \text{ for } t
    }}
    \sum_{\substack{
        \sigma_A \in \Sigma(S, \state_A) \\
        \sigma_B \in \Sigma(S, \state_B)
    }}
    \sgn_{\prec_t}(\sigma_A) \sgn_{\prec_t}(\sigma_B) \mat{A}(\sigma_A) \mat{B}(\sigma_B) \\
    & = \sum_{\substack{
        \state_A, \state_B \text{ for } t \\ 0 \leq s \leq n
    }}
    (-1)^{\nu_A+\nu_B}
    \sum_{S \subseteq {V_t \setminus X_t \choose s}}
    \Upsilon_{t, A}(S, \state_A) \cdot
    \Upsilon_{t, B}(S, \state_B),
\end{align*}
where we define $\Upsilon_{t, A}$ and $\Upsilon_{t, B}$ as
\begin{align*}
\Upsilon_{t, A}(S, \state_A) & \triangleq
\sum_{\sigma_A \in \Sigma(S, \state_A)} \mat{A}(\sigma_A), \\
\Upsilon_{t, B}(S, \state_B) & \triangleq
\sum_{\sigma_B \in \Sigma(S, \state_B)} \mat{B}(\sigma_B).
\end{align*}
We now define a dynamic programming table $dp_{t,s}$ for each node $t \in T$ and each integer $s \in [0\isep n]$ so as to store the following quantity with key
$\left[{\state_A \atop \state_B}\right]$:
\begin{align*}
    dp_{t,s}\left[{
    \state_A \atop
    \state_B
    }\right] \triangleq
    \sum_{S \in {V_t \setminus X_t \choose s}}
    \Upsilon_{t, A}(S, \state_A) \cdot
    \Upsilon_{t, B}(S, \state_B).
\end{align*}
Since there are at most $ 2^{|X_t|} 2^{|X_t|} 2^{|X_t|} 2^{|X_t|} |X_t|! 2 \leq 16^{w+1} (w+1)! 2 $ possible configurations for node $t$ by definition,
$dp_{t,s}$ contains at most $ w^{\bigO(w)}$ entries.
The number of bits required to represent each entry of $dp_{t,s}$ is roughly bounded by $ \log(2^n n!)(\isize(\mat{A}) + \isize(\mat{B})) = \bigO((\isize(\mat{A}) + \isize(\mat{B})) n \log n) $.

Having defined the dynamic programming table,
we are ready to construct $dp_{t,s}$ given
already-filled $dp_{t',s'}$ for children $t'$ of $t$ and $s' \in [0\isep n]$;
the proof is deferred to \cref{subsubsec:fpt:treewidth:proof-2}.

\begin{lemma}
\label{lem:table-update}
Let $t$ be a non-leaf node of $T$, and $s \in [0\isep n]$.
Given $dp_{t',s'}$ for all children $t'$ of $t$ and $s' \in [0\isep n]$,
we can compute each entry of $ dp_{t,s} $ in $ w^{\bigO(w)} n^{\bigO(1)}$ time.
\end{lemma}

\begin{proof}[Proof of \cref{thm:fpt-treewidth-2}]
Our parameterized algorithm works as follows.
Given a tree decomposition for $ ([n], \nnz(\mat{A}) \cup \nnz(\mat{B})) $
of width at most $w$,
we transform it to a nice tree decomposition
$(T, \{X_t\}_{t \in T})$ rooted at $r$
of width at most $w$ that
has $\bigO(wn)$ nodes in polynomial time \citep{cygan2015parameterized}.
For every leaf $\ell$ of $T$,
any configuration $(O_1,O_2,F_1,F_2,\tau,\nu)$ for $\ell$ satisfies that $O_1 = O_2 = F_1 = F_2 = \emptyset$ and $\tau: \emptyset \bij \emptyset$ 
because $X_\ell = \emptyset$.
We thus initialize $dp_{\ell,s}$ so that
\begin{align*}
 dp_{\ell,s}\left[{
    \emptyset, \emptyset, \emptyset, \emptyset, \emptyset \bij \emptyset, \nu_A \atop
    \emptyset, \emptyset, \emptyset, \emptyset, \emptyset \bij \emptyset, \nu_B
}\right] =
    \begin{cases}
    1 & \text{if } s=0 \text{ and } \nu_A=\nu_B=0, \\
    0 & \text{otherwise}.
    \end{cases}
\end{align*}
Then, for each non-leaf node $t$,
we apply \cref{lem:table-update}
to fill $dp_{t,s}$ using the already-filled $dp_{t',s'}$ for all children $t'$ of $t$
in a bottom-up fashion.
Completing dynamic programming,
we compute $\ZZ_2$ using $dp_{r,s}$ at the root $r$ as follows:
\begin{align*}
\ZZ_2(\mat{A}, \mat{B}) = \sum_{
s \in [0\isep n], \nu_A, \nu_B \in \{0,1\}
}
    (-1)^{\nu_A+\nu_B} 
    dp_{r,s}\left[{
    \emptyset, \emptyset, \emptyset, \emptyset, \emptyset \bij \emptyset, \nu_A \atop
    \emptyset, \emptyset, \emptyset, \emptyset, \emptyset \bij \emptyset, \nu_B }\right].
\end{align*}
Correctness follows from \cref{lem:bijection,lem:table-update}.
We finally bound the time complexity.
Because
$T$ has at most $\bigO(wn)$ nodes,
each table is of size $ w^{\bigO(w)} n^{\bigO(1)} $, and
each table entry can be computed in $w^{\bigO(w)} n^{\bigO(1)}$ time by
\cref{lem:table-update},
the whole time complexity is bounded by
$ w^{\bigO(w)} n^{\bigO(1)}$,
thereby completing the proof.
\end{proof}

\begin{remark}
Our dynamic programming implies that an FPT algorithm exists for
permanental processes \citep{macchi1975coincidence} since it holds that
\begin{align*}
  \sum_{S \subseteq [n]} \per(\mat{A}_S) \per(\mat{B}_S) =
\sum_{s, \nu_A, \nu_B} 
    dp_{r,s}\left[{
    \emptyset, \emptyset, \emptyset, \emptyset, \emptyset \bij \emptyset, \nu_A \atop
    \emptyset, \emptyset, \emptyset, \emptyset, \emptyset \bij \emptyset, \nu_B }\right].  
\end{align*}
\end{remark}

\cref{thm:fpt-treewidth-2} can be generalized 
to the case of $m$ matrices $ \mat{A}^1, \ldots \mat{A}^m $.
Computing $\ZZ_m$ parameterized by the treewidth of
$\nnz(\mat{A}^1) \cup \cdots \cup \nnz(\mat{A}^m)$
\emph{plus} the number of matrices $m$ is FPT,
whose proof is deferred to \cref{subsubsec:fpt:treewidth:proof-m}.

\begin{theorem}
\label{thm:fpt-treewidth-m}
For a positive integer $m$,
let $\mat{A}^1, \ldots, \mat{A}^m$ be
$m$ matrices in $\bbQ^{n \times n}$.
Then, there exists a $w^{\bigO(mw)} n^{\bigO(1)}$-time algorithm that,
given a tree decomposition of
the graph $ ([n], \bigcup_{i \in [m]} \nnz(\mat{A}^i)) $
of width at most $w$,
computes $\ZZ_m(\mat{A}^1, \ldots, \mat{A}^m)$ exactly.
\end{theorem}

\subsubsection{Proof of \cref{lem:table-update}}
\label{subsubsec:fpt:treewidth:proof-2}

We first define an ordering $\prec_t$ on $V_t$ for each node $t$ of $T$.
\begin{definition}
\label{def:order}
An ordering $\prec_t$ on set $V_t$ for node $t$ is recursively defined as follows.
\begin{itemize}
\item If $t$ is a leaf: $\prec_t$ is just an ordering on the empty set $V_t = \emptyset$.
\item If $t$ is an introduce node with one child $t'$ such that $X_t = X_{t'}+v$:
Given $\prec_{t'}$ on set $V_{t'}$,
we define $\prec_t$ on set $V_t = V_{t'}+v$ as follows:
\begin{itemize}
\item $\prec_t$ and $\prec_{t'}$ agree on $V_{t}-v = V_{t'}$;
\item $V_{t}-v \prec_t \{v\}$.
\end{itemize}

\item If $t$ is a forget node with one child $t'$ such that $X_t = X_{t'} - v$:
Given $\prec_{t'}$ on set $V_{t'}$,
we define $\prec_t$ on set $V_t = V_{t'}$ as follows:
\begin{itemize}
\item $\prec_t$ and $\prec_{t'}$ agree on $V_t - v$;
\item $V_t \setminus X_t - v \prec_t \{v\} \prec_t X_t$.
\end{itemize}

\item If $t$ is a join node with two children $t',t''$ such that $X_t = X_{t'} = X_{t''}$:
Given $\prec_{t'}$ and $\prec_{t''}$ on set $V_{t'}$ and $V_{t''}$, respectively,
we define $\prec_t$ on set $V_t = V_{t'} \cup V_{t''}$ as follows:
\begin{itemize}
\item $V_{t'} \setminus X_{t'} \prec_t V_{t''} \setminus X_{t''} \prec_t X_t$;
\item $\prec_t$ and $\prec_{t'}$ agree on $V_{t'} \setminus X_{t'}$;
\item $\prec_t$ and $\prec_{t''}$ agree on $V_{t''} \setminus X_{t''}$;
\item $\prec_t$ and $\prec_{t'}$ agree on $X_t = X_{t'} = X_{t''}$.
\end{itemize}

\end{itemize}
\end{definition}
By construction, we have that $V_t \setminus X_t \prec_t X_t$ for every node $t$ of $T$.
We have an auxiliary lemma that plays a role in updating the parity of inversion number.
\begin{lemma}
\label{lem:inv-Xt}
Let $\state = (O_1, O_2, F_1, F_2, \tau, \nu)$ be a configuration for node $t$,
and $\prec_x$ and $\prec_y$ be two orderings on $V_t$ such that
\begin{itemize}
    \item $\prec_x$ and $\prec_y$ agree on $V_t \setminus X_t$
    (i.e., $v \prec_x w$ if and only if $v \prec_y w$ for all $v,w \in V_t \setminus X_t$);
    \item $V_t \setminus X_t \prec_x X_t$ and $V_t \setminus X_t \prec_y X_t$
    (i.e., $v \prec_x w$ and $v \prec_y w$ for all $v \in V_t \setminus X_t$ and $w \in X_t$).
\end{itemize}
Then, we can compute a 0-1 integer $\Delta = \Delta(\state, \prec_x, \prec_y)$ in polynomial time such that
$\inv_{\prec_x}(\sigma) - \inv_{\prec_y}(\sigma) \equiv \Delta$
for all $\sigma \in \Sigma(S, \state)$ and $S \subseteq V_t \setminus X_t$.
\end{lemma}
\begin{proof}
Given two orderings $\prec_x$ and $\prec_y$ on $V_t$ that meet the assumption,
we can construct a sequence of orderings, denoted
$\prec^{(0)}, \prec^{(1)}, \ldots, \prec^{(\ell-1)}, \prec^{(\ell)}$,
starting from $\prec_x = \prec^{(0)}$ and ending with $\prec_y = \prec^{(\ell)}$ such that
each $\prec^{(i)}$ for $i \in [\ell]$ is obtained from $\prec^{(i-1)}$
by reversing the order of two consecutive elements of $X_t$ with regard to $\prec^{(i-1)}$;
i.e.,
there exists a partition of $V_t$, denoted $P \uplus \{v,w\} \uplus Q$,
where $v,w \in X_t$ and $Q \subseteq X_t$, such that
\begin{itemize}
    \item $\prec^{(i-1)}$ and $\prec^{(i)}$ agree on $P \uplus Q$;
    \item $P \prec^{(i-1)} \{v\} \prec^{(i-1)} \{w\} \prec^{(i-1)} Q$;
    \item $P \prec^{(i)} \{w\} \prec^{(i)} \{v\} \prec^{(i)} Q$.
\end{itemize}
Since it holds that
\begin{align*}
\inv_{\prec_y}(\sigma) - \inv_{\prec_x}(\sigma) = 
\sum_{i \in [\ell]} \inv_{\prec^{(i)}}(\sigma) - \inv_{\prec^{(i-1)}}(\sigma),
\end{align*}
we hereafter assume the existence of a partition $P \uplus \{v,w\} \uplus Q$,
where $v,w\in X_t$ and $Q \subseteq X_t$, such that
\begin{itemize}
    \item $\prec_x$ and $\prec_y$ agree on $P \uplus Q$;
    \item $P \prec_x \{v\} \prec_x \{w\} \prec_x Q$;
    \item $P \prec_y \{w\} \prec_y \{v\} \prec_y Q$.
\end{itemize}

Let $\sigma$ be a bijection in $\Sigma(S, \state)$ for $S \subseteq V_t \setminus X_t$.
In order for $\inv_{\prec_x}(\sigma)$ and $\inv_{\prec_y}(\sigma)$ to differ,
one of the following conditions must be satisfied:
\begin{itemize}
\item $\sigma(v) \prec_x \sigma(w)$ and $\sigma(v) \prec_y \sigma(w)$;
\item $\sigma(v) \succ_x \sigma(w)$ and $\sigma(v) \succ_y \sigma(w)$;
\item $\sigma^{-1}(v) \prec_x \sigma^{-1}(w)$ and $\sigma^{-1}(v) \prec_y \sigma^{-1}(w)$;
\item $\sigma^{-1}(v) \succ_x \sigma^{-1}(w)$ and $\sigma^{-1}(v) \succ_y \sigma^{-1}(w)$.
\end{itemize}
The value of $\Delta \equiv \inv_{\prec_y}(\sigma) - \inv_{\prec_x}(\sigma)$
can be determined based on the following case analysis:
\begin{itemize}
    \item \textbf{Case (1)} $\{v,w\} \not\subseteq O_1$, $\{v,w\} \not\subseteq O_2$:
        Observe easily that $\inv_{\prec_x}(\sigma) = \inv_{\prec_y}(\sigma)$; i.e., $\Delta = 0$.
    \item \textbf{Case (2)} $\{v,w\} \subseteq O_1$, $\{v,w\} \not\subseteq O_2$:
        Since we have that $\{v,w\} \not\subseteq \{\sigma(v), \sigma(w)\}$,
        \begin{itemize}
            \item if $\sigma(v) \prec_x \sigma(w)$, then $ \sigma(v) \prec_y \sigma(w)$, and
            thus $\inv_{\prec_y}(\sigma) = \inv_{\prec_x}(\sigma) + 1$;
            \item if $\sigma(v) \succ_x \sigma(w)$, then $ \sigma(v) \succ_y \sigma(w)$, and
            thus $\inv_{\prec_y}(\sigma) = \inv_{\prec_x}(\sigma) - 1$.
        \end{itemize}
        In either case, $\Delta = 1$.
    \item \textbf{Case (3)} $\{v,w\} \not\subseteq O_1$, $\{v,w\} \subseteq O_2$:
        Since we have that $\{v,w\} \not\subseteq \{\sigma^{-1}(v), \sigma^{-1}(w)\}$,
        \begin{itemize}
            \item if $\sigma^{-1}(v) \prec_x \sigma^{-1}(w)$, then $\sigma^{-1}(v) \prec_y \sigma^{-1}(w)$, and
            thus $\inv_{\prec_y}(\sigma) = \inv_{\prec_x}(\sigma) + 1$;
            \item if $\sigma^{-1}(v) \succ_x \sigma^{-1}(w)$, then $\sigma^{-1}(v) \succ_y \sigma^{-1}(w)$, and
            thus $\inv_{\prec_y}(\sigma) = \inv_{\prec_x}(\sigma) - 1$.
        \end{itemize}
        In either case, $\Delta = 1$.
    \item \textbf{Case (4)} $\{v,w\} \subseteq O_1$, $\{v,w\} \subseteq O_2$:
    We prove that $\Delta = 0$ by exhaustion on the size of $\{v,w\} \cap \{\sigma(v), \sigma(w)\}$.
    \begin{itemize}
        \item \textbf{Case (4-1)} $|\{v,w\} \cap \{\sigma(v), \sigma(w)\}| = 2$:
            Since we have that $\{v,w\} = \{\sigma(v), \sigma(w)\} = \{\sigma^{-1}(v), \sigma^{-1}(w)\}$,
            \begin{itemize}
                \item if $\sigma(v) \prec_x \sigma(w)$, then $\sigma(v) \succ_y \sigma(w)$;
                \item if $\sigma(v) \succ_x \sigma(w)$, then $\sigma(v) \prec_y \sigma(w)$;
                \item if $\sigma^{-1}(v) \prec_x \sigma^{-1}(w)$, then $\sigma^{-1}(v) \succ_y \sigma^{-1}(w)$;
                \item if $\sigma^{-1}(v) \succ_x \sigma^{-1}(w)$, then $\sigma^{-1}(v) \prec_y \sigma^{-1}(w)$.
            \end{itemize}
            Thus, $\inv_{\prec_y}(\sigma) = \inv_{\prec_x}(\sigma)$.
        \item \textbf{Case (4-2)} $|\{v,w\} \cap \{\sigma(v), \sigma(w)\}| = 1$:
            \begin{itemize}
                \item if $\sigma(v) \prec_x \sigma(w)$, then $\sigma(v) \prec_y \sigma(w)$;
                \item if $\sigma(v) \succ_x \sigma(w)$, then $\sigma(v) \succ_y \sigma(w)$;
                \item if $\sigma^{-1}(v) \prec_x \sigma^{-1}(w)$, then $\sigma^{-1}(v) \prec_y \sigma^{-1}(w)$;
                \item if $\sigma^{-1}(v) \succ_x \sigma^{-1}(w)$, then $\sigma^{-1}(v) \succ_y \sigma^{-1}(w)$.
            \end{itemize}
            Thus, $\inv_{\prec_y}(\sigma) - \inv_{\prec_x}(\sigma)$ takes $-2$, $0$, or $2$ as a value.
        \item \textbf{Case (4-3)} $|\{v,w\} \cap \{\sigma(v), \sigma(w)\}| = 0$:
            Since we have that
            $\{v,w\} \cap \{\sigma(v), \sigma(w)\} = \emptyset$ and
            $\{v,w\} \cap \{\sigma^{-1}(v), \sigma^{-1}(w)\} = \emptyset$,
            \begin{itemize}
                \item if $\sigma(v) \prec_x \sigma(w)$, then $\sigma(v) \prec_y \sigma(w)$;
                \item if $\sigma(v) \succ_x \sigma(w)$, then $\sigma(v) \succ_y \sigma(w)$;
                \item if $\sigma^{-1}(v) \prec_x \sigma^{-1}(w)$, then $\sigma^{-1}(v) \prec_y \sigma^{-1}(w)$;
                \item if $\sigma^{-1}(v) \succ_x \sigma^{-1}(w)$, then $\sigma^{-1}(v) \succ_y \sigma^{-1}(w)$.
            \end{itemize}
            Thus, $\inv_{\prec_y}(\sigma) - \inv_{\prec_x}(\sigma)$ takes $-2$, $0$, or $2$ as a value.
    \end{itemize}
    Consequently, in either case, $\Delta = 0$.
\end{itemize}
We can determine which case $\sigma$ falls into \emph{without} looking into $\sigma$,
which completes the proof.
\end{proof}

We are now ready to prove \cref{lem:table-update}.
The proof is separated into the following three lemmas.
\begin{lemma}
\label{lem:introduce}
Let $t$ be an introduce node with one child $t'$ such that $X_t = X_{t'} + v$, and $s \in [0\isep n]$.
Given $dp_{t',s'}$ for all $s'$,
we can compute each entry of $ dp_{t,s} $ in $n^{\bigO(1)}$ time.
\end{lemma}

\begin{lemma}
\label{lem:forget}
Let $t$ be a forget node with one child $t'$ such that $X_t = X_{t'} - v$, and $s \in [0\isep n]$.
Given $dp_{t',s'}$ for all $s'$,
we can compute each entry of $dp_{t,s}$ in $w^{\bigO(w)} n^{\bigO(1)}$ time.
\end{lemma}

\begin{lemma}
\label{lem:join}
Let $t$ be a join node with two children $t'$ and $t''$ such that $X_t = X_{t'} = X_{t''}$, and $s \in [0\isep n]$.
Given $dp_{t',s'}$ and $dp_{t'',s''}$ for all
$s'$ and $s''$, respectively,
we can compute each entry of $dp_{t,s}$ in $w^{\bigO(w)} n^{\bigO(1)}$ time.
\end{lemma}

\paragraph{Proof of \cref{lem:introduce}.}

Consider a bijection $\sigma \in \Sigma(S,\state)$ for
$S \subseteq V_t \setminus X_t$ and
$\state = (O_1, O_2, F_1, F_2, \tau, \nu)$ for $t$,
Then, a restriction $\sigma|_{V_{t'}}$ may belong to $\Sigma(S',\state')$ for some $S' \subseteq V_{t'}\setminus X_{t'}$ and $\state'$ for $t'$.
We will show that such $S'$ and $\state'$ can be determined independent of $\sigma$.

Observe first that if $v \in F_1$ or $v \in F_2$,
then we can declare that $\Upsilon_{t,A}(S,\state) = 0$:
if this is the case,
(1) any bijection $\sigma$ in $\Sigma(S,\state)$ satisfies that $\sigma(v) \in V_t \setminus X_t$ or $\sigma^{-1}(v) \in V_t \setminus X_t$, while
(2) $\mat{A}_{\{v\}, V_t \setminus X_t}$ and $\mat{A}_{V_t \setminus X_t, \{v\}}$ must be zero matrices by the separator property of a tree decomposition.
Hereafter, we can safely assume that $v \not\in F_1$ and $v \not\in F_2$.

We first discuss the relation between $\inv_{\prec_t}(\sigma)$ and $\inv_{\prec_{t'}}(\sigma|_{V_{t'}})$.
\begin{lemma}
\label{lem:intro-inv}
Let $\state = (O_1, O_2, F_1, F_2, \tau, \nu)$ be a configuration for $t$.
Then, there exists a 0-1 integer $\Delta = \Delta(\state)$ such that
$\nu \equiv \inv_{\prec_{t'}}(\sigma|_{V_{t'}}) + \Delta$
for any $\sigma \in \Sigma(S,\state)$ for $S \subseteq V_t \setminus X_t$.
Moreover, we can compute the value of $\Delta$ in polynomial time.
\end{lemma}
\begin{proof}
Let $\sigma \in \Sigma(S, \state)$ for $S \subseteq V_t \setminus X_t$.
By \cref{def:order}, we have that $\inv_{\prec_{t'}}(\sigma|_{V_{t'}}) = \inv_{\prec_t}(\sigma|_{V_{t'}})$.
Simple calculation yields that
\begin{align*}
    \nu \equiv \inv_{\prec_t}(\sigma) & = |\{ (u,w) \mid u \prec_t w, \sigma(u) \succ_t \sigma(w) \}| \\
    & = \underbrace{|\{ (u,w) \mid u \prec_t w, \sigma(u) \succ_t \sigma(w), \{u,w,\sigma(u),\sigma(w)\} \subseteq V_{t'} \}|}_{\inv_{\prec_t}(\sigma|_{V_{t'}})} \\
    & + |\{ (u,w) \mid u \prec_t w, \sigma(u) \succ_t \sigma(w), v \in \{u,w,\sigma(u),\sigma(w)\} \}| \\
    & \equiv \inv_{\prec_{t'}}(\sigma|_{V_{t'}}) + \Delta.
\end{align*}
Observing that $V_{t'} \prec_t \{v\}$,
we can determine the value of $\Delta$ based on the following case analysis:
\begin{itemize}
\item \textbf{Case (1)} $v \not\in O_1$, $v \not\in O_2$:
Since $\sigma = \sigma|_{V_{t'}}$, we have that $\Delta \equiv 0$.
\item \textbf{Case (2)} $v \in O_1\setminus F_1$, $v \not\in O_2$:
Since $\sigma(v) = \tau(v) \in X_t$, we have that
\begin{align*}
    \Delta \equiv |\{ (u,v) \mid \sigma(u) \succ_t \sigma(v) \}| = |\{ w \in O_2 \mid w \succ_t \tau(v) \}|.
\end{align*}
\item \textbf{Case (3)} $v \not\in O_1$, $v \in O_2 \setminus F_2$:
Since $\sigma^{-1}(v) = \tau^{-1}(v) \in X_t$, we have that
\begin{align*}
    \Delta \equiv |\{ (\sigma^{-1}(v), w) \mid \sigma^{-1}(v) \prec_t w \}| = |\{ w \in O_1 \mid w \succ_t \tau^{-1}(v) \}|.
\end{align*}
\item \textbf{Case (4)} $v \in O_1 \setminus F_1$, $v \in O_2 \setminus F_2$, $\sigma(v)=\tau(v)=v$:
Since $u \prec_t v$ and $\tau(u) \prec_t \tau(v)$ for all $u \in V_{t'}$, we have that $\Delta = 0$.
\item \textbf{Case (5)} $v \in O_1\setminus F_1$, $v \in O_2 \setminus F_2$, $\sigma(v)=\tau(v)\neq v$:
We have that
\begin{align*}
    \Delta & = |\{ (u,v) \mid \sigma(u) \succ_t \sigma(v) \} \uplus \{ (\sigma^{-1}(v), w) \mid \sigma^{-1}(v) \prec_t w \} \uplus \{ (\sigma^{-1}(v), v) \}| \\
    & = |\{ w \in O_2 \mid w \succ_t \tau(v) \}| + |\{ w \in O_1 \mid w \succ_t \tau^{-1}(v) \}| + 1.
\end{align*}
\end{itemize}
We can determine which case $\state$ falls into and calculate $\Delta$ in polynomial time, as desired.
\end{proof}

We then define a mapping $\intro$ from a configuration for $t$ to a configuration for $t'$.
\begin{definition}
\label{def:intro-map}
Let $\state = (O_1,O_2,F_1,F_2,\tau,\nu)$ be a configuration for $t$ such that
$v \not \in F_1$ and $v \not \in F_2$.
We define $\intro(\state)$ as follows:
\begin{oframed}
\begin{center}
    \textbf{Definition of $\intro(\state)$.}
\end{center}
Compute $\Delta$ according to \cref{lem:intro-inv}.
\begin{itemize}
\item \textbf{Case (1)} $v \not\in O_1$, $v \not\in O_2$:
We define
\begin{align*}
    \intro(\state) \triangleq (O_1, O_2, F_1, F_2, \tau|_{X_{t'}}, \nu-\Delta).
\end{align*}
\item \textbf{Case (2)} $v \in O_1\setminus F_1$, $v \not\in O_2$:
We define
\begin{align*}
    \intro(\state) \triangleq (O_1-v, O_2-\tau(v), F_1, F_2, \tau|_{X_{t'}}, \nu-\Delta).
\end{align*}
\item \textbf{Case (3)} $v \not\in O_1$, $v \in O_2 \setminus F_2$:
We define
\begin{align*}
    \intro(\state) \triangleq (O_1-\tau^{-1}(v), O_1-v, F_1, F_2, \tau|_{X_{t'}}, \nu-\Delta).
\end{align*}
\item \textbf{Case (4)} $v \in O_1 \setminus F_1$, $v \in O_2 \setminus F_2$, $\tau(v)=v$:
We define
\begin{align*}
    \intro(\state) \triangleq (O_1-v, O_2-v, F_1, F_2, \tau|_{X_{t'}}, \nu-\Delta).
\end{align*}
\item \textbf{Case (5)} $v \in O_1\setminus F_1$, $v \in O_2 \setminus F_2$, $\tau(v)\neq v$:
We define
\begin{align*}
    \intro(\state) \triangleq (O_1-v-\tau^{-1}(v), O_2-\tau(v)-v, F_1, F_2, \tau|_{X_{t'}}, \nu-\Delta).
\end{align*}
\end{itemize}
\end{oframed}
\end{definition}
We claim the following:
\begin{claim}
\label{lem:intro-completeness}
For any $\state$ for $t$, there exists unique $\state'$ for $t'$ such that
for any $\sigma \in \Sigma(S,\state)$ for $S \subseteq V_t \setminus X_t$,
it holds that $\sigma|_{V_{t'}} \in \Sigma(S,\state')$.
Moreover, $\state' = \intro(\state)$.
\end{claim}
\begin{proof}
The proof is immediate from \cref{def:intro-map,lem:intro-inv}.
\end{proof}

For any bijection $\sigma \in \Sigma(S,\state)$ for $S \subseteq V_t\setminus X_t$ and
$\state = (O_1,O_2,F_1,F_2,\tau,\nu)$ for $t$,
it holds that 
\begin{align*}
    \mat{A}(\sigma) = \mat{A}(\sigma|_{V_{t'}}) \cdot
    \begin{cases}
    1 & \text{Case } (1) \\
    A_{v,\tau(v)} & \text{Case }(2) \\
    A_{\tau^{-1}(v),v} & \text{Case }(3) \\
    A_{v,v} & \text{Case }(4) \\
    A_{v,\tau(v)} \cdot A_{\tau^{-1}(v),v} & \text{Case }(5)
    \end{cases}
    .
\end{align*}
By \cref{lem:intro-completeness}, we have that
\begin{align*}
    \Upsilon_{t,A}(S,\state)
    & = \sum_{\sigma \in \Sigma(S,\state)} \mat{A}(\sigma)
    = \sum_{\sigma' \in \Sigma(S,\intro(\state))} \mat{A}(\sigma') \cdot 
    \begin{cases}
    1 & \text{Case } (1) \\
    A_{v,\tau(v)} & \text{Case }(2) \\
    A_{\tau^{-1}(v),v} & \text{Case }(3) \\
    A_{v,v} & \text{Case }(4) \\
    A_{v,\tau(v)} \cdot A_{\tau^{-1}(v),v} & \text{Case }(5)
    \end{cases} \\
    & = \Upsilon_{t',A}(S,\intro(\state)) \cdot
    \begin{cases}
    1 & \text{Case }(1) \\
    A_{v,\tau(v)} & \text{Case }(2) \\
    A_{\tau^{-1}(v),v} & \text{Case }(3) \\
    A_{v,v} & \text{Case }(4) \\
    A_{v,\tau(v)} \cdot A_{\tau^{-1}(v),v} & \text{Case }(5)
    \end{cases}.
\end{align*}
We have an analogue with regard to $\Upsilon_{t,B}(S,\state)$.
Observing that $V_t \setminus X_t = V_{t'} \setminus X_{t'}$,
we finally obtain that
for $\state_A = (O_{A1},O_{A2},F_{A1},F_{A2},\tau_A,\nu_A)$ and
$\state_B = (O_{B1},O_{B2},F_{B1},F_{B2},\tau_B,\nu_B)$ for $t$,
\begin{align*}
    dp_{t,s}\left[{\state_A \atop \state_B}\right]
    & = \sum_{S \in {V_t\setminus X_t \choose s}} \Upsilon_{t,A}(S,\state_A) \cdot \Upsilon_{t,B}(S,\state_B) \\
    & = dp_{t',s}\left[{\intro(\state_A) \atop \intro(\state_B)}\right] \cdot
    \underbrace{
    \begin{cases}
    1 & (1) \\
    A_{v,\tau_A(v)} & (2) \\
    A_{\tau_A^{-1}(v),v} & (3) \\
    A_{v,v} & (4) \\
    A_{v,\tau_A(v)} \cdot A_{\tau_A^{-1}(v),v} & (5)
    \end{cases}
    }_{\text{division into cases by } \state_A}
    \cdot
    \underbrace{
    \begin{cases}
    1 & (1) \\
    B_{v,\tau_B(v)} & (2) \\
    B_{\tau_B^{-1}(v),v} & (3) \\
    B_{v,v} & (4) \\
    B_{v,\tau_B(v)} \cdot B_{\tau_B^{-1}(v),v} & (5)
    \end{cases}
    }_{\text{division into cases by } \state_B}
\end{align*}
if $v \not\in F_{A1} \cup F_{A2} \cup F_{B1} \cup F_{B2}$.
Otherwise, it holds that $dp_{t,s}\left[{\state_A \atop \state_B}\right] = 0$.
Because evaluating $\intro(\state_A)$ and $\intro(\state_B)$ completes in $n^{\bigO(1)}$ time,
so does evaluating $dp_{t,s}\left[{\state_A \atop \state_B}\right]$.
\qed

\paragraph{Proof of \cref{lem:forget}.}
Consider a bijection $\sigma' \in \Sigma(S',\state')$ for
$S' \subseteq V_{t'} \setminus X_{t'}$ and $\state'$ for $t'$.
Since $V_t$ is equal to $V_{t'}$, $\sigma'$ may belong to $\Sigma(S,\state)$ for some $S \subseteq V_t \setminus X_t$ and $\state$ for $t$.
We will show that if this is the case, such $S$ and $\state$ can be determined independent of $\sigma'$.

We first discuss the relation between $\inv_{\prec_{t'}}(\sigma')$ and $\inv_{\prec_t}(\sigma')$.
\begin{lemma}
\label{lem:forget-inv}
Let $\state' = (O_1',O_2',F_1',F_2',\tau',\nu')$ be a configuration for $t'$.
Then, there exists a 0-1 integer $\Delta = \Delta(\state')$ such that
$\inv_{\prec_t}(\sigma') \equiv \nu' + \Delta$ for any $\sigma' \in \Sigma(S', \state')$ for $S' \subseteq V_{t'} \setminus X_{t'}$.
Moreover, we can compute the value of $\Delta$ in polynomial time.
\end{lemma}
\begin{proof}
The proof is a direct consequence of \cref{lem:inv-Xt} since
$\prec_t$ and $\prec_{t'}$ satisfy the conditions in \cref{lem:inv-Xt}.
\end{proof}

We then define a mapping $\forget$ from a set-configuration pair for $t'$ to a set-configuration pair for $t$.
\begin{definition}
\label{def:forget-map}
Let $S' \subseteq V_{t'} \setminus X_{t'}$ and $\state' = (O_1',O_2',F_1',F_2',\tau',\nu')$.
Then, we define $\forget(S',\state')$ as follows:
\begin{oframed}
\begin{center}
    \textbf{Definition of $\forget(S',\state')$.}
\end{center}
\begin{itemize}
    \item \textbf{Case (1)} $v \not \in O_1'$, $v \not \in O_2'$:
    We define $\forget(S', \state') \triangleq (S', (O_1',O_2',F_1',F_2',\tau',\nu'+\Delta))$, where
    $\Delta$ is computed according to \cref{lem:forget-inv}.
    
    \item \textbf{Case (2)} $v \in O_1'$, $v \in O_2' $:
    $F_1$ and $F_2$ are defined as follows:
    \begin{itemize}
        \item \textbf{Case (2-1)} $v \in F_1'$, $v \in F_2'$:
            $F_1 \triangleq F_1'-v$, $F_2 = F_2'-v$.
        \item \textbf{Case (2-2)} $v \not \in F_1'$, $v \in F_2'$:
            $F_1 \triangleq F_1'$, $F_2 \triangleq F_2' + \tau'(v) - v$.
        \item \textbf{Case (2-3)} $v \in F_1'$, $v \not \in F_2'$:
            $F_1 \triangleq F_1' + \tau'^{-1}(v) - v$, $F_2 \triangleq F_2'$.
        \item \textbf{Case (2-4)} $v \not \in F_1'$, $v \not \in F_2'$, $\tau(v) \neq v$:
            $F_1 \triangleq F_1'+\tau'^{-1}(v)$, $F_2 \triangleq F_2'+\tau'(v)$.
        \item \textbf{Case (2-5)} $v \not \in F_1'$, $v \not \in F_2'$, $\tau(v) = v$:
            $F_1 \triangleq F_1'$, $F_2 \triangleq F_2'$.
    \end{itemize}

    We define
    \begin{align*}
        \forget(S', \state') \triangleq (S'+v, (O_1'-v, O_2'-v, F_1, F_2, \tau'|_{X_t \cap \tau'^{-1}(X_t)}, \nu'+\Delta)),
    \end{align*}
    where $\Delta$ is computed according to \cref{lem:forget-inv}.
    
    \item \textbf{Case (3)} $v \not \in O_1'$, $v \in O_2'$: $\forget(S',\state')$ is undefined.
    \item \textbf{Case (4)} $v \in O_1'$, $v \not \in O_2'$: $\forget(S',\state')$ is undefined.
\end{itemize}
\end{oframed}
\end{definition}

Here, we claim a kind of completeness and soundness of $\forget$.
\begin{claim}
\label{lem:forget-completeness}
For any $S \subseteq V_t \setminus X_t$, $\state$ for $t$, and $\sigma \in \Sigma(S,\state)$,
there exist unique $S' \subseteq V_{t'} \setminus X_{t'}$ and $\state'$ for $t'$ such that
$\sigma \in \Sigma(S', \state')$.
Moreover, $\forget(S',\state') = (S,\state)$.
\end{claim}
\begin{proof}
Given $S \subseteq V_t \setminus X_t$, $\state = (O_1, O_2, F_1, F_2, \tau, \nu)$, and $\sigma \in \Sigma(S, \state)$,
we construct $S' \subseteq V_{t'} \setminus X_{t'}$ and $\state'$ for $t'$ as follows:

\begin{description}
    \item \textbf{Case (1)} $v \not \in S$:
    We define $S' \triangleq S$ and $\state' \triangleq (O_1, O_2, F_1, F_2, \tau, \inv_{\prec_{t'}}(\sigma))$.
    \item \textbf{Case (2)} $v \in S$:
    $F_1'$ and $F_2'$ are defined as follows:
    \begin{description}
        \item \textbf{Case (2-1)} $\sigma(v) \not \in X_{t'}$, $\sigma^{-1}(v) \not \in X_{t'}$:
        $F_1' \triangleq F_1+v, F_2'\triangleq F_2+v$.
        \item \textbf{Case (2-2)} $\sigma(v) \in X_{t'}$, $\sigma^{-1}(v) \not \in X_{t'}$:
        $F_1' \triangleq F_1, F_2' \triangleq F_2-\sigma(v)+v$.
        \item \textbf{Case (2-3)} $\sigma(v) \not \in X_{t'}$, $\sigma^{-1}(v) \in X_{t'}$:
        $F_1' \triangleq F_1-\sigma^{-1}(v)+v, F_2' \triangleq F_2$.
        \item \textbf{Case (2-4)} $\sigma(v) \in X_{t'}$, $\sigma^{-1}(v) \in X_{t'}$, $\sigma(v) \neq v$:
        $F_1' \triangleq F_1-\sigma^{-1}(v), F_2' \triangleq F_2-\sigma(v)$.
        \item \textbf{Case (2-5)} $\sigma(v) \in X_{t'}$, $\sigma^{-1}(v) \in X_{t'}$, $\sigma(v) = v$:
        $F_1' \triangleq F_1, F_2' \triangleq F_2$.
    \end{description}
    We define
    $S' \triangleq S-v$ and $\state' \triangleq (O_1+v, O_2+v, F_1', F_2', \sigma|_{X_{t'} \cap \sigma^{-1}(X_{t'})}, \inv_{\prec_{t'}}(\sigma))$.
\end{description}
It is easy to verify that $\sigma \in \Sigma(S', \state')$ for all $\sigma \in \Sigma(S, \state)$ and $\forget(S', \state') = (S, \state)$.
The uniqueness is obvious.
\end{proof}

\begin{claim}
\label{lem:forget-soundness}
For any $S' \subseteq V_{t'} \setminus X_{t'}$ and $\state'$ for $t'$ such that
$\forget(S',\state')$ is defined,
there exists unique $S \subseteq V_t \setminus X_t$ and $\state$ for $t$ such that
any $\sigma' \in \Sigma(S',\state')$ belongs to $(S,\state)$.
Moreover, $\forget(S',\state') = (S,\state)$.
\end{claim}
\begin{proof}
The proof is immediate from \cref{def:forget-map,lem:forget-inv}.
\end{proof}

Since $\forget$ determines a configuration for $t$ based only on a configuration for $t'$,
we abuse the notation by writing $\forget(\state') = \state$ if
there exist $S' \subseteq V_{t'} \setminus X_{t'}$ and
$S \subseteq V_t \setminus X_t$ such that $\forget(S',\state') = (S,\state)$.
By definition of $\forget$ and \cref{lem:forget-completeness,lem:forget-soundness},
we have that for $S \subseteq V_t \setminus X_t$ and $\state$ for $t$,
\begin{align*}
    \Sigma(S,\state)
    = \biguplus_{\substack{
        S' \subseteq V_{t'}\setminus X_{t'}, \state' \text{ for } t' \\ \forget(S',\state') = (S, \state)
    }} \Sigma(S', \state')
    = \begin{cases}
    \displaystyle \biguplus_{\substack{
        \state' \text{ for } t' \\ \forget(\state') = \state \\ v \not\in O_1', v \not\in O_2'
    }} \Sigma(S, \state')
    & \text{if } v \not\in S, \\
    \displaystyle \biguplus_{\substack{
        \state' \text{ for } t' \\ \forget(\state') = \state \\ v \in O_1', v \in O_2'
    }} \Sigma(S-v, \state')
    & \text{if } v \in S.
    \end{cases}
\end{align*}
It thus turns out that 
\begin{align*}
    \Upsilon_{t,A}(S, \state) = 
    \begin{cases}
    \displaystyle\sum_{\substack{
        \state' \text{ for } t' \\ \forget(\state') = \state \\ v \not\in O_1', v \not\in O_2'
    }} \Upsilon_{t',A}(S, \state') & \text{if } v \not\in S, \\
    \displaystyle\sum_{\substack{
        \state' \text{ for } t' \\ \forget(\state') = \state \\ v \in O_1', v \in O_2'
    }} \Upsilon_{t',A}(S-v, \state') & \text{if } v \in S.
    \end{cases}
\end{align*}
We have an analogue regarding $\Upsilon_{t,B}$.
Observing that $V_{t'} \setminus X_{t'} = V_t \setminus X_t -v$,
we decompose $dp_{t,s}$ into the sum over $dp_{t', s'}$ as follows.
\begin{align*}
    dp_{t,s}\left[\state_A \atop \state_B \right]
    & = \sum_{S \in {V_t \setminus X_t \choose s}: v \not\in S}
    \Upsilon_{t,A}(S, \state_A) \cdot \Upsilon_{t,B}(S, \state_B) 
    + \sum_{S \in {V_t \setminus X_t \choose s}: v \in S}
    \Upsilon_{t,A}(S, \state_A) \cdot \Upsilon_{t,B}(S, \state_B) \\
    & = \sum_{S \in {V_t \setminus X_t \choose s}: v \not \in S}
    \sum_{\substack{
        \state_A' \text{ for } t' \\
        \forget(\state_A') = \state_A \\
        v \not \in O_{A1}', v \not \in O_{A2}'
    }}
    \sum_{\substack{
        \state_B' \text{ for } t' \\
        \forget(\state_B') = \state_B \\
        v \not \in O_{B1}', v \not \in O_{B2}'
    }}
    \Upsilon_{t', A}(S, \state_A') \cdot
    \Upsilon_{t', B}(S, \state_B') \\
    & +
    \sum_{S \in {V_t \setminus X_t \choose s}: v \in S}
    \sum_{\substack{
        \state_A' \text{ for } t' \\
        \forget(\state_A') = \state_A \\
        v \in O_{A1}', v \in O_{A2}'
    }}
    \sum_{\substack{
        \state_B' \text{ for } t' \\
        \forget(\state_B') = \state_B \\
        v \in O_{B1}', v \in O_{B2}'
    }}
    \Upsilon_{t', A}(S-v, \state_A') \cdot
    \Upsilon_{t', B}(S-v, \state_B')
    \\
    & =
    \sum_{\substack{
        \state_A' \text{ for } t' \\
        \forget(\state_A') = \state_A \\
        v \not \in O_{A1}', v \not \in O_{A2}'
    }}
    \sum_{\substack{
        \state_B' \text{ for } t' \\
        \forget(\state_B') = \state_B \\
        v \not \in O_{B1}', v \not \in O_{B2}'
    }}
    \sum_{S' \in {V_{t'} \setminus X_{t'} \choose s}}
    \Upsilon_{t', A}(S', \state_A') \cdot
    \Upsilon_{t', B}(S', \state_B') \\
    & +
    \sum_{\substack{
        \state_A' \text{ for } t' \\
        \forget(\state_A') = \state_A \\
        v \in O_{A1}', v \in O_{A2}'
    }}
    \sum_{\substack{
        \state_B' \text{ for } t' \\
        \forget(\state_B') = \state_B \\
        v \in O_{B1}', v \in O_{B2}'
    }}
    \sum_{S' \in {V_{t'} \setminus X_{t'} \choose s-1}}
    \Upsilon_{t', A}(S', \state_A') \cdot
    \Upsilon_{t', B}(S', \state_B')
    \\
    & =
    \sum_{\substack{
        \state_A' \text{ for } t' \\
        \forget(\state_A') = \state_A \\
        v \not \in O_{A1}', v \not \in O_{A2}'
    }}
    \sum_{\substack{
        \state_B' \text{ for } t' \\
        \forget(\state_B') = \state_B \\
        v \not \in O_{B1}', v \not \in O_{B2}'
    }}
    dp_{t',s}\left[{\state_A' \atop \state_B'} \right]
    +
    \sum_{\substack{
        \state_A' \text{ for } t' \\
        \forget(\state_A') = \state_A \\
        v \in O_{A1}', v \in O_{A2}'
    }}
    \sum_{\substack{
        \state_B' \text{ for } t' \\
        \forget(\state_B') = \state_B \\
        v \in O_{B1}', v \in O_{B2}'
    }}
    dp_{t',s-1}\left[{\state_A' \atop \state_B'}\right].
\end{align*}
Note that we define $dp_{t',-1}[\cdot] \triangleq 0$.
Running through all possible combinations of
$ \state_A' $ and $\state_B'$,
we can compute $dp_{t,s}\left[\state_A \atop \state_B \right]$ by $ w^{\bigO(w)} n^{\bigO(1)} $ arithmetic operations.
\qed

\paragraph{Proof of \cref{lem:join}.}
Consider a bijection $\sigma' \in \Sigma(S',\state')$ for $S' \subseteq V_{t'} \setminus X_{t'}$ and $\state'$ for $t'$ and
a bijection $\sigma'' \in \Sigma(S'', \sigma'')$ for $S'' \subseteq V_{t''} \setminus X_{t''}$ and $\state''$ for $t''$.
We would like to examine a new bijection $\sigma$ obtained by \emph{concatenating} $\sigma'$ and $\sigma''$,
which may belong to $\Sigma(S, \state)$ for some $S \subseteq V_t \setminus X_t$ and $\state$ for $t$.
For this purpose, we first define the concatenation of two bijections.

\begin{definition}
\label{def:concat}
Given two bijections $\sigma_1: S_1 \bij T_1$ and $\sigma_2: S_2 \bij T_2$ such that $|S_1| = |T_1$ and $|S_2| = |T_2|$,
we assume that
\begin{itemize}
    \item $\sigma_1(i) = \sigma_2(i)$ for all $i \in S_1 \cap S_2$;
    \item $\sigma_1(i_1) \neq \sigma_2(i_2)$ for all $i_1 \in S_1 \setminus S_2$ and $i_2 \in S_2 \setminus S_1$.
\end{itemize}
Then, the \emph{concatenation} of $\sigma_1$ and $\sigma_2$, denoted $\sigma_1 \sqcup \sigma_2$, is defined as
a bijection from $S_1 \cup S_2$ to $T_1 \cup T_2$ such that
\begin{align*}
    (\sigma_1 \sqcup \sigma_2)(i) =
    \begin{cases}
    \sigma_1(i) & \text{if } i \in S_1, \\
    \sigma_2(i) & \text{if } i \in S_2 \setminus S_1.
    \end{cases}
\end{align*}
\end{definition}

We first discuss the relation between $\inv_{\prec_{t'}}(\sigma')$, $\inv_{\prec_{t''}}(\sigma'')$, and $\inv_{\prec_t}(\sigma' \sqcup \sigma'')$.
\begin{lemma}
\label{lem:join-inv}
Given $s' \in [0 \isep n]$,
$\state' = (O_1',O_2',F_1',F_2',\tau',\nu')$ for $t'$,
$s'' \in [0 \isep n]$, and
$\state'' = (O_1'',O_2'',F_1'',F_2'',\tau'',\nu'')$ for $t''$,
we can compute a 0-1 integer $\nu$ in polynomial time such that
for all 
$\sigma' \in \Sigma(S', \state')$ with $S' \in {V_{t'} \setminus X_{t'} \choose s'}$ and
$\sigma'' \in \Sigma(S'', \state'')$ with $S' \in {V_{t''} \setminus X_{t''} \choose s''}$
with $\sigma' \sqcup \sigma''$ defined,
$\inv_{\prec_t}(\sigma' \sqcup \sigma'') \equiv \nu$.
\end{lemma}
\begin{proof}
Define $\sigma \triangleq \sigma' \sqcup \sigma''$.
Since $\sigma' \sqcup \sigma''$ is defined,
it must hold that $\tau' = \tau''$,
According to the definition of $\sigma' \sqcup \sigma''$,
we can expand $\inv_{\prec_t}(\sigma' \sqcup \sigma'')$ as follows:
\begin{align*}
& \inv_{\prec_t}(\sigma' \sqcup \sigma'') = \{ (v,w) \mid v \prec_t w, \sigma(v) \succ_t \sigma(w) \} \\
& = \underbrace{\{ (v,w) \mid v \prec_t w, \sigma(v) \succ_t \sigma(w), \{v, w, \sigma(v), \sigma(w)\} \subseteq V_{t'} \}}_{\inv_{\prec_t}(\sigma')} \\
& + \underbrace{\{ (v,w) \mid v \prec_t w, \sigma(v) \succ_t \sigma(w), \{v, w, \sigma(v), \sigma(w)\} \subseteq V_{t''} \}}_{\inv_{\prec_t}(\sigma'')} \\
& - \underbrace{\{ (v,w) \mid v \prec_t w, \sigma(v) \succ_t \sigma(w), \{v, w, \sigma(v), \sigma(w)\} \subseteq V_{t'} \cap V_{t''} = X_t \}}_{\inv_{\prec_t}(\tau')} \\
& + \left\{ (v,w) \mid v \prec_t w, \sigma(v) \succ_t \sigma(w),
    {\{v, w, \sigma(v), \sigma(w)\} \cap V_{t'} \setminus X_{t'} \neq \emptyset,
    \atop \{v, w, \sigma(v), \sigma(w)\} \cap V_{t''} \setminus X_{t''} \neq \emptyset} \right\}.
\end{align*}
Since
$\sigma(v') = \sigma'(v') \not\in V_{t''} \setminus X_{t''}$ for all $v' \in V_{t'} \setminus X_{t'}$ and
$\sigma(v'') = \sigma''(v'') \not\in V_{t'} \setminus X_{t'}$ for all $v'' \in V_{t''} \setminus X_{t''}$,
it holds that
\begin{align*}
& \left\{ (v,w) \mid v \prec_t w, \sigma(v) \succ_t \sigma(w),
    {\{v, w, \sigma(v), \sigma(w)\} \cap V_{t'} \setminus X_{t'} \neq \emptyset,
    \atop \{v, w, \sigma(v), \sigma(w)\} \cap V_{t''} \setminus X_{t''} \neq \emptyset} \right\} \\
& = \underbrace{\{ (v',v'') \mid v' \in V_{t'} \setminus X_{t'}, v'' \in V_{t''} \setminus X_{t''}, \sigma(v') \succ_t \sigma(v'') \}}_{|S''| \cdot |F_1'| + |\{ (v',v'') \in F_1' \times F_1'' \mid v' \succ_t v'' \}|} \\
& + \underbrace{\{ (\sigma^{-1}(v'), \sigma^{-1}(v'')) \mid v' \in V_{t'} \setminus X_{t'}, v'' \in V_{t'} \setminus X_{t''}, \sigma^{-1}(v') \succ_t \sigma^{-1}(v'') \}}_{|S''| \cdot |F_2'| + |\{ (v',v'') \in F_2' \times F_2'' \mid v' \succ_t v'' \}|} \\
& - \underbrace{\{ (v', v'') \mid v', \sigma(v') \in V_{t'} \setminus X_{t'}, v'', \sigma(v'') \in V_{t'} \setminus X_{t''} \}}_{0}.
\end{align*}
Consequently, $\inv_{\prec_t}(\sigma' \sqcup \sigma'')$ is equal to
\begin{align*}
    & \inv_{\prec_t}(\sigma') + \inv_{\prec_t}(\sigma'') - \inv_{\prec_t}(\tau') \\
    & \; + |S''| \cdot (|F_1'| + |F_2|') + |\{ (v',v'') \in F_1' \times F_1'' \mid v' \succ_t v'' \}| + |\{ (v',v'') \in F_2' \times F_2'' \mid v' \succ_t v'' \}|.
\end{align*}
By \cref{lem:inv-Xt}, we can compute 0-1 integers $\Delta'$ and $\Delta''$ 
(which are independent of $\sigma'$ and $\sigma''$) such that
$\inv_{\prec_t}(\sigma') \equiv \inv_{\prec_{t'}}(\sigma') + \Delta' \equiv \nu' + \Delta'$ and
$\inv_{\prec_t}(\sigma'') \equiv \inv_{\prec_{t''}}(\sigma'') + \Delta'' \equiv \nu'' + \Delta''$.
Since $\inv_{\prec_t}(\tau')$ can be computed naively and
the remaining terms are easy-to-compute,
we can compute (the parity of) $\inv_{\prec_t}(\sigma' \sqcup \sigma'')$, which completes the proof.
\end{proof}

We then define a mapping $\join$ from a set-configuration pair for $t'$ and a set-configuration pair for $t''$ to a set-configuration pair for $t$.
\begin{definition}
\label{def:join-map}
Let $S' \subseteq V_{t'}\setminus X_{t'}$,
$\state' = (O_1', O_2', F_1', F_2', \tau', \nu')$ for node $t'$,
$S'' \subseteq V_{t''}\setminus X_{t''}$, and
$\state'' = (O_1'', O_2'', F_1'', F_2'', \tau'', \nu'')$ for node $t''$.
Then, we define $\join(S', \state', S'', \state'')$ as follows:

\begin{oframed}
\begin{center}
    \textbf{Definition of $\join(S', \state', S'', \state'').$}
\end{center}
\begin{itemize}
\item \textbf{Case (1)} If the following conditions are satisfied:
    \begin{itemize}
        \item $O_1' \setminus F_1' = O_1'' \setminus F_1''$;
        \item $O_2' \setminus F_2' = O_2'' \setminus F_2''$;
        \item $F_1' \cap F_1'' = \emptyset$;
        \item $F_2' \cap F_2'' = \emptyset$;
        \item $\tau' = \tau''$.
    \end{itemize}
    We define 
    \begin{align*}
    \join(S', \state', S'', \state'') & \triangleq (S' \uplus S'', \state), \text{where} \\
    \state & \triangleq (O_1' \cup O_1'', O_2' \cup O_2'', F_1' \uplus F_1'', F_2' \uplus F_2'', \tau', \nu),
    \end{align*}
    where $\nu$ is computed according to \cref{lem:join-inv}.
\item \textbf{Case (2)} Otherwise: $\join(S', \state', S'', \state'')$ is undefined.
\end{itemize}
\end{oframed}
\end{definition}

Here, we claim a kind of completeness and soundness of $\join$.
\begin{claim}
\label{lem:join-completeness}
For any $S \subseteq V_t \setminus X_t$, $\state$ for $t$, and $\sigma \in \Sigma(S, \state)$,
there exist unique
$\sigma' \in \Sigma(S', \state')$ and $\sigma'' \in \Sigma(S'', \state'')$ for
$S' \in V_{t'} \setminus X_{t'}$,
$\state'$ for $t'$,
$S'' \in V_{t''} \setminus X_{t''}$, and
$\state''$ for $t''$
such that $\sigma' \sqcup \sigma'' = \sigma$.
Moreover, $\join(S', \state', S'', \state'') = (S,\state)$.
\end{claim}
\begin{proof}
Observe that $X_t = X_{t'} = X_{t''}$ and $V_t \setminus X_t = (V_{t'} \setminus X_{t'}) \uplus (V_{t''} \setminus X_{t''})$.
Given $S \subseteq V_t \setminus X_t$,
$\state = (O_1, O_2, F_1, F_2, \tau, \nu)$, and
$\sigma \in \Sigma(S, \state)$,
we first define $\sigma' \triangleq \sigma|_{V_{t'}} $ and $\sigma'' \triangleq \sigma|_{V_{t''}}$.
Observe that $\sigma' \sqcup \sigma'' = \sigma$.
We then construct $S', S'', \state', \state''$ as

\begin{align*}
    S'  & \triangleq S \cap (V_{t'} \setminus X_{t'}), &
    S'' & \triangleq S \cap (V_{t''} \setminus X_{t''}), \\
    \state' & \triangleq (O_1',O_2',F_1',F_2',\tau',\nu'), &
    \state'' & \triangleq (O_1'',O_2'',F_1'',F_2'',\tau'',\nu''),
\end{align*}
where we further define
\begin{align*}
    O_1'  & \triangleq \sigma^{-1}(S' \uplus O_2) \cap X_t, &
    O_1'' & \triangleq \sigma^{-1}(S'' \uplus O_2) \cap X_t, \\
    O_2'  & \triangleq \sigma(S' \uplus O_1) \cap X_t, &
    O_2'' & \triangleq \sigma(S'' \uplus O_1) \cap X_t, \\
    F_1'  & \triangleq \sigma^{-1}(S') \cap X_t, &
    F_1'' & \triangleq \sigma^{-1}(S'') \cap X_t, \\
    F_2'  & \triangleq \sigma(S') \cap X_t, &
    F_2'' & \triangleq \sigma(S'') \cap X_t, \\
    \tau'  & \triangleq \sigma|_{X_t \cap \sigma^{-1}(X_t)}, &
    \tau'' & \triangleq \sigma|_{X_t \cap \sigma^{-1}(X_t)}, \\
    \nu'  & \triangleq \inv_{\prec_{t'}}(\sigma'), &
    \nu'' & \triangleq \inv_{\prec_{t''}}(\sigma''). \\
\end{align*}
It is easy to verify that $\sigma' \in \Sigma(S', \state')$ and $\sigma'' \in \Sigma(S'', \state'')$ and that
$\join(S', \state', S'', \state'') = (S, \state)$.
Uniqueness is obvious.
\end{proof}

\begin{claim}
\label{lem:join-soundness}
For any $S' \subseteq V_{t'}\setminus X_{t'}$, 
$\state'$ for $t'$,
$S'' \subseteq V_{t''}\setminus X_{t''}$,
$\state''$ for $t''$
such that $\join(S',\state',S'',\state'')$ is defined,
there exist unique
$S \subseteq V_t \setminus X_t$ and $\state$ for $t$ such that
for any $\sigma' \in \Sigma(S', \state')$ and $\sigma'' \in \Sigma(S'', \state'')$,
it holds that $\sigma' \sqcup \sigma'' \in \Sigma(S,\state)$.
Moreover, $\join(S',\state',S'',\state'') = (S,\state)$.
\end{claim}
\begin{proof}
The proof is immediate from \cref{def:join-map,lem:join-inv}.
\end{proof}

By \cref{lem:join-completeness,lem:join-soundness}, we have that 
for $S \subseteq V_t \setminus X_t$ and $\state$ for $t$,
\begin{align*}
    \Sigma(S, \state) = \Bigl\{ \sigma' \sqcup \sigma''
        \mid \sigma' \in \Sigma(S', \state'), \sigma'' \in \Sigma(S'', \state''), \join(S', \state', S'', \state'') = (S, \state) \Bigr\}.
\end{align*}

Suppose $\join(S', \state', S'', \state'') = (S, \state)$ for
some $S$, $S'$, $S''$, $\state$, $\state'$, $\state''$.
By \cref{def:join-map}, we have that
$S = S'\uplus S''$, $S' = S \cap (V_{t'} \setminus X_{t'})$, and $S'' = S \cap (V_{t''} \setminus X_{t''})$.
Further, $\join$ determines $\state$ based only on $|S'|$, $\state'$, $|S''|$, $\state''$.
We thus abuse the notation by writing
$\join(s',\state',s'',\state'') = \state$ for two integers $s'$ and $s''$ if
there exist $S' \in {V_{t'}\setminus X_{t'} \choose s'}$ and $S'' \in {V_{t''}\setminus X_{t''} \choose s''}$ such that
$\join(S',\state',S'',\state'') = (S,\state)$.

By \cref{lem:join-completeness,lem:join-soundness} and the the above discussion, we have that for
$S' \subseteq V_{t'} \setminus X_{t'}$,
$S'' \subseteq V_{t''} \setminus X_{t''}$, and
$\state$ for $t$,
\begin{align*}
    \Sigma(S' \uplus S'', \state) = \Bigl\{ \sigma' \sqcup \sigma'' \mid
    \sigma' \in \Sigma(S', \state'), \sigma'' \in \Sigma(S'', \state''),
    \join(|S'|, \state', |S''|, \state'') = \state
    \Bigr\}.
\end{align*}
Since any bijection $\sigma' \sqcup \sigma'' \in \Sigma(S' \uplus S'', \state)$ for
$\state = (O_1, O_2, F_1, F_2, \tau, \nu)$ satisfies that
\begin{align*}
\mat{A}(\sigma' \sqcup \sigma'') = \frac{\mat{A}(\sigma') \cdot \mat{A}(\sigma'')}{\mat{A}(\tau)},
\end{align*}
we have that
\begin{align*}
\Upsilon_{t,A}(S' \uplus S'', \state)
& = \sum_{\sigma \in \Sigma(S' \uplus S'', \state)} \mat{A}(\sigma) \\
& = \sum_{\substack{
\state' \text{ for } t', \state'' \text{ for } t'' \\
\join(|S'|, \state', |S''|, \state'') = \state
}}
\sum_{\substack{
    \sigma' \in \Sigma(S', \state') \\ \sigma'' \in \Sigma(S'', \state'')
}}
\mat{A}(\sigma' \sqcup \sigma'') \\
& = \frac{1}{\mat{A}(\tau)} 
\sum_{\substack{
\state' \text{ for } t', \state'' \text{ for } t'' \\
\join(|S'|, \state', |S''|, \state'') = \state
}}
\sum_{\sigma' \in \Sigma(S', \state')} \mat{A}(\sigma')
\sum_{\sigma'' \in \Sigma(S'', \state'')} \mat{A}(\sigma'') \\
& = \frac{1}{\mat{A}(\tau)}
\sum_{\substack{
\state' \text{ for } t', \state'' \text{ for } t'' \\
\join(|S'|, \state', |S''|, \state'') = \state
}}
\Upsilon_{t',A}(S', \state') \cdot \Upsilon_{t',A}(S'', \state'').
\end{align*}
We have an analogue regarding $\Upsilon_{t,B}$.
Observing the quality  that
\begin{align*}
    {V_t \setminus X_t \choose s} = \biguplus_{\substack{
        s', s'' \in [0\isep s] \\ s' + s'' = s    
    }}
    \left\{ S' \uplus S'' \mid S' \in {V_{t'} \setminus X_{t'} \choose s'}, S'' \in {V_{t''} \setminus X_{t''} \choose s''} \right\},
\end{align*}
we can decompose $dp_{t,s}$ into the sum over $dp_{t',s'}$ and $dp_{t'',s''}$ as follows.
\begin{align*}
& dp_{t,s}\left[{\state_A \atop \state_B}\right] = \sum_{S \in {V_t \setminus X_t \choose s}} \Upsilon_{t,A}(S, \state_A) \cdot \Upsilon_{t,B}(S, \state_B) \\
& = \sum_{\substack{s', s'' \in [0\isep s] \\ s'+s''=s}}
\sum_{\substack{
    S' \in {V_{t'} \setminus X_{t'} \choose s'} \\
    S'' \in {V_{t'} \setminus X_{t''} \choose s''}
}}
\Upsilon_{t,A}(S' \uplus S'', \state_A) \cdot \Upsilon_{t,B}(S' \uplus S'', \state_B) \\
& = \sum_{\substack{
    s', s'' \in [0\isep s] \\
    s'+s''=s \\
    S' \in {V_{t'} \setminus X_{t'} \choose s'} \\
    S'' \in {V_{t'} \setminus X_{t''} \choose s''}
}}
\sum_{\substack{
    \state_A' \text{ for } t', \state_A'' \text{ for } t'' \\ \join(s',\state_A',s'',\state_A'') = \state_A
}}
\frac{\Upsilon_{t',A}(S',\state_A') \cdot \Upsilon_{t'',A}(S'',\state_A'')}{\mat{A}(\tau)}
\sum_{\substack{
    \state_B' \text{ for } t', \state_B'' \text{ for } t'' \\ \join(s',\state_B',s'',\state_B'') = \state_B
}}
\frac{\Upsilon_{t',B}(S',\state_B') \cdot \Upsilon_{t'',B}(S'',\state_B'')}{\mat{B}(\tau)} \\
& = \frac{1}{\mat{A}(\tau) \cdot \mat{B}(\tau)}
\sum_{\substack{
    s',s'': s'+s'' = s \\
    \state_A', \state_A'', \state_B', \state_B'' \\
    \join(s', \state_A', s'', \state_A'') = \state_A \\
    \join(s', \state_B', s'', \state_B'') = \state_B
}}
\Biggl( \sum_{S' \in {V_{t'} \setminus X_{t'} \choose s'}} \Upsilon_{t',A}(S',\state_A') \cdot \Upsilon_{t',B}(S',\state_B') \Biggr)
\\
& \qquad\qquad\qquad\qquad\qquad\qquad\qquad
\Biggl( \sum_{S'' \in {V_{t''} \setminus X_{t''} \choose s''}} \Upsilon_{t'',A}(S'',\state_A'') \cdot \Upsilon_{t'',B}(S'',\state_B'') \Biggr) \\
& = \frac{1}{\mat{A}(\tau) \cdot \mat{B}(\tau)}
\sum_{\substack{
    s',s'': s'+s'' = s \\
    \state_A', \state_A'', \state_B', \state_B'' \\
    \join(s', \state_A', s'', \state_A'') = \state_A \\
    \join(s', \state_B', s'', \state_B'') = \state_B
}}
dp_{t',s'}\left[{\state_A' \atop \state_B'}\right] \cdot dp_{t'',s''}\left[{\state_A'' \atop \state_B''}\right].
\end{align*}
Running through all possible combinations of
$s'$, $s''$,
$ \state_A' $, $\state_A''$, $\state_B'$, $\state_B''$,
we can compute $dp_{t,s}\left[{\state_A \atop \state_B}\right]$ by
$ w^{\bigO(w)} n^{\bigO(1)} $ arithmetic operations.
\qed

\subsubsection{Proof of \cref{thm:fpt-treewidth-m}}
\label{subsubsec:fpt:treewidth:proof-m}

Let $\mat{A}^1, \ldots, \mat{A}^m$ be $m$ matrices in $\bbQ^{n \times n}$ and
$ (T, \{X_v\}_{v \in T}) $ be a nice tree decomposition of graph
$ ([n], \bigcup_{i \in [m]} \nnz(\mat{A}^i)) $,
which is of width at most $w$ and rooted at $r \in T$.
We aim to compute the following quantity for each node $t$:
\begin{align}
\label{eq:dp-sum-m}
    \sum_{\substack{
    S \subseteq V_t \setminus X_t \\
    }}
    \prod_{i \in [m]}
    \sum_{\substack{
        O_{i,1}, O_{i,2} \subseteq X_t: |O_{i,1}|=|O_{i,2}| \\
        \sigma_i: S \uplus O_{i,1} \bij S \uplus O_{i,2}
    }}
    \sgn_{\prec_t}(\sigma_i) \mat{A}^i(\sigma_i).
\end{align}
In particular,
\cref{eq:dp-sum-m} is equal to $\ZZ_m(\mat{A}^1, \ldots, \mat{A}^m)$ at the root $r$.
A \emph{configuration for node $t$} is defined as a tuple
$ \state = (O_1, O_2, F_1, F_2, \tau, \nu) $
in the same manner as in \cref{subsec:fpt:treewidth}.
Due to \cref{lem:bijection,eq:bijection-partition},
letting $ \state_i = (O_{i,1}, O_{i,2}, F_{i,1}, F_{i,2}, \tau_i, \nu_i) $ for $i \in [m]$ be a configuration for node $t$ of $T$,
we can express \cref{eq:dp-sum-m} as follows:
\begin{align*}
    \sum_{\substack{
    \state_1, \ldots, \state_m \text{ for } t \\ 0 \leq s \leq n
    }}
    (-1)^{\nu_1 + \cdots + \nu_m}
    \sum_{S \subseteq {V_t \setminus X_t \choose s}}
    \prod_{i \in [m]} \Upsilon_{t,i}(S,\state_i),
\end{align*}
where for all $i \in [m]$, we define
\begin{align*}
    \Upsilon_{t,i}(S, \state_i) \triangleq \sum_{\sigma_i \in \Sigma(S, \state_i)} \mat{A}^i(\sigma_i).
\end{align*}

We then define a dynamic programming table $dp_{t,s}$ for each $t \in T$ and $s \in [0\isep n]$
to store the following quantity with key
$\left[\begin{matrix}\state_1 \\ \vdots \\ \state_m\end{matrix}\right]$:
\begin{align*}
    dp_{t,s}\left[\begin{matrix}
    \state_1 \\ \vdots \\ \state_m
    \end{matrix}\right] \triangleq
    \sum_{S \in {V_t \setminus X_t \choose s}}
    \prod_{i \in [m]} \Upsilon_{t,i}(S,\state_i).
\end{align*}
By definition, $dp_{t,s}$ contains at most $ w^{\bigO(mw)} $ entries.
The number of bits required to represent each entry of $dp_{t,s}$ is roughly bounded by $ \bigO(\sum_{i \in [m]}\isize(\mat{A}^i) n \log n) $.

We can easily extend the proof of \cref{lem:introduce,lem:forget,lem:join} for the case of $m$ matrices as follows.

\begin{lemma}
\label{lem:introduce-m}
Let $t$ be an introduce node with one child $t'$ such that $X_t = X_{t'} + v$, and $s \in [0\isep n]$.
Given $dp_{t',s'}$ for all $s'$,
we can compute each entry of $ dp_{t,s} $ in $ (mn)^{\bigO(1)} $ time.
\end{lemma}
\begin{proof}
Using a mapping $\intro$ introduced in the proof of \cref{lem:introduce},
we have that for $m$ configurations $\state_1, \ldots, \state_m$ for $t$,
\begin{align}
    dp_{t,s}\left[\begin{matrix}
    \state_1 \\ \vdots \\ \state_m
    \end{matrix}\right] =
    dp_{t',s}\left[\begin{matrix}
    \intro(\state_1) \\ \vdots \\ \intro(\state_m)
    \end{matrix}\right]
    \prod_{i \in [m]}
    \underbrace{
    \begin{cases}
    1 & \text{Case } (1), \\
    A^i_{v,\tau_i(v)} & \text{Case } (2), \\
    A^i_{\tau_i^{-1}(v),v} & \text{Case } (3), \\
    A^i_{v,v} & \text{Case } (4), \\
    A^i_{v,\tau_i(v)} \cdot A^1_{\tau_i^{-1}(v),v} & \text{Case } (5).
    \end{cases}
    }_{\text{division into cases by } \state_i}
\end{align}
Since evaluating $\intro(\state_i)$ for each $i \in [m]$
completes in $(mn)^{\bigO(1)}$ time, so does evaluating $dp_{t,s}$.
\end{proof}

\begin{lemma}
\label{lem:forget-m}
Let $t$ be a forget node with one child $t'$ such that $X_t = X_{t'} - v $, and $s \in [0\isep n]$.
Given $dp_{t',s'}$ for all $s'$,
we can compute each entry of $dp_{t,s}$ in $w^{\bigO(wm)} n^{\bigO(1)}$ time.
\end{lemma}
\begin{proof}
Using a mapping $\forget$ introduced in the proof of \cref{lem:forget},
for any $S \subseteq V_t \setminus X_t$,
$\state$ for $t$, and $i \in [m]$,
we have that 
\begin{align*}
    \Upsilon_{t,i}(S, \state) = 
    \begin{cases}
    \displaystyle\sum_{\substack{
        \state' \text{ for } t' \\ \forget(\state') = \state \\ v \not\in O_1', v \not\in O_2'
    }} \Upsilon_{t',i}(S, \state') & \text{if } v \not\in S, \\
    \displaystyle\sum_{\substack{
        \state' \text{ for } t' \\ \forget(\state') = \state \\ v \in O_1', v \in O_2'
    }} \Upsilon_{t',i}(S-v, \state') & \text{if } v \in S.
    \end{cases}
\end{align*}
Consequently, $dp_{t,s}$ can be decomposed into two sums as follows.

\begin{align*}
dp_{t,s}\left[\begin{matrix}
\state_1 \\ \vdots \\ \state_m
\end{matrix}\right]
= \sum_{\substack{
        \state_1' \text{ for } t' \\
        \forget(\state_1') = \state_1 \\
        v \not \in O_{1,1}', v \not \in O_{1,2}'
    }}
    \cdots
    \sum_{\substack{
        \state_m' \text{ for } t' \\
        \forget(\state_m') = \state_m \\
        v \not \in O_{m,1}', v \not \in O_{m,2}'
    }}
    dp_{t',s}\left[\begin{matrix}
    \state_1' \\ \vdots \\ \state_m'
    \end{matrix}\right] +
    \sum_{\substack{
        \state_1' \text{ for } t' \\
        \forget(\state_1') = \state_1 \\
        v \in O_{1,1}', v \in O_{1,2}'
    }}
    \cdots
    \sum_{\substack{
        \state_m' \text{ for } t' \\
        \forget(\state_m') = \state_m \\
        v \in O_{m,1}', v \in O_{m,2}'
    }}
    dp_{t',s}\left[\begin{matrix}
    \state_1' \\ \vdots \\ \state_m'
    \end{matrix}\right].
\end{align*}
Here, $\state_i'$ for $i \in [m]$ denotes a tuple
$(O'_{i,1}, O'_{i,2}, F'_{i,1}, F'_{i,2}, \tau'_i, \nu'_i)$.
Running through all possible combinations of $\state_1', \ldots, \state_m'$,
we can compute each entry of $ dp_{t,s} $ in $ w^{\bigO(wm)} n^{\bigO(1)} $ time.
\end{proof}

\begin{lemma}
\label{lem:join-m}
Let $t$ be a join node with two children $t'$ and $t''$ such that $X_t = X_{t'} = X_{t''}$, and $s \in [0\isep n]$.
Given $dp_{t',s'}$ and $dp_{t'',s''}$ for all $s'$ and $s''$, respectively,
we can compute each entry of $dp_{t,s}$ in $w^{\bigO(wm)} n^{\bigO(1)}$ time.
\end{lemma}
\begin{proof}
Using a mapping $\join$ introduced in the proof of \cref{lem:join},
for any
$S' \subseteq V_{t'} \setminus X_{t'}$,
$S'' \subseteq V_{t''} \setminus X_{t''}$
$\state = (O_1,O_2,F_1,F_2,\tau,\nu)$ for $t$, and $ i \in [m]$,
we have that 
\begin{align*}
\Upsilon_{t,i}(S' \uplus S'', \state) = \frac{1}{\mat{A}^i(\tau)}
\sum_{\substack{
\state' \text{ for } t', \state'' \text{ for } t'' \\
\join(|S'|, \state', |S''|, \state'') = \state
}}
\Upsilon_{t',A}(S', \state') \cdot \Upsilon_{t',A}(S'', \state'').
\end{align*}
Consequently, $ dp_{t,s} $ can be decomposed as follows.

\begin{align*}
dp_{t,s}\left[\begin{matrix}
\state_1 \\ \vdots \\ \state_m
\end{matrix}\right] =
    \frac{1}{\mat{A}^1(\tau) \cdots \mat{A}^m(\tau)}
    \sum_{\substack{s',s'' \in [0\isep s] \\ s'+s'' = s}}
    \sum_{\substack{
        \state_1', \state_1'' \text{ for } t' \\
        \join(s',\state_1',,s''\state_1'') = \state_1
    }}
    \cdots
    \sum_{\substack{
        \state_m', \state_m'' \text{ for } t' \\
        \join(s',\state_m',s'',\state_m'') = \state_m
    }}
    dp_{t',s'}\left[\begin{matrix}
    \state_1' \\ \vdots \\ \state_m'
    \end{matrix}\right] \cdot
    dp_{t'',s''}\left[\begin{matrix}
    \state_1'' \\ \vdots \\ \state_m''
    \end{matrix}\right].
\end{align*}
Running through all possible combinations of 
$s',s'',\state_1',\state_1'', \ldots, \state_m',\state_m''$,
we can compute each entry of $dp_{t,s}$ in $w^{\bigO(wm)} n^{\bigO(1)}$ time.
\end{proof}

We are now ready to describe an FPT algorithm for computing $\ZZ_m$.
Our algorithm is almost identical that for \cref{thm:fpt-treewidth-2}.
Given a tree decomposition for $ ([n], \bigcup_{i \in [m]} \nnz(\mat{A}^i)) $
of width at most $w$,
we transform it to a nice tree decomposition
$(T, \{X_t\}_{t \in T})$ rooted at $r$
of width at most $w$ that
has $\bigO(wn)$ nodes in polynomial time \citep{cygan2015parameterized}.
For every leaf $\ell$ of $T$, we initialize $dp_{\ell,s}$ as
\begin{align*}
    dp_{\ell,s}\left[\begin{matrix}
    \emptyset,\emptyset,\emptyset,\emptyset,\emptyset \bij \emptyset,\nu_1 \\ \vdots \\ \emptyset,\emptyset,\emptyset,\emptyset,\emptyset \bij \emptyset,\nu_m
    \end{matrix}\right] = 
    \begin{cases}
    1 & \text{if } s = 0 \text{ and } \nu_1 = \cdots = \nu_m = 0, \\
    0 & \text{otherwise}.
    \end{cases}
\end{align*}
Then, for each non-leaf node $t \in T$,
we apply either of \cref{lem:introduce-m,lem:forget-m,lem:join-m} to fill
$dp_{t,s}$ using already-filled $dt_{t',s'}$ for
all children $t'$ of $t$
in a bottom-up fashion.
Completing dynamic programming,
we compute $\ZZ_m$ as follows:
\begin{align*}
    \ZZ_m(\mat{A}^1, \ldots, \mat{A}^m) =
    \sum_{s, \nu_1, \ldots, \nu_m} (-1)^{\nu_1 + \cdots + \nu_m} \cdot
    dp_{r,s}\left[\begin{matrix}
    \emptyset,\emptyset,\emptyset,\emptyset,\emptyset \bij \emptyset,\nu_1 \\
    \vdots \\
    \emptyset,\emptyset,\emptyset,\emptyset,\emptyset \bij \emptyset,\nu_m
    \end{matrix}\right].
\end{align*}
The correctness follows from \cref{lem:introduce-m,lem:forget-m,lem:join-m}.
Because
$T$ has at most $\bigO(wn)$ nodes, each table is of size $w^{\bigO(wm)}$, and
each table entry can be computed in $ w^{\bigO(wm)} n^{\bigO(1)} $ time by
\cref{lem:introduce-m,lem:forget-m,lem:join-m},
the whole time complexity is bounded by $ w^{\bigO(wm)} n^{\bigO(1)} $.
\qed

\subsection{Parameterization by Maximum Treewidth}
\label{subsec:fpt:intract}

Finally, let us take the \emph{maximum treewidth} of two matrices as a parameter and
refute its fixed-parameter tractability.
This parameterization is preferable to the previous one
because the maximum treewidth can be far smaller than the treewidth of the union.
One can, for example, construct two $n \times n$ matrices $\mat{A}$ and $\mat{B}$ such that
$ \max\{\tw(\mat{A}), \tw(\mat{B})\} = 1 $ and
$ \tw(\nnz(\mat{A}) \cup \nnz(\mat{B})) = \bigO(\sqrt{n}) $
because we can ``weave'' two paths into a grid.
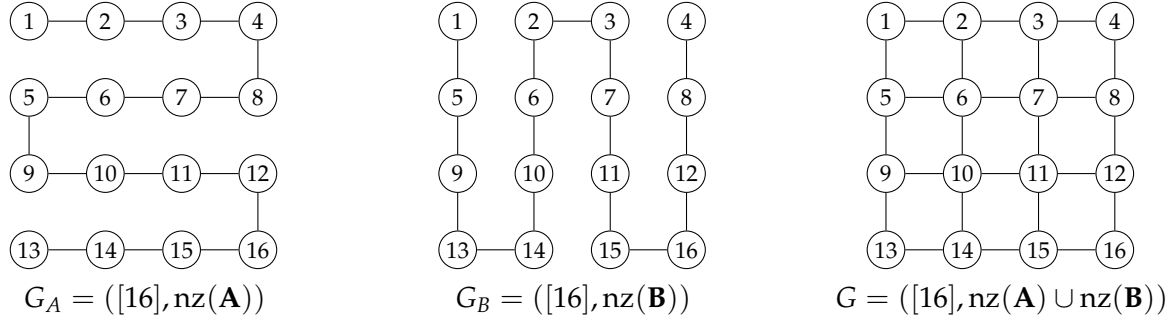
\begin{figure}[tbp]
\begin{minipage}{0.31\hsize}
\centering
\scalebox{0.5}{
\begin{tikzpicture}
[circlenode/.style={draw, circle, minimum height=1cm, font=\LARGE, inner sep=0}]
\node[circlenode, thick](v1){$1$};
\node[circlenode, thick, right=1cm of v1](v2){$2$};
\node[circlenode, thick, right=1cm of v2](v3){$3$};
\node[circlenode, thick, right=1cm of v3](v4){$4$};

\node[circlenode, thick, below=1cm of v1](v5){$5$};
\node[circlenode, thick, right=1cm of v5](v6){$6$};
\node[circlenode, thick, right=1cm of v6](v7){$7$};
\node[circlenode, thick, right=1cm of v7](v8){$8$};

\node[circlenode, thick, below=1cm of v5](v9){$9$};
\node[circlenode, thick, right=1cm of v9](v10){$10$};
\node[circlenode, thick, right=1cm of v10](v11){$11$};
\node[circlenode, thick, right=1cm of v11](v12){$12$};

\node[circlenode, thick, below=1cm of v9](v13){$13$};
\node[circlenode, thick, right=1cm of v13](v14){$14$};
\node[circlenode, thick, right=1cm of v14](v15){$15$};
\node[circlenode, thick, right=1cm of v15](v16){$16$};
\foreach \u / \v in {v1/v2, v2/v3, v3/v4, v4/v8, v5/v6, v6/v7, v7/v8, v5/v9, v9/v10, v10/v11, v11/v12, v12/v16, v13/v14, v14/v15, v15/v16}
    \draw[thick] (\u)--(\v);
\end{tikzpicture}
}
$G_A=([16], \nnz(\mat{A}))$
\end{minipage}
\hfill
\begin{minipage}{0.31\hsize}
\centering
\scalebox{0.5}{
\begin{tikzpicture}
[circlenode/.style={draw, circle, minimum height=1cm, font=\LARGE, inner sep=0}]
\node[circlenode, thick](v1){$1$};
\node[circlenode, thick, right=1cm of v1](v2){$2$};
\node[circlenode, thick, right=1cm of v2](v3){$3$};
\node[circlenode, thick, right=1cm of v3](v4){$4$};

\node[circlenode, thick, below=1cm of v1](v5){$5$};
\node[circlenode, thick, right=1cm of v5](v6){$6$};
\node[circlenode, thick, right=1cm of v6](v7){$7$};
\node[circlenode, thick, right=1cm of v7](v8){$8$};

\node[circlenode, thick, below=1cm of v5](v9){$9$};
\node[circlenode, thick, right=1cm of v9](v10){$10$};
\node[circlenode, thick, right=1cm of v10](v11){$11$};
\node[circlenode, thick, right=1cm of v11](v12){$12$};

\node[circlenode, thick, below=1cm of v9](v13){$13$};
\node[circlenode, thick, right=1cm of v13](v14){$14$};
\node[circlenode, thick, right=1cm of v14](v15){$15$};
\node[circlenode, thick, right=1cm of v15](v16){$16$};
\foreach \u / \v in {v1/v5, v5/v9, v9/v13, v13/v14, v14/v10, v10/v6, v6/v2, v2/v3, v3/v7, v7/v11, v11/v15, v15/v16, v16/v12, v12/v8, v8/v4}
    \draw[thick] (\u)--(\v);
\end{tikzpicture}
}
$G_B=([16], \nnz(\mat{B}))$
\end{minipage}
\hfill
\begin{minipage}{0.31\hsize}
\centering
\scalebox{0.5}{
\begin{tikzpicture}
[circlenode/.style={draw, circle, minimum height=1cm, font=\LARGE, inner sep=0}]
\node[circlenode, thick](v1){$1$};
\node[circlenode, thick, right=1cm of v1](v2){$2$};
\node[circlenode, thick, right=1cm of v2](v3){$3$};
\node[circlenode, thick, right=1cm of v3](v4){$4$};

\node[circlenode, thick, below=1cm of v1](v5){$5$};
\node[circlenode, thick, right=1cm of v5](v6){$6$};
\node[circlenode, thick, right=1cm of v6](v7){$7$};
\node[circlenode, thick, right=1cm of v7](v8){$8$};

\node[circlenode, thick, below=1cm of v5](v9){$9$};
\node[circlenode, thick, right=1cm of v9](v10){$10$};
\node[circlenode, thick, right=1cm of v10](v11){$11$};
\node[circlenode, thick, right=1cm of v11](v12){$12$};

\node[circlenode, thick, below=1cm of v9](v13){$13$};
\node[circlenode, thick, right=1cm of v13](v14){$14$};
\node[circlenode, thick, right=1cm of v14](v15){$15$};
\node[circlenode, thick, right=1cm of v15](v16){$16$};
\foreach \u / \v in {v1/v2, v2/v3, v3/v4, v5/v6, v6/v7, v7/v8, v9/v10, v10/v11, v11/v12, v13/v14, v14/v15, v15/v16, v1/v5, v5/v9, v9/v13, v2/v6, v6/v10, v10/v14, v3/v7, v7/v11, v11/v15, v4/v8, v8/v12, v12/v16}
    \draw[thick] (\u)--(\v);
\end{tikzpicture}
}
$G = ([16], \nnz(\mat{A}) \cup \nnz(\mat{B}))$
\end{minipage}
\caption{Three graphs.}
\label{fig:weave}
\end{figure} 

\begin{example}
Consider any two $16 \times 16$ matrices $\mat{A}$ and $\mat{B}$
whose nonzero entries $\nnz(\mat{A})$ and $\nnz(\mat{B})$ are drawn in \cref{fig:weave}.
Because each of $\nnz(\mat{A})$ and $\nnz(\mat{B})$ forms a path graph on $16$ vertices, we have that
$ \max\{\tw(\mat{A}), \tw(\mat{B})\} = 1 $.
On the other hand, the graph $([16], \nnz(\mat{A}) \cup \nnz(\mat{B}))$ forms a $4 \times 4$ grid graph, and we thus have that
$ \tw(\nnz(\mat{A}) \cup \nnz(\mat{B})) = 4 $.
\end{example}
Unluckily, even when both $\mat{A}$ and $\mat{B}$ have treewidth at most $3$,
it is still \shP-hard to compute $\ZZ_2(\mat{A},\mat{B})$.
Thus, parameterization by $\max\{ \tw(\mat{A}), \tw(\mat{B}) \}$ is not even XP unless \FP~$=$~\shP.

\begin{theorem}
\label{thm:treewidth-hard}
Let $\mat{A}, \mat{B}$ be two positive semi-definite $(0,1)$-matrices
such that $\tw(\mat{A}) \leq 3$ and $\tw(\mat{B}) \leq 3$.
Then, computing $\ZZ_2(\mat{A}, \mat{B})$ is \shP-hard.
\end{theorem}
\begin{proof}
We reduce from the problem of counting all (not necessarily perfect) matchings
in a bipartite graph of maximum degree $4$,
which is known to be \shP-complete \citep{vadhan2001complexity}.
Let $ H = (X \uplus Y, E) $
be a bipartite graph of maximum degree $4$,
where $ E \subseteq X \times Y $ is the set of edges between $X$ and $Y$.
In accordance with \citet{gillenwater2014approximate}'s reduction,
we construct two positive semi-definite matrices
$ \mat{A}, \mat{B} \in \{0,1\}^{E \times E} $
so that
$A_{i,j}$ is $1$ if edges $i, j$ in $E$ share a common vertex in $X$ and $0$ otherwise, and
$B_{i,j}$ is $1$ if edges $i, j$ in $E$ share a common vertex in $Y$ and $0$ otherwise.
Then,
we have that for any edge set $S \subseteq E$,
the value of $ \det(\mat{A}_{S,S}) \det(\mat{B}_{S,S}) $
is $1$ if $S$ is a matching in $H$ and $0$ otherwise, that is,
$ \ZZ_2(\mat{A}, \mat{B}) $ is exactly equal to
the total number of matchings in $H$,
and thus \shP-hardness follows.

By construction, $\mat{A}$ and $\mat{B}$ must be block-diagonal matrices,
where each block is a $k \times k$ all-one matrix for some $k \in [4]$.
Because the graph $([n], \nnz(\mat{A}))$ is the disjoint union of cliques of size at most $4$, whose treewidth is at most $3$,
it holds that $ \tw(\mat{A}) \leq 3 $.
Similarly, we have that $ \tw(\mat{B}) \leq 3 $, completing the proof.
\end{proof}

\section{Extensions to Fixed-Size $\Pi$-DPPs}
\label{sec:kdpp}

In this section,
we impose \emph{size constraints} on $\Pi$-DPPs so as to produce a fixed-size set and 
analyze the normalizing constant of the resulting distribution.
Formally, given $m$ P$_0$-matrices $\mat{A}^1, \ldots, \mat{A}^m \in \bbR^{n \times n}$ and a size parameter $k \in [0\isep n]$,
the \emph{$k\Pi$-DPP} is defined as a distribution whose probability mass for each subset $S \subseteq [n]$ as
proportional to $\det(\mat{A}^1_{S,S}) \cdots \det(\mat{A}^m_{S,S}) \cdot [\![S \in {[n] \choose k}]\!]$.
We will use $\ZZ_{m,k}(\mat{A}^1, \ldots, \mat{A}^m)$ to denote 
its \emph{normalizing constant}; namely,
\begin{align*}
    \ZZ_{m,k}(\mat{A}^1, \ldots, \mat{A}^m) \triangleq
    \sum_{S \in {[n] \choose k}} \prod_{i \in [m]} \det(\mat{A}^i_{S,S}).
\end{align*}
The special case of $m=1$ coincides with $k$-DPPs \citep{kulesza2011kdpps},
the normalizing constant of which is known to be amenable to compute.
We will investigate the computational complexity of estimating $\ZZ_{m,k}$ by taking a similar approach to the one in
\cref{sec:intract,sec:inapprox,sec:fpt}.
Intuitively, $\ZZ_{m,k}$ seems harder to estimate than $\ZZ_m$ because 
we have that
$ \ZZ_m = \sum_{0 \leq k \leq n} \ZZ_{m,k} $.

\subsection{Intractability, Inapproximability, and Indistinguishability}
First, we provide hardness results that can be thought of as size-constraint counterparts to 
those in \cref{sec:intract,sec:inapprox}:
\begin{theorem}
\label{thm:kdpp-intract-inapprox}
The following results hold for $\ZZ_{m,k}$:
\begin{itemize}
    \item For every fixed positive even integer $p$,
    computing $ \ZZ_{p,k}(\mat{A}, \ldots, \mat{A}) \bmod 3 $ for either a $(-1,0,1)$-matrix or a P-matrix $\mat{A}$ is \UP-hard and \ModkP{3}-hard.
    \item It is \NP-hard to determine whether
    $\ZZ_{3,k}(\mat{A}, \mat{B}, \mat{C})$ is positive or $0$
    for three positive semi-definite matrices $ \mat{A}, \mat{B}, \mat{C}$ in $ \bbQ^{n \times n} $.
    Therefore, for any polynomial-time computable function $\rho$,
    $\ZZ_{3,k}$ cannot be approximated within a factor of $\rho(|\calI|)$, where $|\calI|$ is the input size, unless \cP~$=$~\NP.
    \item The mixed discriminant $D$ is AP-reducible to $\ZZ_{2,k}$.
    Therefore, if there exists an FPRAS for $\ZZ_{2,k}$, then there exists an FPRAS for $D$.
\end{itemize}
\end{theorem}

Some remarks are in order.
The first statement is a direct consequence of \cref{cor:pow-hard}.
The second statement corresponds to \cref{thm:inapprox-3};
however, there is a crucial difference between 
$\ZZ_{3,k}$ and $\ZZ_3$ wherein
the former cannot be approximated within \emph{any} (polynomial-time computable) factor, while
the latter is approximable within a factor of $2^{\bigO(|\calI|^2)}$ (\cref{obs:approx-3}).
The third statement corresponds to \cref{thm:AP-D-ZZ2},
whose proof is much simpler than that of \cref{thm:AP-D-ZZ2}.

\begin{proof}[Proof of \cref{thm:kdpp-intract-inapprox}]
The first statement is a direct consequence of 
\cref{cor:pow-hard} and the following equality
for any $n \times n$ matrix $\mat{A}$ and any positive integer $p$:
\begin{align}
\ZZ_p(\underbrace{\mat{A}, \ldots, \mat{A}}_{p \text{ times}}) = \sum_{0 \leq k \leq n} \ZZ_{p,k}(\underbrace{\mat{A}, \ldots, \mat{A}}_{p \text{ times}}).
\end{align}
To prove the second statement,
we show a polynomial-time Turing reduction from \textsc{HamiltonianPath} \citep{garey1979computers}.
Let $G = (V,E)$ be a directed graph on $n$ vertices and $m$ edges.
We construct $m \times m$ three positive semi-definite matrices
$\mat{A}, \mat{B}, \mat{C} \in \bbQ^{E \times E}$ in polynomial time according 
to the procedure described in the proof of \cref{thm:inapprox-3}.
In particular,
the value of $\det(\mat{A}_{S,S}) \det(\mat{B}_{S,S}) \det(\mat{C}_{S,S})$
for $S \in {E \choose n-1}$ is positive if and only if 
$S$ is a Hamiltonian path.
We can thus use
$\ZZ_{3,n-1}(\mat{A},\mat{B},\mat{C})$
to decide the Hamiltonicity of $G$ as follows:
\begin{itemize}
    \item \textbf{Case (1)} if there exists (at least) one Hamiltonian path in $G$, then $\ZZ_{3,n-1}(\mat{A},\mat{B},\mat{C}) > 0$;
    \item \textbf{Case (2)} if no Hamiltonian path exists in $G$,
    then $\ZZ_{3,n-1}(\mat{A},\mat{B},\mat{C}) = 0$.
\end{itemize}

Finally, we show a polynomial-time Turing reduction from $D$ to $\ZZ_{2,k}$, which is sufficient to prove AP-reducibility.
Let $\mat{K}^1, \ldots, \mat{K}^m$ be $m$ positive semi-definite matrices in $\bbQ^{m \times m}$, and define $n = m^2$.
We construct the two matrices $\mat{A},\mat{B} \in \bbQ^{n \times n}$ by following the procedure in the proof of \cref{thm:AP-D-ZZ2}.
By \cref{eq:ZZ2-D},
we have the following relation:
\begin{align}
    \ZZ_{2,m}(\mat{A},\mat{B})
    = \sum_{S \in {[n] \choose m}}\det(\mat{A}_{S,S})\det(\mat{B}_{S,S})
    = m!\;D(\mat{K}^1, \ldots, \mat{K}^m),
\end{align}
which completes the reduction.
\end{proof}

\subsection{Fixed-Parameter Tractability}
Here, we demonstrate the fixed-parameter tractability of computing $\ZZ_{m,k}$.
First, we show that even if we are only given access to an oracle for $\ZZ_m$,
we can \emph{recover}
the values of $\ZZ_{m,k}$ by calling the oracle polynomially many times.
\begin{theorem}
\label{thm:ZZ-ZZk}
Let $\mat{A}^1, \ldots, \mat{A}^m$ be $m$ matrices in $\bbQ^{n \times n}$ and $k$ be an integer in $[n]$.
Suppose we are given access to an oracle that returns $\ZZ_m$.
Then, for all $k$, $ \ZZ_{m,k}(\mat{A}^1, \ldots, \mat{A}^m)$
can be computed in polynomial time, by calling
the oracle $L \triangleq n+1$ times.
Furthermore, if $\{(\mat{A}^{1,\ell}, \ldots, \mat{A}^{m,\ell})\}_{\ell \in [L]}$ denotes the set consisting of the tuples of $m$ matrices for which the oracle is called,
we have that $\rank(\mat{A}^{i,\ell}) \leq \rank(\mat{A}^i)$ and 
$\nnz(\mat{A}^{i,\ell}) \subseteq \nnz(\mat{A}^i)$
for all $i \in [m]$ and $\ell \in [L]$.
\end{theorem}
\begin{proof}
Introduce a positive number $x \in \bbQ$ and
consider a polynomial $\ZZ(x) \triangleq \ZZ_m(x \mat{A}^1, \mat{A}^2, \ldots, \mat{A}^m)$ in $x$.
We can expand it as follows:
\begin{align*}
    \ZZ_m(x \mat{A}^1, \mat{A}^2, \ldots, \mat{A}^m)
    & = \sum_{S \subseteq [n]} x^{|S|} \prod_{i \in [m]} \det(\mat{A}^i)
    = \sum_{k \in [n]} x^k \cdot \ZZ_{m,k}(\mat{A}^1, \ldots, \mat{A}^m).
\end{align*}
Given the values of $\ZZ(x)$ for all $x \in [n+1]$,
each of which can be computed by calling the oracle for $\ZZ_m$,
we can recover all the coefficients of $\ZZ(x)$ by Lagrange interpolation as desired.
The structural arguments are obvious.
\end{proof}

As a corollary of \cref{thm:fpt-rank-m,thm:fpt-treewidth-m,thm:ZZ-ZZk},
we obtain respective FPT algorithms 
for computing $\ZZ_{m,k}$ parameterized by maximum rank and treewidth.
\begin{corollary}
\label{cor:kdpp-fpt}
For $m$ matrices $\mat{A}^1, \ldots, \mat{A}^m$ in $\bbQ^{n \times n}$,
the following fixed-parameter tractability results hold:
\begin{itemize}
    \item there exists an $ r^{\bigO(mr)} n^{\bigO(1)} $-time algorithm
    for computing $\ZZ_{m,k}(\mat{A}^1, \ldots, \mat{A}^m)$ for all $k \in [n]$, where
    $r = \max_{i \in [m]} \rank(\mat{A}^i)$ denotes the maximum rank among the $m$ matrices;
    \item there exists a $ w^{\bigO(mw)} n^{\bigO(1)} $-time algorithm
    for computing $\ZZ_{m,k}(\mat{A}^1, \ldots, \mat{A}^m)$ for all $k \in [n]$, where
    $w = \tw(\bigcup_{i \in [m]} \nnz(\mat{A}^i))$ denotes
    the treewidth of the union of nonzero entries in the $m$ matrices.
\end{itemize}
\end{corollary}

\subsection{Parameterization by Output Size}
Finally, we investigate the fixed-parameter tractability of the computation of $\ZZ_{m,k}$ parameterized by $k$.
Since a brute-force algorithm that examines
all possible $ {n \choose k} = \bigO(n^k) $ subsets has time complexity $ m n^{k + \bigO(1)} $,
computing $\ZZ_{m,k}$ parameterized by $k$ is XP.
On the other hand, we show that computing $\ZZ_{m,k}$ is \shW{1}-hard.
Consequently, an FPT algorithm parameterized by $k$ for computing $\ZZ_{m,k}$ does not exist unless \FPT~$=$~\shW{1},
which is suspected to be false.

\begin{theorem}
\label{thm:kdpp-size-hard}
Let $\mat{A}, \mat{B}$ be two $n \times n$ positive semi-definite $(0,1)$-matrices
and $k$ be a positive integer in $[n]$.
Then,
it is \shW{1}-hard to compute $\ZZ_{2,k}(\mat{A},\mat{B})$ parameterized by $k$.
\end{theorem}

\begin{proof}
We show a polynomial-time parsimonious reduction from the problem of
counting all (imperfect) matchings of size $k$ in a bipartite graph $H$,
which was proven to be \shW{1}-hard by \citet[Theorem I.2]{curticapean2014complexity}.
In according with \citet{gillenwater2014approximate} (cf.~proof of \cref{thm:treewidth-hard}),
we construct two matrices $ \mat{A}$ and $\mat{B} $ from $H$ in polynomial time in $n$, so that
$\ZZ_{2,k}(\mat{A},\mat{B})$
is equal to the total number of $k$-matchings in $G$.
Because this reduction from an input $G$ with a parameter $k$ to an input $(\mat{A}, \mat{B})$ with a parameter $k$
is an FPT parsimonious reduction
\citep[see][]{flum2004parameterized},
the proof follows.
\end{proof}

\section{Application of FPT Algorithms to Two Related Problems}
\label{sec:app}
In this section,
we introduce two applications of the FPT algorithm for computing $\ZZ_m$ parameterized by the treewidth (\cref{thm:fpt-treewidth-m}).

\subsection{Approximation Algorithm for E-DPPs of Fractional Exponents}
\label{subsec:app:edpp}

We first use \cref{thm:fpt-treewidth-m} to estimate the normalizing constant for E-DPPs of \emph{fractional} exponent $p$.
Since $\Pi$-DPPs include E-DPPs of exponent $p$ only if $p$ is an \emph{integer},
fixed-parameter tractability does not apply to the fractional case.
The resulting application is a $w^{\bigO(wp)}$-time parameterized algorithm that estimates
$\ZZ^p(\mat{A})$ within a factor of $2^{\frac{n}{2p-1}}$.

\begin{theorem}
\label{thm:edpp-frac}
Let $\mat{A}$ be a P$_0$-matrix in $\bbQ^{n \times n}$ and
$p$ be a positive fractional number greater than $1$.
Then, there exists a $w^{\bigO(wp)}n^{\bigO(1)}$-time algorithm that
returns a $2^{\frac{n}{2p-1}}$-factor approximation to $\ZZ^p(\mat{A})$.
\end{theorem}

We introduce the general equivalence of $\ell_p$ norms,
which can be derived from \Holder's inequality.

\begin{lemma}[General equivalence of $\ell_p$-norms, see, e.g., \citealp{steele2004cauchy}]
For an $n$-dimensional vector $\vec{x}$ in $\bbR^{n}$ and 
two positive real numbers $p,q$ such that $0<p<q$,
the following inequality holds:
\begin{align}
\label{eq:lp-ineq}
    \|\vec{x}\|_q \leq \|\vec{x}\|_p \leq n^{\frac{1}{p}-\frac{1}{q}} \cdot \|\vec{x}\|_q,
\end{align}
where $\| \cdot \|_p$ denotes the $\ell_p$ norm; i.e.,
\begin{align*}
    \|\vec{x}\|_p \triangleq \Biggl(\sum_{i \in [n]} |x_i|^p\Biggr)^{\frac{1}{p}}.
\end{align*}
\end{lemma}

\begin{remark}
\label{rmk:edpp-frac}
If we apply \cref{eq:lp-ineq} to a $2^n$-dimensional vector $\vec{x} \in \bbQ^{2^{[n]}}$ such that $x_S \triangleq \det(\mat{A}_{S,S}) $ for each $S \subseteq [n]$
for a P$_0$-matrix $\mat{A} \in \bbQ^{n \times n}$,
we have that $ \ZZ^p(\mat{A}) \leq \det(\mat{A}+\mat{I})^{p} \leq 2^{n(p-1)} \ZZ^p(\mat{A}) $ for any $p>1$,
which gives a $2^{n(p-1)}$-approximation.
On the other hand,
it is \NP-hard to approximate $\ZZ^p(\mat{A})$ within a factor of $2^{\beta pn}$ for some $\beta > 0$ \citep{ohsaka2021unconstrained}.
Therefore, $2^{\bigO(pn)}$ is a tight approximation factor for the normalizing constant for E-DPPs in the general case.
\end{remark}

\begin{proof}[Proof of \cref{thm:edpp-frac}]
Fix a positive semi-definite matrix $\mat{A} \in \bbQ^{n \times n}$ and
a fractional number $p > 1$.
Since $\floor{p} < p < \ceil{p}$,
we can write $p = \lambda \floor{p} + (1-\lambda) \ceil{p}$
for some $\lambda \in (0,1)$.
First, we derive two estimates of $\ZZ^p(\mat{A})$ using
$\ZZ^{\floor{p}}(\mat{A})$ and $\ZZ^{\ceil{p}}(\mat{A})$.

\paragraph{Estimate Using $\ZZ^{\floor{p}}$.}
We will show that $(\ZZ^{\floor{p}}(\mat{A}))^{\frac{p}{\floor{p}}}$
is a $2^{n(\frac{p}{\floor{p}}-1)}$-approximation to $\ZZ^p(\mat{A})$.
By using \cref{eq:lp-ineq},
we bound $(\ZZ^{\floor{p}}(\mat{A}))^{\frac{1}{\floor{p}}}$ as follows:
\begin{align*}
\Biggl(\sum_{S \subseteq [n]} \det(\mat{A}_{S,S})^{p} \Biggr)^{\frac{1}{p}}
\leq
\Biggl(\sum_{S \subseteq [n]} \det(\mat{A}_{S,S})^{\floor{p}} \Biggr)^{\frac{1}{\floor{p}}}
\leq 
(2^n)^{\frac{1}{\floor{p}} - \frac{1}{p}} \cdot
\Biggl(\sum_{S \subseteq [n]} \det(\mat{A}_{S,S})^{p} \Biggr)^{\frac{1}{p}},
\end{align*}
which immediately implies that
\begin{align}
\label{eq:ZZp-floor}
    \ZZ^p(\mat{A}) \leq (\ZZ^{\floor{p}}(\mat{A}))^{\frac{p}{\floor{p}}}
    \leq 2^{n(\frac{p}{\floor{p}}-1)} \cdot \ZZ^p(\mat{A}).
\end{align}

\paragraph{Estimate Using $\ZZ^{\ceil{p}}$.}
We will show that $(\ZZ^{\ceil{p}}(\mat{A}))^{\frac{p}{\ceil{p}}}$ is
a $2^{n(1-\frac{p}{\ceil{p}})}$-approximation to
$\ZZ^p(\mat{A})$.
By using \cref{eq:lp-ineq},
we bound $(\ZZ^{\ceil{p}}(\mat{A}))^{\frac{1}{\ceil{p}}}$ as follows:
\begin{align*}
(2^n)^{\frac{1}{\ceil{p}} - \frac{1}{p}} \cdot
\Biggl(\sum_{S \subseteq [n]} \det(\mat{A}_{S,S})^{p} \Biggr)^{\frac{1}{p}}
\leq
\Biggl(\sum_{S \subseteq [n]} \det(\mat{A}_{S,S})^{\ceil{p}} \Biggr)^{\frac{1}{\ceil{p}}}
\leq
\Biggl(\sum_{S \subseteq [n]} \det(\mat{A}_{S,S})^{p} \Biggr)^{\frac{1}{p}},
\end{align*}
which immediately implies that
\begin{align}
\label{eq:ZZp-ceil}
    \ZZ^p(\mat{A})
    \leq 2^{n(1-\frac{p}{\ceil{p}})} \cdot (\ZZ^{\ceil{p}}(\mat{A}))^{\frac{p}{\ceil{p}}}
    \leq 2^{n(1-\frac{p}{\ceil{p}})} \cdot \ZZ^p(\mat{A}).
\end{align}

Next, we choose the ``better'' of the two estimates, depending on the value of $\lambda$.
Since it holds that
$p = \lambda \floor{p} + (1-\lambda) \ceil{p} = \floor{p} + 1 - \lambda$,
we can simplify the exponents of the approximation factors in \cref{eq:ZZp-floor,eq:ZZp-ceil} as follows:
\begin{align*}
    \frac{p}{\floor{p}} - 1 = \frac{1-\lambda}{\floor{p}} \text{ and }
    1-\frac{p}{\ceil{p}} = \frac{\lambda}{\floor{p} + 1}.
\end{align*}
Observing that
the linear equation $\frac{1-\lambda}{\floor{p}} = \frac{\lambda}{\floor{p}+1}$
has a unique solution $\lambda^* \triangleq \frac{\floor{p}+1}{2\floor{p}+1}$,
we have the following relation:
\begin{align*}
    \min\left\{ \frac{1-\lambda}{\floor{p}}, \frac{\lambda}{\floor{p}+1} \right\}
    =
    \begin{cases}
    \dfrac{1-\lambda}{\floor{p}} & \text{if } \lambda > \lambda^* = \dfrac{\floor{p} + 1}{2\floor{p}+1}, \\[1em]
    \dfrac{\lambda}{\floor{p}+1} & \text{otherwise}.
    \end{cases}
\end{align*}
Further, $\min\left\{ \frac{1-\lambda}{\floor{p}}, \frac{\lambda}{\floor{p}+1} \right\}$ takes $\frac{1}{2\floor{p}+1}$ as the maximum value
when $\lambda = \lambda^* \in (0,1)$.

Our algorithm works as follows.
First, we compute $\ZZ^{\floor{p}}(\mat{A})$ and $\ZZ^{\ceil{p}}(\mat{A})$ in $w^{\bigO(wp)}n^{\bigO(1)}$ time by using the FPT algorithm in \cref{thm:fpt-treewidth-m}.
Then, if $\lambda > \lambda^*$, we output an $\alpha$-approximation to
$(\ZZ^{\floor{p}}(\mat{A}))^{\frac{p}{\floor{p}}}$;
otherwise we output an $\alpha$-approximation to $2^{n(1-\frac{p}{\ceil{p}})} \cdot (\ZZ^{\ceil{p}}(\mat{A}))^{\frac{p}{\ceil{p}}}$,
where $\alpha \triangleq 2^{(\frac{1}{2p-1} - \frac{1}{2\floor{p}+1})n} > 1$.%
\footnote{
We can use, for example, the Newton--Raphson method to compute an $\alpha$-approximation to fractional exponents efficiently.
}
The output ensures an approximation factor of
$\alpha \cdot 2^{\frac{n}{2\floor{p}+1}} \leq 2^{\frac{n}{2p-1}}$,
which completes the proof.
\end{proof}

\subsection{Subexponential Algorithm for Unconstrained MAP Inference}
\label{subsec:app:map}
We apply the FPT algorithm to maximum a posteriori (MAP) inference.
For a positive semi-definite matrix $\mat{A} \in \bbQ^{n \times n}$,
\emph{unconstrained MAP inference}
requests that we find a subset $S \subseteq [n]$ having the maximum determinant,
i.e., $\argmax_{S \subseteq [n]} \det(\mat{A}_{S,S})$.
We say that a polynomial-time algorithm $\alg$ is
a \emph{$\rho$-approximation algorithm} for $\rho \geq 1$ if for all positive semi-definite matrix $\mat{A} \in \bbQ^{n \times n}$,
\begin{align*}
    \det(\alg(\mat{A})) \geq \left(\frac{1}{\rho}\right) \cdot \max_{S \subseteq [n]} \det(\mat{A}_{S,S}),
\end{align*}
where $\alg(\mat{A})$ is the output of $\alg$ on $\mat{A}$.
The approximation factor $\rho$ can be a function in the input size $|\calI|$, and
(asymptotically) smaller $\rho$ is a better approximation factor.
The following theorem states that
we can find a $2^{\sqrt{n}}$-approximation to unconstrained MAP inference in $2^{\bigO(\sqrt{n})}$ time (with high probability)
provided that the treewidth of $\mat{A}$ is a constant.

\begin{theorem}
\label{thm:map-approx}
Let $\mat{A}$ be a positive semi-definite matrix of treewidth $w$ in $\bbQ^{n \times n}$.
Then, there exists a $w^{\bigO(w\sqrt{n})} n^{\bigO(1)}$-time randomized algorithm that
outputs a $2^{\sqrt{n}}$-approximation to unconstrained MAP inference on $\mat{A}$ with probability at least $1-2^{-n}$.
In particular, if the treewidth $w$ is $\bigO(1)$,
then the time complexity is bounded by
$2^{\bigO(\sqrt{n})}$.
\end{theorem}
\begin{proof}
Fix a positive semi-definite matrix $\mat{A} \in \bbQ^{n \times n}$ of treewidth $w$ and
let $\OPT \triangleq \max_{S \subseteq [n]} \det(\mat{A}_{S,S})$ be the optimum value of unconstrained MAP inference on $\mat{A}$.
Define $p \triangleq \lceil 2 \sqrt{n} \rceil$.
Consider an E-DPP of exponent $p$ defined by $\mat{A}$, whose
probability mass for each subset $S \subseteq [n]$
is $ \frac{\det(\mat{A}_{S,S})^p}{\ZZ^p(\mat{A})} $.
Since this E-DPP coincides with the $\Pi$-DPP defined by $p$ copies of $\mat{A}$,
by \cref{thm:sampling,thm:fpt-treewidth-m},
we can draw a random sample $S$ from it in $ w^{\bigO(w \sqrt{n})} n^{\bigO(1)} $ time.
Observe that
the event of $S$ being a $2^{\sqrt{n}}$-approximation
(which includes the case of $\det(\mat{A}_{S,S}) = \OPT$)
occurs with probability at least $\frac{\OPT^p}{\ZZ^p(\mat{A})} $ and that
it does not occur
with probability at most
\begin{align*}
    \sum_{S \subseteq [n]: S \text{ is not } 2^{\sqrt{n}}\text{-approx.}}
    \frac{\det(\mat{A}_{S,S})^p}{\ZZ^p(\mat{A})}
    \leq \sum_{S \subseteq [n]: S \text{ is not } 2^{\sqrt{n}}\text{-approx.}}
    \frac{(2^{-\sqrt{n}} \OPT)^p}{\ZZ^p(\mat{A})}
    \leq \frac{2^n \cdot 2^{-p\sqrt{n}} \OPT^p}{\ZZ^p(\mat{A})}.
\end{align*}
Hence, we have that the probability of success is at least
\begin{align*}
1-\frac{\left(\dfrac{2^n \cdot 2^{-p\sqrt{n}} \OPT^p}{\ZZ^p(\mat{A})}\right)}{\left(\dfrac{\OPT^p}{\ZZ^p(\mat{A})}\right)}
 = 1-2^{n-p\sqrt{n}} \geq 1-2^{-n},
\end{align*}
which completes the proof.
\end{proof}

\begin{remark}
A similar approach does not work for size-constrained MAP inference for the following reason:
for a size parameter $k \in [n]$,
the objective is to compute a $2^k$-approximation, i.e.,
a set $S \in {[n] \choose k}$ such that
$\det(\mat{A}_{S,S}) \geq 2^{-\sqrt{k}} \argmax_{S^* \in {[n] \choose k}} \det(\mat{A}_{S^*,S^*}) $.
Let $p \triangleq \ceil{2\sqrt{k}}$, and 
consider a fixed-size E-DPP of exponent $p$ defined by $\mat{A}$, which coincides with the $k\Pi$-DPP defined by $p$ copies of $\mat{A}$.
Then, a sample drawn from this $k\Pi$-DPP is a $2^{\sqrt{k}}$-approximation with probability at least
\begin{align*}
    1-\frac{{n \choose k} \cdot 2^{-p\sqrt{k}}\OPT^p}{\OPT^p} \geq 
    1 - \bigO\left(\frac{n}{k}\right)^k,
\end{align*}
which can be $0$.
\end{remark}

\section{Concluding Remarks and Open Questions}
\label{sec:conclusion}

We studied the computational complexity of 
the normalizing constant for the product of determinantal point processes.
Our results (almost) rule out the possibility of efficient algorithms for general cases and
devised the fixed-parameter tractability.
Several open questions are listed below.

\begin{itemize}
    \item\textbf{Q1.}~Can we show the intractability of computing
    $ \ZZ^p $ for $p$ which is not a positive even integer,
    such as $p=3$ and $p=1.1$?
    The proof strategy employed in \cref{sec:intract} would no longer work.
    
    \item\textbf{Q2.}~Can we develop more ``practical'' FPT algorithms having a small exponential factor, e.g., $f(k) = 2^k$?
    We might have to avoid enumerating bijections.

    \item\textbf{Q3.}~Can we establish fixed-parameter tractability for parameters other than the treewidth or matrix rank, such as the cliquewidth?

    \item \textbf{Q4.}~Can we devise an \emph{exact} FPT algorithm for computing $\ZZ^p$ for a fractional exponent $p$ parameterized by the treewidth?

    \item \textbf{Q5.}~Can we remove the bounded-treewidth assumption from \cref{thm:map-approx} to obtain a subexponential algorithm for (unconstrained) MAP inference?
\end{itemize}
The first author of this article gave a partial answer to \textbf{Q1} \citep{ohsaka2021unconstrained}: it is \NP-hard to approximate $\ZZ^p$ within an exponential factor in the order of an input matrix if $p \geq 10^{10^{13}}$.

\section*{Acknowledgments}
The authors are grateful for the helpful conversations with Shinji Ito and
thank Leonid Gurvits very much for pointing out the reference \citep{gurvits2009complexity}.

\bibliographystyle{abbrvnat}
\bibliography{ref.bib}

\end{document}